\documentclass[11pt,letter]{article}

\usepackage{amsmath}
\usepackage{amsthm}
\usepackage{amssymb}
\usepackage{verbatim}
\usepackage{bbm}
\usepackage{algorithm}
\usepackage[margin = 1in]{geometry}
\usepackage{graphicx}
\usepackage{makecell}

\newcommand{\bdisp}{\begin{displaystyle}}
\newcommand{\edisp}{\end{displaystyle}}

\renewcommand{\Pr}{\operatorname*{\mathbb{P}}}
\newcommand{\Var}{\operatorname*{\mathrm{Var}}}

\newcommand{\Exp}{\operatorname*{\mathbb{E}}}
\newcommand{\from}{\leftarrow}

\newcommand{\poly}{\operatorname*{\mathrm{poly}}}

\newcommand{\Bin}{\mathrm{Bin}}

\newcommand{\Bernoulli}{\mathrm{Bernoulli}}

\newcommand{\scary}{disc}
\newcommand{\DKL}{D_\mathrm{KL}}

\renewcommand{\vec}[1]{\mathbf{#1}}

\newcommand{\Grad}{\nabla}

\newcommand{\eps}{\epsilon}

\newcommand{\nmax}{n_{\max}}
\allowdisplaybreaks

\newcounter{nTheorems}

\newtheorem{theorem}[nTheorems]{Theorem}
\newtheorem{corollary}[nTheorems]{Corollary}

\newtheorem{lemma}[nTheorems]{Lemma}
\newtheorem{proposition}[nTheorems]{Proposition}

\newtheorem{fact}[nTheorems]{Fact}
\newtheorem{example}{Example}

\theoremstyle{definition}
\newtheorem{definition}[nTheorems]{Definition}

\title{Uncertainty about Uncertainty: Optimal Adaptive Algorithms for Estimating Mixtures of Unknown Coins\footnote{We thank Tim Kraska and Yeounoh Chung for bringing these problems to our attention in the data analytics setting, and for contributing to the simulation results in this work. We thank an anonymous reviewer for asking about the $\delta$ dependence in lower bounds, which led to the current tight results. This work is partially supported by NSF award IIS-1562657. In addition, Paul Valiant is partially supported by NSF award DMS-1926686, and indirectly supported by NSF award CCF-1900460.}}
\author{
Jasper C.H.~Lee\\ \ \\
Brown University\\
\texttt{jasperchlee@brown.edu}\\
\and
Paul Valiant\\ \ \\
IAS and Purdue University\\
\texttt{pvaliant@gmail.com}
}
\begin{document}

\maketitle

\begin{abstract}
	Given a mixture between two populations of coins, ``positive" coins that each have---unknown and potentially different---bias $\geq\frac{1}{2}+\Delta$ and ``negative" coins with bias $\leq\frac{1}{2}-\Delta$, we consider the task of estimating the fraction $\rho$ of positive coins to within additive error $\eps$.
	We achieve an upper and lower bound of $\Theta(\frac{\rho}{\eps^2\Delta^2}\log\frac{1}{\delta})$ samples for a $1-\delta$ probability of success, where crucially, our lower bound applies to all \emph{fully-adaptive} algorithms.
	Thus, our sample complexity bounds have tight dependence for every relevant problem parameter.
	A crucial component of our lower bound
proof is a decomposition lemma (see Lemmas~\ref{Lem:Reduction} and \ref{Lem:ReductionKL}) showing how to assemble partially-adaptive bounds into a fully-adaptive bound, which may be of independent interest: though we invoke it for the special case of Bernoulli random variables (coins), it applies to general distributions.
We present simulation results to demonstrate the practical efficacy of our approach for realistic problem parameters for crowdsourcing applications, focusing on the ``rare events" regime where $\rho$ is small. 
	The fine-grained adaptive flavor of both our algorithm and lower bound contrasts with much previous work in distributional testing and learning.
\end{abstract}

\thispagestyle{empty}
\newpage
\setcounter{page}{1}
\section{Introduction}

We consider a natural statistical estimation task, motivated by a practical setting, with an intriguing adaptive flavor. We provide a new adaptive algorithm and a matching fully adaptive lower bound, tight up to multiplicative constants.

In our problem setting, there is a universe of coins of two types: positive coins each have a (potentially different) probability of heads that lies in the interval $[\frac{1}{2}+\Delta,1]$, while negative coins lie in the interval $[0,\frac{1}{2}-\Delta]$, where $\Delta\in(0,\frac{1}{2}]$ parameterizes the ``quality" of the coins. Our only access to the coins is by choosing a coin and then flipping it, without access to the true biases of the coins. An algorithm in this setting may employ arbitrary adaptivity---for example, flipping three different coins in sequence and then flipping the first coin 5 more times if and only if the results of the first 3 flips were heads, tails, heads. The challenge is to estimate the \emph{fraction} $\rho$ of coins that are of positive type, to within a given additive error $\eps$, using as few coin flips (samples) as possible. We assume because of the symmetry of the problem (between positive and negative coins) that $\rho\leq \frac{1}{2}$.

This model arose from a collaboration with colleagues in data science and database systems, about harnessing paid crowdsourced workers to estimate the ``quality" of a database. Our model is a 
direct theoretical analog of the following practical problem, where sample complexity linearly translates into the amount of money that must be paid to workers, and thus even multiplicative factors crucially affect the usefulness of an algorithm. Given a set of data and a predicate on the data, the task is to estimate what fraction of the data satisfies the predicate---for example, estimating the proportion of records in a large database that contain erroneous data. After automated tools have labeled whatever portion of the data they are capable of dealing with, the remaining data must be processed via \emph{crowdsourcing}, an emerging setting that potentially offers sophisticated capabilities but at the cost of unreliability. Namely, for each data item, one may ask many human users/workers online whether they think the item satisfies the predicate, with the caveat that the answers returned could be noisy. In the case that the workers have no ability to distinguish the predicate, we cannot hope to succeed; however, if the histograms of detection probabilities for positive versus negative data have a gap between them (the gap is $2\Delta$ in the model above), then the challenge is to estimate $\rho$ as accurately as possible, from a limited budget of queries to workers~\cite{Chung:2017}.

A key feature that makes this estimation problem distinct from many others studied in the literature is the richness of adaptivity available to the algorithm. Achieving a tight lower bound in this setting requires considering and bounding all possible uses of adaptivity available to an algorithm; and achieving an optimal algorithm requires choosing the appropriate adaptive information flow between different parts of the algorithm.
Much of the previous work in the area of statistical estimation is focused on non-adaptive algorithms and lower bounds; however see~\cite{Canonne:survey}, and in particular, Sections 4.1 and 4.2 of that work, for a survey of several distribution testing models that allow for adaptivity.
In our setting there are two distinct kinds of adaptivity that an algorithm can leverage:
1) single-coin adaptivity, deciding how many times a particular coin should be flipped---a per-coin stopping rule---in terms of the results of its previous flips, and
2) cross-coin adaptivity, deciding which coin to flip next in terms of the results of previous flips across \emph{all} coins. 
Our final optimal algorithm (Section~\ref{sec:optimal-algorithm}) leverages both kinds of adaptivity.
In our tight lower bound analysis (Section~\ref{sect:UnknownLower}), we overcome the technical obstacles presented by the richness of adaptivity by giving a reduction (Section~\ref{sect:Reduction}) from fully-adaptive algorithms that leverage both kinds of adaptivity to single-coin adaptive algorithms that process each coin independently, valid for our specific lower bound instance. We discuss the approaches and challenges of our lower bound in more detail in Section~\ref{sect:LowerIntro}.


The main \emph{algorithmic} challenge in this problem is what we call ``uncertainty about uncertainty": we make no assumptions about the quality of the coins beyond the existence of a gap $2\Delta$ between biases of the coins of different types (centered at $\frac{1}{2}$).
If we relaxed the problem, and assumed (perhaps unrealistically) that we know 1) the conditional distribution of biases of positive coins, and 2) the same for negative coins, and 3) an initial estimate of the mixture parameter $\rho$ between the two distributions, then we show that it is easy---using mathematical programming techniques in Section~\ref{sect:QP}---to construct an estimation algorithm with sample complexity that is optimal \emph{by construction} up to a multiplicative constant (see Section~\ref{sect:QPLower}).
On the other hand, our algorithm for the original setting has to return estimates with small bias, and be sample efficient at the same time, regardless of the bias of the coins, be they all deterministic, or all maximally noisy as allowed by the $\Delta$ parameter, or some quality in between.
While intuitively the hardest settings to distinguish information theoretically involve coins with biases as close to each other as possible (and indeed our lower bound relies on mixtures of only $\frac{1}{2}\pm\Delta$ coins), settings with biases near but not equal to $\frac{1}{2}\pm\Delta$ introduce ``uncertainty about uncertainty" challenges. 
The two kinds of adaptivity available to the algorithm allow us to meet these challenges by  trading off, optimally, between 1) investigating a single coin to reduce uncertainty about its bias, and 2) apportioning resources between different coins to reduce uncertainty about the ground truth fraction $\rho$, which is the objective of the problem.

\subsection{Our Approaches and Results}

To motivate the new algorithms of this paper, we start by describing the straightforward analysis of perhaps the most natural approach to the problem, which is non-adaptive, based on subsampling.

\begin{example}\label{ex:non-adaptive-algorithm}
Recall that it takes $\Omega(\frac{1}{\Delta^2})$ samples to distinguish a coin of bias $\frac{1}{2}-\Delta$ from a coin of bias $\frac{1}{2}+\Delta$.
We can therefore imagine an algorithm that chooses a random subset of the coins, and flips each coin $\Omega(\frac{1}{\Delta^2})$ many times.
Asking for $\Theta(\frac{1}{\Delta^2}\log \frac{1}{\eps})$ flips from each coin guarantees that all but $\eps$ fraction of the coins in the subset will be accurately classified. Given an accurate classification of $m$ randomly chosen coins, we use the fraction of these that appear positive as an estimate on the overall mixture parameter $\rho$. Estimating $\rho$ to within error $\epsilon$ requires $m=O(\frac{\rho}{\epsilon^2})$ randomly chosen coins. Overall, taking $\Theta(\frac{1}{\Delta^2}\log \frac{1}{\eps})$ samples from each of $m=\Theta(\frac{\rho}{\epsilon^2})$ coins uses $\Theta(\frac{\rho}{\eps^2\Delta^2}\log\frac{1}{\eps})$ samples.
\end{example}


As we will see, the above straightforward algorithm is potentially wasteful in samples by up to a $\log\frac{1}{\eps}$ factor, since it makes $\Theta(\frac{1}{\Delta^2}\log\frac{1}{\eps})$ flips for every single coin, yet---since $\Omega(\frac{1}{\Delta^2})$ samples suffices to label a coin with constant accuracy---each sample beyond the first $\Theta(\frac{1}{\Delta^2})$ samples from a single coin gives increasing certainty yet diminishing information-per-coin.
If we can save on this $\log\frac{1}{\eps}$ factor without sacrificing impractical constants, then our approach leads to significant practical savings in samples, and thus monetary cost---in regimes, such as crowdsourcing, where gathering data is by far the most expensive part of the estimation process.



\subsubsection{Algorithmic Construction}

We give two algorithmic constructions. Algorithm~\ref{Alg:Main}, which we call the Triangular Walk algorithm, is single-coin adaptive, and is theoretically almost-tight in sample complexity. Second, Algorithm~\ref{Alg:Optimal} has the optimal sample complexity, by combining the Triangular Walk algorithm with a new (and surprisingly) non-adaptive component (Algorithm~\ref{Alg:RefinedSampling}).

The Triangular Walk algorithm (Algorithm~\ref{Alg:Main}) is designed for the specific \emph{practical} parameter regime where $\rho$ is small: in our earlier crowdsourcing example, practitioners typically preprocess data items by using automated techniques and heuristics to classify a majority of the items, before leaving to crowdsourced workers a small number of items that cannot be automatically classified.
These automated filtering techniques usually flag significantly more  ``negative" items than ``positive" items as ``unclassifiable automatically", resulting in a small fraction $\rho$ of positive items among the ones selected for crowdsourced human classification.
The intuition behind our approach, then, is to try to abandon sampling (frequent) negative coins as soon as possible, after $\Theta(\frac{1}{\Delta^2})$ samples, while being willing to investigate (infrequent) positive coins up to depth $\Theta(\frac{1}{\Delta^2}\log\frac{1}{\eps})$. Thus we disproportionately bias our investment of resources towards the rare and valuable regime. Using techniques from random walk theory, we design a linear estimator based on this behavior (Algorithm~\ref{Alg:Single}), whose expectation across many coins yields a robust estimator, Algorithm~\ref{Alg:Main}, as shown in Theorem~\ref{Thm:OneSizeFinal} (restated and proved in Section~\ref{sect:TWalk}).

\begin{theorem}
	\label{Thm:OneSizeFinal}
	Given coins where a $\rho$ fraction of the coins have bias $\geq\frac{1}{2}+\Delta$, and $1-\rho$ fraction have bias $\leq\frac{1}{2}-\Delta$, then running
	Algorithm~\ref{Alg:Main} on $t=\Theta(\frac{\rho}{\eps^2}\log\frac{1}{\delta})$ randomly chosen coins will estimate $\rho$ to within an additive error of $\pm\eps$, with probability at least $1-\delta$, with an expected sample complexity of $O(\frac{\rho}{\eps^2\Delta^2}(1+\rho\log\frac{1}{\eps})\log\frac{1}{\delta})$.
\end{theorem}

The analysis of Algorithm~\ref{Alg:Main} uses only standard concentration inequalities, and thus the big-O notation for the sample complexity does not hide large constants.
As further evidence of the good practical performance of Algorithm~\ref{Alg:Main}, Section~\ref{sec:experiments} shows simulation-based experimental results, run on settings with practical problem parameters for crowdsourcing applications.
These results demonstrate the advantages of our algorithm as compared with the straightforward majority vote algorithm as well as the state-of-the-art algorithm~\cite{Chung:2017} (which does not enjoy any theoretical guarantees).

As for our second, optimal, algorithmic construction (Algorithm~\ref{Alg:Optimal} in Section~\ref{sec:optimal-algorithm}), we combine the adaptive techniques from the Triangular Walk algorithm with a non-adaptive estimation component.
More concretely, in the regimes where Algorithm~\ref{Alg:Main} is not optimal, Algorithm~\ref{Alg:Optimal} uses Algorithm~\ref{Alg:Main} to first give a 2-approximation of $\rho$, before using this information to non-adaptively estimate $\rho$ much more accurately, while keeping the variance of the estimate small, to control the sample complexity.
The theoretical guarantees of Algorithm~\ref{Alg:Optimal} are shown in Theorem~\ref{thm:opt} (restated and proved in Section~\ref{sec:optimal-algorithm}).

\begin{theorem}[Informal]\label{thm:opt}
Given coins where a $\rho$ fraction of the coins have bias $\geq\frac{1}{2}+\Delta$, and $1-\rho$ fraction have bias $\leq\frac{1}{2}-\Delta$, then for large enough constant $c$, running
Algorithm~\ref{Alg:Optimal} on a budget of $B\geq c\frac{\rho}{\Delta^2\eps^2}$ coin flips will estimate $\rho$ to within an additive error of $\pm\eps$, with probability at least $2/3$.
If the algorithm is repeated $\Theta(\log\frac{1}{\delta})$ times, and the median estimate is returned, then the probability of failure is at most $\delta$.

\end{theorem}

\subsubsection{Lower Bounds and Discussion}
\label{sect:LowerIntro}

Complementary to our algorithm, we show a matching lower bound of $\Omega(\frac{\rho}{\eps^2\Delta^2}\log\frac{1}{\delta})$ samples for a success probability of $1-\delta$ for the problem. Crucially, our bounds match across choices of all four parameters, $\rho,\eps,\Delta,\delta$.
To show the lower bound, we use the following setup: consider a scenario where all positive coins have bias \emph{exactly} $\frac{1}{2}+\Delta$ and all negative coins have bias \emph{exactly} $\frac{1}{2}-\Delta$.

The overall intuition for our lower bound is that, for each coin, even flipping it enough to learn whether it is a positive or negative coin will tell us little about whether the true fraction of positive coins is $\rho$ versus $\rho+\eps$, and thus the flow of information to our algorithm is at most a slow trickle. To capture this intuition, we aim to decompose the analysis into a sum of coin-by-coin bounds; however, the key challenge is the \emph{cross-coin adaptivity} that is available to the algorithm. 

To demonstrate the challenge of tightly analyzing cross-coin adaptivity, consider the following natural attempt at a lower bound. 

\vspace*{-0.0em}
\begin{enumerate}
    \setlength{\itemsep}{0em}
    \item Consider flipping a fair coin $S$ to choose between a universe with $\rho$ fraction of positive coins, versus $\rho+\eps$ fraction.
    \item The aim is to bound the amount of mutual information that the entire transcript of an adaptive coin-flipping algorithm can have with the coin $S$.
    \item Suppose this mutual information can be bounded by the mutual information of the sub-transcript of the $i^\textrm{th}$ coin with $S$, summed over all $i$.
    \item\label{step:single} Thus consider and bound the amount of mutual information between the sub-transcript of just coin $i$ alone, with $S$; and sum these bounds over all coins at the end.
\end{enumerate}
\vspace*{-0.0em}
While one would intuitively expect the bounds of Step~\ref{step:single} to be small for each single coin, cross-coin adaptivity allows for each single-coin sub-transcript to encode a lot of mutual information via its \emph{length}, which may be adaptively chosen by the algorithm in light of information gathered across all other coins. The amount of mutual information about $S$ in a sub-transcript may be linear in the number of times \emph{other} coins have been flipped, implying that summing up such mutual information across all coins would yield a bound that uselessly grows quadratically with the number of flips, instead of linearly.

\medskip\noindent{\bf Our approach:} 
We show that no fully-adaptive algorithm can distinguish the following two scenarios with probability at least $1-\delta$, using $o(\frac{\rho}{\eps^2\Delta^2}\log\frac{1}{\delta})$ samples: 1) when a $\rho$ fraction of the coins are positive, and 2) when a $\rho+\eps$ fraction of the coins are positive.
This is formalized as the following theorem (Theorem~\ref{thm:UnknownLower}), and proved in Section~\ref{sect:UnknownLower}.

\begin{theorem}
	\label{thm:UnknownLower}
	For $\rho\in [0,\frac{1}{2})$ and $\epsilon\in (0, 1-2\rho]$, the following two situations are impossible to distinguish with at least $1-\delta$ probability using an expected $o(\frac{\rho}{\eps^2\Delta^2}\log\frac{1}{\delta})$ samples: A) $\rho$ fraction of the coins have probability $\frac{1}{2}+\Delta$ of landing heads and $1-\rho$ fraction of the coins have probability $\frac{1}{2}-\Delta$ of landing heads, versus B) $\rho+\epsilon$ fraction of the coins have probability $\frac{1}{2}+\Delta$ of landing heads and $1-(\rho+\epsilon)$ fraction of the coins have probability $\frac{1}{2}-\Delta$ of landing heads.
	This impossibility crucially includes fully-adaptive algorithms.
\end{theorem}

In Section~\ref{sect:Reduction}, we capture rather generally via Lemmas~\ref{Lem:Reduction} and~\ref{Lem:ReductionKL} the above intuitive decomposition of a many-coin adaptive algorithm into its single-coin contributions, but via a careful simulation argument that precludes the kind of information leakage between coins that we described above. More explicitly, instead of decomposing a single transcript into many (possibly correlated) sub-transcripts, we relate an $n$-coin transcript to $n$ \emph{separate} runs of the algorithm (each on freshly drawn random coins), where in the $i^\textrm{th}$ run, coin $i$ is authentically sampled (from either the $\rho$ scenario or the $\rho+\eps$ scenario), while all the remaining coins are simulated by the algorithm. Crucially, since the remaining simulated coins do not depend on the ``real" scenario, no cross-coin adaptivity can leak any information about the real world to coin $i$, beyond the information gained from flipping coin $i$ itself.

Furthermore, Lemmas~\ref{Lem:Reduction} and~\ref{Lem:ReductionKL} apply to a broad variety of problem settings, where the population of random variables can be arbitrary and not necessarily Bernoulli coins.
We believe these lemmas are of independent interest beyond this work, and can be a useful tool for proving lower bounds for other problem settings, for example a Gaussian variant of the current problem, where instead of being input a noisy yes/no answer on the positivity of an item, we instead receive a numerical Gaussian-distributed score with mean, say, $>1$ for positive items and $<0$ for negative items.

Given the decomposition lemmas (Lemmas~\ref{Lem:Reduction} and~\ref{Lem:ReductionKL}), completing the lower bound analysis for the current problem requires upper bounding the squared Hellinger distance between running any single-coin adaptive algorithm on the two coin populations described earlier, with slightly different positive-to-negative mixture ratios.
This forms the bulk (and technical parts) of the proof of Theorem~\ref{thm:UnknownLower}.

\medskip\noindent{\bf Non-adaptive bounds:}
As motivation for the algorithmic results of this paper, it is reasonable to ask, given Theorem~\ref{thm:UnknownLower}'s lower bound of $\Omega(\frac{\rho}{\eps^2\Delta^2}\log\frac{1}{\delta})$ on the number of samples for our problem, is it possible that a \emph{non-adaptive} algorithm can approach this performance, or is the adaptive flavor of Algorithms~\ref{Alg:Main} or~\ref{Alg:Optimal} required?
We briefly describe how the framework of the ``natural attempt" (the numbered list above) in fact yields a lower bound for \emph{non-adaptive} algorithms that is a $\log\frac{1}{\rho}$ factor higher than that of Theorem~\ref{thm:UnknownLower}, namely $\Omega(\frac{\rho}{\eps^2\Delta^2}\log\frac{1}{\rho})$, when $\rho\geq\eps^2$

Given a random variable $S$ that uniformly chooses between scenarios ``$\rho$" and ``$\rho+\eps$" respectively, and a sample of size $n$ from a coin that has bias $\frac{1}{2}+\Delta$ with probability $\rho$ or $\rho+\eps$ respectively, and bias $\frac{1}{2}-\Delta$ otherwise, what is the mutual information between the $n$ observed flips (from a single coin) and the scenario variable $S$? A non-adaptive algorithm must fix the number of queries $n$ independent of the observed outcomes from the coins, where the information about $S$ is the sum received from sampling each coin. Thus the optimal such algorithm chooses $n$ that maximizes the mutual information per sample. Estimates of this mutual information in the relevant cases are not too difficult, as this is the mutual information between a univariate distribution that is the mixture of two binomials, with a fair coin that determines the mixture probabilities. In terms of $\Delta$, and $\rho\geq \eps^2$, some calculation shows that the optimal value of $n$ is $\Theta(\frac{1}{\Delta^2}\log\frac{1}{\rho})$, which yields the above-claimed non-adaptive lower bound on sample complexity of $\Omega(\frac{\rho}{\eps^2\Delta^2}\log\frac{1}{\rho})$.
See Appendix~\ref{app:non-adaptive} for the complete calculations.

For the constant-$\rho$ (and constant probability) regime, this lower bound is in fact tight.
A major component of our final algorithm, Algorithm~\ref{Alg:RefinedSampling}, when run on a single constant quality ($\Delta = \Omega(1)$\,) coin with the parameter $f=f_0(p)$ as defined in Definition~\ref{def:f(p)}, is a non-adaptive unbiased estimator for the indicator function of the positivity of the coin, with small variance and constant sample complexity.
For a low quality coin ($\Delta = o(1)$\,), we can simulate a flip of a constant quality coin by taking the majority result of $\Theta(1/\Delta^2)$ low quality coin flips.
Returning the mean of $O(\frac{1}{\eps^2})$ repetitions of  Algorithm~\ref{Alg:RefinedSampling} on different coins yields an $\eps$-accurate estimate of $\rho$.
The total sample complexity is $O(\frac{1}{\eps^2\Delta^2})$, which matches the non-adaptive lower bound in the constant-$\rho$ regime.

In summary, we have the adaptive and non-adaptive bounds in Table~\ref{table:bounds}.
As shown in Table~\ref{table:bounds}, the non-adaptive bounds match each other and the adaptive bounds only in the regime where $\rho = \Theta(1)$ (and in the trivial $\eps = \Theta(1)$ regime).
In the non-constant $\rho$ regime, the non-adaptive lower bound is asymptotically larger than the adaptive lower bound, demonstrating the need for adaptivity in the design of our final optimal algorithm.

\begin{table*}[t]
\renewcommand{\arraystretch}{2.2}
\centering
\hspace*{-0.4cm}
\begin{tabular}{|r|l|l|}
\hline
     & \multicolumn{1}{c|}{Upper Bound} & \multicolumn{1}{c|}{Lower Bound} \\[0.5em]
     \hline
    Adaptive & $O(\frac{\rho}{\eps^2\Delta^2}\log\frac{1}{\delta})$ \hspace{2cm}(Algorithm~\ref{Alg:Optimal}) & $\Omega(\frac{\rho}{\eps^2\Delta^2}\log\frac{1}{\delta})$ \hspace{2.7cm}(Section~\ref{sect:UnknownLower})\\[0.5em]
    \hline
    Non-adaptive & \makecell{$O(\frac{\rho}{\eps^2\Delta^2}\log\frac{1}{\eps}\log\frac{1}{\delta})$ (Trivial, Example~\ref{ex:non-adaptive-algorithm})\\ $O(\frac{1}{\eps^2\Delta^2}\log\frac{1}{\delta})$ (Algorithm~\ref{Alg:RefinedSampling} for $\rho=\Theta(1)$\,)} & $\Omega(\frac{\rho}{\eps^2\Delta^2}\log\frac{1}{\rho})$ (For $\rho \ge \eps^2$ and constant $\delta$) \\[0.5em]
    \hline
\end{tabular}
\caption{Sample Complexity Upper and Lower Bounds}
\label{table:bounds}
\end{table*}

%

\subsubsection{Practical Considerations}

The keen-eyed reader might notice that the algorithmic results in Theorems~\ref{Thm:OneSizeFinal} and~\ref{thm:opt} both depend on the unknown ground truth $\rho$, so thus these bounds are not immediately invokable by a user.
We present two approaches to address this issue.

The first approach is to note that Algorithm~\ref{Alg:Main} can be interpreted as an \emph{anytime} algorithm: it can produce an estimate at any point in its execution.
As more coins are used in Algorithm~\ref{Alg:Main}, the estimate simply gains accuracy.
Section~\ref{sec:practice} discusses this approach in more detail, and our experiments in Section~\ref{sec:experiments} are also run using the same approach.
Because of its simplicity, we recommend this method in practice.

A complication arising from this approach is the fact the sample complexity bound of Theorem~\ref{Thm:OneSizeFinal} is an \emph{expected} sample complexity bound.
Thus there are potential issues introduced by abruptly stopping the algorithm after a fixed budget of samples, which might inadvertently introduce bias to the estimate. Section~\ref{sec:practice} also shows how to analyze and address this issue.

The second, theoretically more interesting approach is to fix a total budget of allowable coin flips, and have the algorithm ``discover" the optimal achievable accuracy $\eps$ just from interacting with the different coins.
Our presentation and analysis of Algorithm~\ref{Alg:Optimal}, in Section~\ref{sec:optimal-algorithm}, follows this approach.
We point out that Algorithm~\ref{Alg:Main} can also be made to have this theoretical guarantee, as demonstrated by the invocation of Algorithm~\ref{Alg:Main} in Algorithm~\ref{Alg:Optimal}.

\subsubsection{Designing Optimal Estimators when Coin Biases are Known}

By contrast with the above results that analyze the ``uncertainty about uncertainty" regime with unknown populations of coins, we shed light on the algorithmic challenges of that regime by providing a tight analysis in the case where knowledge about the populations of coins can be leveraged by the algorithm.
In particular, we give a bootstrapping approach which takes some initial guess of $\rho$ along with knowledge of the coin population, and produces an optimal-by-construction estimator that can be used improve on the initial estimate.
Explicitly, consider the regime where we know 1) the distribution of coin biases conditioned on being a positive coin, 2) analogously for negative coins and 3) also $\rho$ itself, for bootstrapping purposes even though in practice we would only have a guess.
Suppose further that we are given 4) the constraint that we will invest at most $\nmax$ flips on a single coin, controlling both the sample complexity but also the computational complexity we can afford to computer the optimal estimator.
In Section~\ref{sect:QP}, we use quadratic and linear programming techniques to find a single-coin adaptive algorithm, taking the form of a linear estimator, with the minimum variance possible subject to the constraint that, even if our knowledge of $\rho$ is wrong, the estimator is still unbiased.
This construction yields the following theorem.

\begin{theorem}
\label{thm:QPUpper}
Suppose we are given 1) the distribution of coin biases conditioned on being a positive coin, 2) the analogous distribution for negative coins and 3) the mixture parameter $\rho$ (which, again, is a circular assumption but useful for a bootstrapping approach).
Suppose further that we are given 4) the parameter $\nmax$, which controls the maximum depth of the triangular walk.

Then, following the method described in Section~\ref{sect:QP}, we can find the linear estimator for $\rho$ that minimizes variance, subject to a) the expected output of the estimator on input a random positive coin is 1 and b) the analogous expected output for a random negative coin is 0.

Moreover, if the objective of the linear program in Figure~\ref{Fig:NewLP} is $U$, then the expected sample complexity of the constructed linear estimator is $O(\frac{1}{U\eps^2}\log\frac{1}{\delta})$, which will estimate $\rho$ to within an additive error of $\eps$ with probability at least $1-\delta$.
\end{theorem}

We further show in Section~\ref{sect:QPLower} that this linear estimator construction is in fact optimal in sample complexity, up to constant multiplicative factors, in the regime of constant probability success and subject to the same constraint that each coin can only be flipped at most $\nmax$ times.
The following theorem captures the exact guarantees.

\begin{theorem}
\label{thm:QPOpt}
Suppose we are given the 4 pieces of data as in Theorem~\ref{thm:QPUpper} above.
	
The linear estimator produced from solving the linear program in Figure~\ref{Fig:NewLP}, as described in Theorem~\ref{thm:QPUpper}, has total expected sample complexity that is within a constant factor of any optimal fully-adaptive algorithm with $\geq\frac{2}{3}$ probability of success, subject to the same constraint that no coin is flipped more than $\nmax$ many times.
\end{theorem}

The proof of this theorem---like our main lower bound of Theorem~\ref{thm:UnknownLower}---also relies on Lemma~\ref{Lem:Reduction} to relate fully-adaptive algorithms to single-coin-adaptive algorithms; and constant-factor tightness comes from the fact that the linear programs minimizing the variance of a linear estimator versus maximizing squared Hellinger distance are within a constant factor of each other.


\subsection{Related Work}

A related line of work considers the scenario where all positive coins have identical bias (not necessarily greater than $1/2$), and negative coins also have identical bias (strictly less than the positive coins' bias), with the ultimate goal of \emph{identifying} any single coin that is positive (or ``heavy" in the terminology of these works).
The problem has been studied and solved optimally in the context where the biases and positive-negative proportions are known~\cite{Karp:2014}, and also when none of this information is known~\cite{Malloy:2012,Jamieson:2016}. Such problems may be seen as a special case of bandit problems.

Another related line of work concerns the \emph{learning} of distributions of (e.g.~coin) parameters over a population, which arises in various scientific domains~\cite{lord1965strong,lord1975empirical,millar1986distribution,palmer1990small,colwell1994estimating,bell2000environmental}.
In particular, the works of Lord~\cite{lord1965strong}, and Kong et al.~\cite{Tian:2017,vinayak2019maximum} consider a model similar to ours, with the crucial difference that each coin is sampled a fixed number $t$ many times---instead of allowing adaptive sampling as in the current work---with the objective of learning the distribution of biases of the coins in the universe.

Since an earlier version of this paper was posted on arXiv, more recent work by Brennan et al.~\cite{brennan:2020} considers a generalization of our setting, but because of different motivation and parameterization, both their upper and 
lower bounds are not directly comparable with ours.

Our problem also sits in the context of estimation and learning tasks with \emph{noisy} or \emph{uncalibrated} queries.
The noiseless version of our problem would be when $\Delta = \frac{1}{2}$ and thus $\frac{1}{2}\pm\Delta$ equals either 0 or 1.
That is, all coins are either deterministically heads or deterministically tails, and thus estimating the mixture parameter $\rho$ is equivalent to estimating the parameter of a \emph{single} coin with bias $\rho$, which has a standard analysis.
Prior works have considered noisy versions of well-studied computational problems, such as (approximate) sorting and maximum selection under noisy access to pairwise comparisons~\cite{falahatgar2017maximum,falahatgar2017maxing} and maximum selection under access to uncalibrated numerical scores that are consistent with some global ranking~\cite{Wang:2018}.

Furthermore, our problem can be interpreted as a special case of the ``testing collections of distributions" model introduced by Levi, Ron and Rubinfeld~\cite{Levi:2013,Levi:2014}, modulo the distinction between testing and parameter estimation.
In their model, a collection of $m$ distributions $(D_1,\ldots,D_m)$ (over the same domain) is given to the tester, and the task is to test whether the collection satisfies a particular \emph{property}, where a property in this case is defined as a subset of $m$-tuples of distributions.
In the \emph{query} access model, one is allowed to name an index $i \in \{1,\ldots, m\}$ and get a fresh sample from the distribution $D_i$.
Our problem can be analogously phrased in this model, where the distributions are over the domain $\{0,1\}$, and the property in question is whether the fraction $\rho$ of distributions in the collection having bias $\ge 1/2$ is greater than some threshold $\tau$.

We highlight other distribution testing models that allow for adaptive sampling access.
For example, in testing contexts, conditional sampling oracles have been considered~\cite{CFGM13:2016,CRS14,CRS15,Can15,Falahatgar:2015faster,ACK14}, where a subset of the domain is given as input to the oracle, which in turn outputs a sample from the underlying unknown distribution conditioned on the subset.
Evaluation oracles have also been considered~\cite{RS09,BDKR05,GMV06,CR14}, where the testing algorithm has access to an oracle that evaluates the probability mass function or the cumulative mass function of the underlying distribution.
See the survey by Canonne~\cite{Canonne:survey} for detailed comparisons between the different standard and specialized access models, along with a discussion of recent results.

Adaptive lower bounds of problems related to testing monotonicity of high-dimensional Boolean functions have a somewhat similar setup to ours, where binary decisions adaptively descend a decision tree according to probabilities that depend both on the algorithm and its (unknown) input that it seeks to categorize~\cite{Belovs:2016,Chen:2017}.
Lower bounds in these works rely on showing that the probabilities of reaching any leaf in the decision tree under the two scenarios that they seek to distinguish are either exponentially small or within a constant factor of each other. This proof technique is powerful yet does not work in our setting, as many adaptive algorithms have high-probability outcomes that yield non-negligible insight into which of the two scenarios we are in. By contrast, our proof technique involves showing that, while such ``insightful" outcomes may be realized with high probability, in these cases we must pay a correspondingly high sample complexity cost somewhere \emph{else} in the adaptive tree.

A crucial part of our lower bound proof, Lemma~\ref{Lem:ReductionKL}, involves carefully ``decomposing" fully-adaptive (multi-coin) algorithms into their single-coin components.
Work by Braverman et al.~\cite{Braverman:2016} gives a data processing inequality in the context of communication lower bounds, whose proof uses similar ideas to how we prove Lemma~\ref{Lem:ReductionKL}.

As described at the beginning of the introduction, results of this work have practical applications in \emph{crowdsourcing} algorithms in the context of data science and beyond.
Theoretical studies with similar aims to our own have been undertaken on handling potentially noisy answers from crowdsourced workers due to lack of expertise~\cite{Shah:2017,Shah:2016},  (including this work); in practice it is also crucial to understand how to incentivize workers to answer truthfully~\cite{Shah:2015}.
Our work also addresses directly the practical problem proposed by Chung et al.~\cite{Chung:2017}, to issue queries to potentially unreliable crowdsourced workers in order to estimate the fraction of records containing ``wrong" data within a database; here adaptive queries are a natural capability of the model.

\section{The Triangular Walk Algorithm}
\label{sect:TWalk}

In this section, we present the \emph{Triangular Walk} algorithm (Algorithm~\ref{Alg:Main}) for the problem, in the regime where both $\rho$ and the coin biases are unknown.
This is an important subroutine of our main, optimal algorithm; and the Triangular Walk algorithm itself can be used as an estimator in its own right.
We demonstrate later in Section~\ref{sec:experiments}, with simulation results, that this algorithm offers practical advantages over the straightforward majority vote estimator mentioned in the introduction, as well as the state-of-the-art method used in practice.

The Triangular Walk algorithm leverages only single-coin adaptivity, and makes no use of cross-coin adaptivity.
At the heart of our algorithm is an estimator (Algorithm~\ref{Alg:Single}) that works coin-by-coin, in the regime $\Delta\geq\frac{1}{4}$; subsequently we show how to use this estimator to solve the general problem, with an arbitrary (but known) $\Delta$.

We describe an asymmetric estimator (Algorithm~\ref{Alg:Single}) that, given sampling access to a single coin of bias $p$, returns a real number whose expectation is in $[1\pm\frac{\eps}{2}]$ if $p\geq\frac{3}{4}$, and whose expectation is in $[\pm\frac{\eps}{2}]$ if $p\leq\frac{1}{4}$.
The estimator is asymmetric in the sense that it will quickly ``give up on" coins with $p\leq\frac{1}{4}$, taking only a constant number of samples from them in expectation, while it will more deeply investigate the rare and interesting case of $p\geq\frac{3}{4}$. Below, $c$ will be a constant that emerges from the analysis, where $c\log\frac{1}{\eps}$ coin flips  suffice to yield an empirical fraction of heads within $poly(\eps)$ of the ground truth, $p$.

\begin{algorithm}
\caption{Single-coin estimate}
\label{Alg:Single}
\vspace*{0.3em}
{\bf Given:} a coin of bias $p$, error parameter $\eps$
\vspace*{-.5em}
\begin{enumerate}
    \setlength{\itemsep}{-.3em}
    \item Let $n\leftarrow 0$ \hfill (\emph{representing the total number of coin flips so far})
    \item  Let $k\leftarrow 0$ \hfill (\emph{representing the total number of observed heads so far})
    \item Repeat:
    \vspace*{-.3em}
    \begin{enumerate}
        \setlength{\itemsep}{0em}
        \item Flip the coin, and increment $n\leftarrow n+1$
        \item If heads, increment $k\leftarrow k+1$
        \item If $2k\leq n$, {\bf return} 0 and {\bf halt} \hfill (\emph{majority of flips are tails, evidence that $p$ is small})
        \item If $n=c\log\frac{1}{\eps}$, {\bf return} $\min(4,\frac{n}{2k-n})$ and {\bf halt} \hfill (\emph{enough flips for concentration})
    \end{enumerate}
\end{enumerate}
\vspace*{-0.5em}
\end{algorithm}

Our overall algorithm robustly combines estimates from running Algorithm~\ref{Alg:Single} on many coins via the standard median-of-means technique. To deal with the general case when $\Delta$ might be much smaller than $\frac{1}{4}$, we ``simulate a $\frac{1}{4}$-quality coin" by running Algorithm~\ref{Alg:Single} not on individual flips, but rather on the majority vote of blocks of $\Theta(\frac{1}{\Delta^2})$ flips; this majority vote will convert a coin of bias $\leq\frac{1}{2}-\Delta$ to a simulated coin of bias $\leq\frac{1}{4}$, and symmetrically, convert a coin of bias $\geq\frac{1}{2}+\Delta$ to a simulated coin of bias $\geq\frac{3}{4}$.

\begin{algorithm}
\caption{Triangular walk algorithm}
\label{Alg:Main}
\vspace*{0.3em}
{\bf Given:} $t$ coins, quality parameter $\Delta$, error parameter $\eps$, and failure probability parameter $\delta$
\vspace*{-0.5em}
\begin{enumerate}
\setlength{\itemsep}{0em}
    \item For each coin: simulate a new ``virtual" coin by computing the majority of $\Theta(\frac{1}{\Delta^2})$ flips each time a ``virtual" flip is requested; run  Algorithm~\ref{Alg:Single} on each virtual coin, using, inputting $\eps$ unchanged, and record the returned estimates.
    \item Partition the returned estimates into $\Theta(\log\frac{1}{\delta})$ groups and compute the mean of each group.
    \item {\bf Return} the median of the $\Theta(\log\frac{1}{\delta})$ means, or 0 if any of the groups in step 2 are empty.
\end{enumerate}
\vspace*{-0.5em}
\end{algorithm}

\medskip\noindent{\bf Theorem~\ref{Thm:OneSizeFinal}.}
\emph{Given coins where a $\rho$ fraction of the coins have bias $\geq\frac{1}{2}+\Delta$, and $1-\rho$ fraction have bias $\leq\frac{1}{2}-\Delta$, then running
	Algorithm~\ref{Alg:Main} on $t=\Theta(\frac{\rho}{\eps^2}\log\frac{1}{\delta})$ randomly chosen coins will estimate $\rho$ to within an additive error of $\pm\eps$, with probability at least $1-\delta$, with an expected sample complexity of $O(\frac{\rho}{\eps^2\Delta^2}(1+\rho\log\frac{1}{\eps})\log\frac{1}{\delta})$.}\medskip

The rest of this section concerns the (relatively straightforward) proof of Theorem~\ref{Thm:OneSizeFinal}, via an analysis of Algorithms~\ref{Alg:Single} and~\ref{Alg:Main}; Section~\ref{sect:Framework} instead formulates a more general algorithmic framework that adds some perspective to Algorithm~\ref{Alg:Single}, and whose abstractions will be crucial to the lower bound analysis in Section~\ref{sect:UnknownLower}.

\medskip\noindent{\bf Intuition and analysis of Algorithm~\ref{Alg:Single}:}
Recall that Algorithm~\ref{Alg:Single} is designed to work for coins of constant noise-quality $\Delta$, namely, coins have bias either $\le \frac{1}{4}$ or $\ge \frac{3}{4}$, and nothing in between. Algorithm~\ref{Alg:Single} halts under two conditions: either the majority of observed flips have been tails---Step 3(c)---or our budget of coin flips (for that coin) is exhausted---Step 3(d). The first stopping condition is designed to make it more likely to halt early for negative coins (coins with bias $p\leq\frac{1}{4}$), even though \emph{all} coins may have a significant chance of halting early.
Importantly, the chance of Algorithm~\ref{Alg:Single} halting early depends on the coin's bias $p$, which is a priori unknown.
The output coefficients in Step 3(d) are designed so that the expected output, given any negative coin (of bias $\leq \frac{1}{4}$), is close to 0, and similarly close to 1 given a positive coin (of bias $\geq\frac{3}{4}$).
Furthermore, the output coefficients are all bounded by a constant, which gives a constant bound on the variance of the estimate.

Lemma~\ref{Lem:AlgSingle} captures the guarantees we need from Algorithm~\ref{Alg:Single} in order to analyze the triangular walk algorithm, Algorithm~\ref{Alg:Main}.

\begin{lemma}
\label{Lem:AlgSingle}
	If Algorithm~\ref{Alg:Single} is run with a sufficiently large universal constant $c$, then the following statements hold.
	\begin{enumerate}
		\item Given an arbitrary negative coin (having bias $p \le \frac{1}{4}$), the output of Algorithm~\ref{Alg:Single} has expectation in $[\pm \frac{\eps}{2}]$ and variance upper bounded by $\eps^2$.
		Furthermore, the expected sample complexity in this case is upper bounded by a constant.
		\item Given an arbitrary positive coin (having bias $p \ge \frac{3}{4}$), the output of Algorithm~\ref{Alg:Single} has expectation in $[1\pm\frac{\eps}{2}]$ and variance upper bounded by a constant.
		The expected sample complexity in this case is (trivially) upper bounded by $c \log\frac{1}{\eps}$.
	\end{enumerate}
	The overall expected sample complexity, when the fraction of positive coins is $\rho$, is $O(1+\rho\log\frac{1}{\eps})$.
\end{lemma}

\begin{proof}
Consider running Algorithm~\ref{Alg:Single} on a coin of bias $p$, and let $\nmax=c\log\frac{1}{\eps}$ be the number of coin flips after which the algorithm always halts in Step 3(d). Consider running Algorithm~\ref{Alg:Single} on a sequence of $\nmax$ flips of the coin (even if the algorithm may halt early before exhausting the sequence of flips). If the sequence is majority-tails, then the algorithm must halt early via Step 3(c) at some point, and thus return 0.

For a negative coin, namely with bias $p\leq\frac{1}{4}$, the chance of observing majority-heads after $c\log\frac{1}{\eps}$ coin flips is $\eps^{O(c)}$, and we choose $c$ so that this probability is $O(\eps^2)$, so that (given that estimates returned by Algorithm~\ref{Alg:Single} are always bounded by 4), for negative coins, the expected estimate of Algorithm~\ref{Alg:Single} is in $[0,\frac{\eps}{2}]$ and the variance is at most $\eps^2$.

By contrast, for a positive coin, with bias $p\geq\frac{3}{4}$, the fraction of observed heads after $c\log\frac{1}{\eps}$ flips will concentrate around $p\geq \frac{3}{4}$. The challenge is to choose nonzero output coefficients in Step 3(d) of Algorithm~\ref{Alg:Single} that will average out to 1 in expectation, despite the fact that many of these sequences of coin flips will lead Algorithm~\ref{Alg:Single} to terminate early in Step 3(c) and output 0. Moreover, as mentioned earlier, the proportion of sequences that will halt early depends on $p$ itself, which is a priori unknown.

The key idea, from standard results on random walks, is that, \emph{conditioned on} $k$ out of $\nmax$ flips landing heads, the probability of reaching $\nmax$ flips---without ever halting in Step 3(c) by having a temporary majority of tails---is \emph{independent} of $p$, and is in fact expressed by the formula $\frac{2k-\nmax}{\nmax}$. Conditioned on $k$ out of $\nmax$ flips being heads, whether Algorithm~\ref{Alg:Single} halts early depends only on the permutation of the coin flips, and each such permutation of $k$ heads out of $\nmax$ flips is \emph{equally} likely. We thus apply the following standard random walk result to derive the aforementioned formula---where heads is interpreted as a $+1$ step in a 1-D random walk, tails is interpreted as a $-1$ step, and observing $k$ out of $n$ heads is analogous to reaching position $v=2k-n$ in the random walk.

\begin{fact}[The Ballot Theorem]
	\label{Fact:Ballot}
	Consider a 1-D walk that starts at the origin, and moves one step in either the positive or negative direction at each time.
	The number of paths from the origin that end at $v$ at time $\nmax$, which do not revisit the origin, is a $\frac{|v|}{\nmax}$ fraction of the total number of paths from the origin to $v$ at time $\nmax$.
\end{fact}

Algorithm~\ref{Alg:Single}, in Step 3(d), returns $\min(4,\frac{\nmax}{2k-\nmax})$, which equals $\frac{\nmax}{2k-\nmax}$ when $k\geq\frac{5}{8} \nmax$. In light of Fact~\ref{Fact:Ballot}, in Algorithm~\ref{Alg:Single}, conditioned on $k \ge \frac{5}{8} \nmax$ out of $\nmax$ coin flips being heads, the nonzero coefficient $\frac{\nmax}{2k-\nmax}$ will be output in Step 3(d) with probability $\frac{2k-\nmax}{\nmax}$, and thus the conditional expected output is exactly 1. For $p\geq\frac{3}{4}$, the probability of $k\geq\frac{5}{8} \nmax$ flips being heads is $1-\eps^{O(c)}$ by our choice of $\nmax$.
The expected output of Algorithm~\ref{Alg:Single} will therefore be within $\eps/2$ of 1 for a sufficiently large choice of the constant $c$.

The variance of Algorithm~\ref{Alg:Single} given a positive coin is clearly upper bounded by a constant, simply because the output coefficients are bounded by 4.

We lastly analyze the expected sample complexity of Algorithm~\ref{Alg:Single}, run on negative and positive coins.
For positive coins, we can simply upper bound the sample complexity by $\nmax = \Theta(\log\frac{1}{\eps})$, which is tight if the coin has bias $p = 1$.
For negative coins, even ignoring the halting conditions, the probability of getting $> \frac{n}{2}$ heads after $n$ coin flips decreases exponentially in $n$. Since the algorithm halts if it ever observes majority-tails, proceeding for many flips becomes exponentially unlikely. Thus the expected number of flips of the algorithm is bounded by a constant in the case of a negative coin.
\end{proof}

\medskip\noindent{\bf Analyzing Algorithm~\ref{Alg:Main}:} We conclude by proving Theorem~\ref{Thm:OneSizeFinal}, which analyzes Algorithm~\ref{Alg:Main}.

\begin{proof}[Proof of Theorem~\ref{Thm:OneSizeFinal}]
For this proof, we assume that $\rho = \Omega(\eps^2)$.
Otherwise, the case is handled in Step 3 of Algorithm~\ref{Alg:Main}, which returns the valid estimate of 0.

At a high-level, Algorithm~\ref{Alg:Main} runs Algorithm~\ref{Alg:Single} repeatedly on independently chosen coins.

Observe that in Step 1 of Algorithm~\ref{Alg:Main}, for each coin we simulate a new ``virtual" coin, by using the majority vote of $\Theta(1/\Delta^2)$ coin flips to compute each requested coin flip.
By Chernoff bounds, if each given coin has bias either $p \le \frac{1}{2}-\Delta$ or $p \ge \frac{1}{2}+\Delta$, then the corresponding virtual coin will have bias $p \le \frac{1}{4}$ and $p \ge \frac{3}{4}$ respectively.
Therefore, by Lemma~\ref{Lem:AlgSingle}, the output of Step 1 for each coin is a random variable with expectation in $[\rho \pm \frac{\eps}{2}]$.
As for the variance of the output, we do the following calculation.
Let $X_0$ denote the random variable that is the output of Algorithm~\ref{Alg:Single} when given a random \emph{negative} coin, and similarly $X_1$ for a random positive coin.
The output of Algorithm~\ref{Alg:Single}, which we call $Y$, is thus distributed as $X_1$ with $\rho$ probability and as $X_0$ with $1-\rho$ probability.
The variance of $Y$ is
\begin{align*}
	\Var[Y] &= \rho \Var[X_1] + (1-\rho)\Var[X_0] + \Var_{i \from \Bernoulli(\rho)} [\Exp[X_i]]\\
	&\le O(\rho) + (1-\rho)\eps^2 + \rho (\Exp[X_1])^2 + (1-\rho) (\Exp[X_0])^2\\
	&\le O(\rho) + \eps^2 + O(\rho) + O(\eps^2)\\
	&= O(\rho)
\end{align*}

Steps 2 and 3 of Algorithm~\ref{Alg:Main} are the median-of-means method for estimating the mean of a (real-valued) random variable.
Using $t = \Theta(\frac{\rho}{\eps^2}\log\frac{1}{\delta})$ coins, each of the $\Theta(\log\frac{1}{\delta})$ groups will have $\Theta(\frac{\rho}{\eps^2})$ coins and hence outputs from Algorithm~\ref{Alg:Single}.
By Chebyshev's inequality, with constant probability, the sample mean of each group's estimates will be within $O(\sqrt{\frac{\eps^2}{\rho}})$ standard deviations of the expected output of Algorithm~\ref{Alg:Single}.
The estimation error is therefore equal to $O(\eps)$, with a multiplicative constant that can be made arbitrarily small by adjusting the constant in the choice of the number of coins $t$.
Step 3 computes the median of $\Theta(\log\frac{1}{\delta})$ such sample means, which boosts the success probability from constant to $1-\delta$, via standard uses of Chernoff bounds.

Lastly, the total expected sample complexity is the product of 1) the choice of $t$ in the theorem statement, 2) $\Theta(1/\Delta^2)$ which is the number of coin flips used for each majority vote in Step 1, and 3) the sample complexity of Algorithm~\ref{Alg:Single} as stated in Lemma~\ref{Lem:AlgSingle}, yielding $O(\frac{\rho}{\eps^2\Delta^2}(1+\rho\log\frac{1}{\eps})\log\frac{1}{\delta})$.
\end{proof}

While Theorem~\ref{Thm:OneSizeFinal} gives $\eps$ as input to Algorithm~\ref{Alg:Main} and then asks how many coins are needed to achieve this $\eps$ error, it will be useful as a preliminary step of our optimal Algorithm~\ref{Alg:Optimal} to consider the performance of Algorithm~\ref{Alg:Main} where these two roles for $\eps$ are decoupled. Explicitly, how many coins or samples does it take for Algorithm~\ref{Alg:Main} to achieve error $\eps_1$, when Algorithm~\ref{Alg:Main} is given $\eps_2$ as input? We will use this result in the regime where the failure probability for Algorithm~\ref{Alg:Main} should be a constant, and thus for simplicity we omit $\delta$ from the following statement.

\begin{corollary}
	\label{Cor:OneSizeFinal}
	Given coins where a $\rho$ fraction of the coins have bias $\geq\frac{1}{2}+\Delta$, and $1-\rho$ fraction have bias $\leq\frac{1}{2}-\Delta$, then, for parameters $\eps_1,\eps_2>0$, running
	Algorithm~\ref{Alg:Main} on $t=\Theta(\frac{\rho}{\eps_1^2})$ randomly chosen coins with parameter $\eps = \eps_2$ will estimate $\rho$ to within an additive error of $\pm \eps_1$, with failure probability at most $0.1+O(t\cdot\poly(\eps_2))$, with an expected sample complexity of $O(\frac{\rho}{\eps_1^2\Delta^2}(1+\rho\log\frac{1}{\eps_2}))$.
	Note that the degree of the polynomial term (in $\eps_2$) in the failure probability can be made arbitrarily high, by choosing a large constant $c$ in Step 3(d) of Algorithm~\ref{Alg:Single}.
\end{corollary}

\begin{proof}(Sketch)
    The proof is essentially the same as that of Theorem~\ref{Thm:OneSizeFinal}.
    
    The crucial difference is that, instead of interpreting the $\eps$ parameter of Algorithm~\ref{Alg:Main} as an (additive) error parameter for the produced estimate, we interpret $\eps$ (which we parameterize as $\eps_2$ in the corollary statement) as a ``coin misclassification probability".
    We explain this in more detail.
    
    As shown in the proof of Theorem~\ref{Thm:OneSizeFinal}, with $\nmax = \Theta(\log(1/\eps_2))$, the probability that Algorithm~\ref{Alg:Single} produces a non-0 estimate for a negative coin is $\poly(\eps_2)$.
    As for a positive coin, based on the analysis using Fact~\ref{Fact:Ballot}, the probability that $\nmax$ flips of a positive coin resulting in fewer than $k = \frac{5}{8}\nmax$ heads is also $\poly(\eps_2)$.
    Conditioned on such failure not happening, the expected value of Algorithm~\ref{Alg:Single} on a positive coin is \emph{exactly} 1.
    
    Therefore, we can interpret Algorithm~\ref{Alg:Main} as follows.
    Taking a union bound over the probabilities of the aforementioned failure modes, there is at most $O(t \cdot \poly(\eps_2))$ probability that any of the $t$ coins are ``misclassified".
    Conditioned on that not happening, Algorithm~\ref{Alg:Single} is just a (meta-)Bernoulli coin that flips positive with probability $\rho$ and negative with probability $1-\rho$, in expectation. Explicitly, assuming (as happens with probability $1-O(t \cdot \poly(\eps_2))$ that none of the $t$ coins are ``misclassified"), each negative coin will yield an output of exactly 0, and each positive coin will yield an output of expectation exactly 1 and constant variance. Thus, given that a $\rho$ fraction of coins from the underlying distribution are positive, the output will be exactly $\rho$ in expectation (except with $O(t \cdot \poly(\eps_2))$ misclassification probability) and has $O(\rho)$ variance.
    By Chebyshev's inequality, the mean output over $t=\Theta(\frac{\rho}{\eps_1^2})$ coins will be within $\pm\eps_1$, except with probability $0.1$ (choosing the multiplicative constant in the definition of $t$ appropriately), as desired.

    The variance of Algorithm~\ref{Alg:Single} has already been bounded by $O(\rho)$ as in Theorem~\ref{Thm:OneSizeFinal}, and so the performance of Algorithm~\ref{Alg:Main} can be analyzed by a straightforward application of Chebyshev's inequality, yielding the accuracy part of the corollary statement, as well as the 0.1 probability term in the failure probability.
\end{proof}

\subsection{Implementing Algorithm~\ref{Alg:Main}}
\label{sec:practice}

Later in Section~\ref{sec:experiments}, we give experimental results to demonstrate the performance of Algorithm~\ref{Alg:Main} in practice.
Here we address other concerns regarding the practical implementation and use of the algorithm.

The first concern is the fact that $\rho$, the ground truth that we are trying to estimate, appears in the sample complexity bound.
We note that, once we fix an error parameter $\eps$ and the constant $c$ in Algorithm~\ref{Alg:Single}, the overall algorithm of Algorithm~\ref{Alg:Main} is an \emph{anytime} algorithm: it can produce an estimate given any number of samples/coin flips, and the estimate simply gets more accurate (until it is as small as $\eps$) as it gets a larger sample size.
Crucially, the algorithm execution does not depend on the value of $\rho$ itself.
Thus, the fact that $\rho$ appears in the sample complexity has no bearing on the execution on the algorithm.
In Section~\ref{sec:optimal-algorithm}, we present our final algorithm (Algorithm~\ref{Alg:Optimal}) in a form where, given a budget $B$ of coin flips, the algorithm ``discovers" the correct $\eps$ based on the unknown answer $\rho$ and the tight sample complexity formula.
Algorithm~\ref{Alg:Main} can enjoy the same guarantee following its invocation in Algorithm~\ref{Alg:Optimal}.
The key insight is that, instead of fixing an $\eps$ for Algorithm~\ref{Alg:Single}, we use the budget $B$ to derive a cutoff for the maximum number of flips we invest in a single coin.

The second issue is on the practical parameter regime of the noise parameter $\Delta$.
In this work, we study the asymptotics of the sample complexity as $\Delta \to 0$, but in practice, the quality of yes/no questions being asked will have at least constant correlation with the truth.
To run our algorithms, then, we would ignore Step 1 of Algorithm~\ref{Alg:Main}, namely simulating ``virtual" coins from real coins, and use real coin flips directly in Algorithm~\ref{Alg:Single}.

The third practical concern is on the non-zero output coefficients in Algorithm~\ref{Alg:Single}, in Step 3(d).
In the algorithm and its subsequent analysis, we gave coefficients with a simple form of $\min(4,\frac{n}{2k-n})$, which allowed for a straightforward analysis with Chernoff bounds.
However, these output coefficients may not be the best possible, recalling that the objective of these coefficients is to make sure that the expected output of any underlying coin of bias $p \ge \frac{3}{4}$ is as close to 1 as possible.
A simple observation is that the expected output of Algorithm~\ref{Alg:Single} is in fact a smooth polynomial in $p$ with coefficients determined by the output coefficients in the algorithm.
Therefore, in practice, once we fix the constant $c$ (or the entire quantity $c\log\frac{1}{\eps}$) in the algorithm description (Step 3(d) again), we can run a local search/gradient-based method to find output coefficients that make the expected output polynomial as close to 1 as possible, with initial coefficients being the ones given in Algorithm~\ref{Alg:Single}.
These new output coefficients are reusable in practice as long as the noise parameter $\Delta$ in practice is not lower than the $\Delta$ for which the output coefficients were generated.

The fourth and last concern we address here is related to the first concern and the ``anytime algorithm" implementation of Algorithm~\ref{Alg:Main}.
The concentration results were phrased in terms of the number of coins that need to be sampled, namely the number of times Algorithm~\ref{Alg:Single} is called by Algorithm~\ref{Alg:Main}.
Each run of Algorithm~\ref{Alg:Single} then flips each coin an \emph{a priori} unknown number of times to produce estimates.
Since the sample complexity of a single triangular walk (Algorithm~\ref{Alg:Single}) is random, the concentration results only give an expected overall sample complexity for the algorithm.
On the other hand, in practice one may wish to impose a \emph{fixed} budget for sample complexity and simply use the entire budget.
Such an approach introduces the issue that the triangular walk started last will probably not have finished by the time the budget is exhausted.
How then can we aggregate the estimates obtained from the completed triangular walks without introducing bias in subtle ways?

Here we show the surprising result that the most natural algorithm does in fact work as is: that is, we take the average estimate of all the completed walks, ignoring the incomplete walk in progress.
As an example, to demonstrate that this success is unintuitive and nontrivial, if we instead separately run two executions of Algorithm~\ref{Alg:Main} with \emph{separate} budgets, and averaged the estimates of all completed walks across both executions, this average \emph{would} be biased; but if we computed the average of those walks completed under each budget, separately, then---by the main result of this section---each average would be unbiased, and we could average these averages together to yield an unbiased estimator.

To show that this ``budgeted estimator" is unbiased, we view it as the following two-stage estimator:
1) We estimate without bias the distribution over states $(n,k)$ that the triangular walk (Algorithm~\ref{Alg:Single}) terminates at, when given a randomly chosen coin from the universe, namely the numbers $\{\alpha_{n,k}(\rho h^{+}_{n,k} + (1-\rho) h^{-}_{n,k})\}$ (using notation defined in Definition~\ref{Def:Alpha}).
2) We simply take the dot product of this distribution with the corresponding output values $\{v_{n,k}\}$ (as defined in Algorithm~\ref{Alg:Single}).

In order to perform step 1), that is to estimate the distribution of termination over the states, we use the estimator $i_{(n,k)}/\sum_{m,j} i_{(m,j)}$ where $i_{(n,k)}$ is the number of observed walks that terminated at $(n,k)$, ignoring incomplete walks.
The following proposition shows that the estimation in step 1) is unbiased, from which it follows that the entire estimator is indeed also unbiased.

\begin{proposition}\label{prop:budget}
	Given a budget $T>0$, and suppose we repeatedly run an adaptive algorithm $A$ on a single coin until we have flipped the coin $T$ times in total. Given a set of outcomes for the algorithm $A$, indexed by $k\in\{1,\ldots,K\}$, let $p_k$ be a probability distribution over outcomes, and let $t_k$ denote the number of coin flips taken to reach this outcome.
	When an outcome using $t$ coin flips is drawn, if $t$ is less than or equal to the remaining budget, then $t$ is subtracted from the remaining budget; and otherwise the most recent outcome is discarded as ``over budget" and the algorithm terminates.
	Let $i_k$ be the number of times that outcome $k$ is drawn.
	Then $i_k/\sum_j i_j$ is an unbiased estimator of $p_k$.
\end{proposition}

\begin{proof}
	Given the coin budget $T$, the possible sequences of samples can be classified into the following cases.
	Either 1) the sequence ends exactly at time $T$, or 2) the sequence ends with a time interval of length $t_m$ for some $m$, which in turn ends \emph{after} time $T$. For a vector $\vec{i}$, whose $k^\text{th}$ index denotes the number of times outcome $k$ occurs, the dot product with vector $\vec{t}$ counts the total number of coin flips used by this sequence. Thus, if $\vec{i}\cdot\vec{t}=T$, then the probability of $\vec{i}$ occurring equals \[{i_1+\cdots+i_K \choose i_1;\cdots;i_K} p_1^{i_1}\cdots p_K^{i_K}\]
	
	This expression captures all cases where we use \emph{exactly} our budget $T$. In the remaining cases, there is a final (discarded) outcome $m$ that goes ``over budget". In this case, $\vec{i}\cdot\vec{t}\in[T-t_m+1,T-1]$, and the probability of observing $\vec{i}$ and discarding $m$ equals \[{i_1+\cdots+i_K \choose i_1;\cdots;i_K} p_1^{i_1}\cdots p_K^{i_K} \cdot p_m\]

	Therefore, the expectation of $i_k/\sum_j i_j$ can be written as
	\begin{align*}
		&\sum_{\substack{\text{vector }\vec{i}\\\text{s.t. } \vec{i}\cdot\vec{t} = T}} {i_1+\cdots+i_K \choose i_1;\cdots;i_K} p_1^{i_1}\cdots p_K^{i_K} \cdot \frac{i_k}{i_1+\cdots+i_K}\\
		+&\sum_m \sum_{\substack{\text{vector }\vec{i}\\\text{s.t. } \vec{i}\cdot\vec{t} \in [T-t_m+1,T-1]}} {i_1+\cdots+i_K \choose i_1;\cdots;i_K} p_1^{i_1}\cdots p_K^{i_K} \cdot p_m \cdot\frac{i_k}{i_1+\cdots+i_K}
	\end{align*}
	Now observe that
	$$ {i_1+\cdots+i_K \choose i_1;\cdots;i_K} \frac{i_k}{i_1+\cdots+i_K} = {i_1+\cdots + (i_k-1) + \cdots +i_K \choose i_1;\cdots;i_k-1;\cdots;i_K} $$
	meaning that, letting the vector $\vec{i'}$ equal the vector $\vec{i}$ with its $k^\text{th}$ entry decreased by 1, the expectation can be rewritten and simplified as
	$$ p_k\left[\sum_{\substack{\vec{i'}\\\text{s.t. } \vec{i'}\cdot\vec{t} = T'}} {i'_1+\cdots+i'_K \choose i'_1;\cdots;i'_K} p_1^{i'_1}\cdots p_K^{i'_K} + \sum_m \sum_{\substack{\vec{i'}\\\text{s.t. } \vec{i'}\cdot\vec{t} \in [T'-t_m+1,T'-1]}} {i'_1+\cdots+i'_K \choose i'_1;\cdots;i'_K} p_1^{i'_1}\cdots p_K^{i'_K} \cdot p_m\right] $$
	where $T' = T-t_k$.
	The term inside the square brackets sums to 1, as we observed at the beginning of the proof, but substituting $T'$ for $T$.
	Thus the expectation is $p_k$, as desired.
\end{proof}

\section{The Main Algorithm}
\label{sec:optimal-algorithm}

Here we present our main algorithm,  Algorithm~\ref{Alg:Optimal}, analyzed in Theorem~\ref{thm:opt}, which uses $O(\frac{\rho}{\eps^2\Delta^2}\log\frac{1}{\delta})$ samples, matching the fully-adaptive lower bound we prove in Section~\ref{sect:UnknownLower}.

Algorithm~\ref{Alg:Optimal} uses the Triangular Walk estimator as a subroutine and has a hybrid flavor, combining both (single-coin) adaptive and non-adaptive techniques, where the algorithm is increasingly adaptive for smaller values of $\rho$.
Crucially, in the adaptive component of Algorithm~\ref{Alg:Optimal}, we use the Triangular Walk estimator to provide a $2$-approximation to $\rho$, and a variant of the algorithm to ``filter" out most negative coins such that we get a constant ratio of positive vs negative coins, to reduce variance.
The coins ``surviving" the filter are then fed into a new, non-adaptive algorithm (Algorithm~\ref{Alg:RefinedSampling}) that we call ``refined sampling", which like Algorithm~\ref{Alg:Single} flips different coins a \emph{different} number of times, yet the number of flips is chosen \emph{non-adaptively}; the information from different coins is combined in a subtle way.

As a general motivation, consider taking $t$ coins, flipping them $n$ times each, and trying to estimate the fraction of positive coins.
For a slightly different setting that may have cleaner intuition, consider having sample access to many univariate Gaussian distributions of bounded variance, some of which have mean $\leq 0$ and some of which have mean $\geq 1$, where the goal is to estimate the fraction of ``positive" Gaussians with as few samples as possible.
If we take $n$ samples from a given distribution, then testing whether the sample mean is $>\frac{1}{2}$ lets us correctly determine its identity with probability  $1-\exp(-n)$, incentivizing us to choose a large $n$.
However, for a fixed budget on the total number of samples across all distributions, choosing many samples per distribution means we can only sample from a limited number of distributions, introducing sampling errors across distributions (as opposed to within each distribution), and thus introducing a variance into our estimate inversely proportional to the number of coins sampled, and thus $O(n/T)$ for a total budget of $T$.
This is the classic bias-variance tradeoff, where larger $n$ induces a better bias but worse variance.

While in many settings, one might try to find an optimal $n$ that balances these two concerns, the right answer here is instead to combine the two approaches: sample some distributions many times, to get a low-bias signal, and also sample many distributions a few times, to get a low-variance signal; and combine these two signals with care.
Explicitly, the coefficients in Step 3 of Algorithm~\ref{Alg:RefinedSampling} are carefully chosen so that their contributions ``telescope" in expectation between distributions sampled different numbers of times, allowing, essentially, all the high-variance terms to cancel out without worsening the bias.

We first present the non-adaptive component (Algorithm~\ref{Alg:RefinedSampling}) of Algorithm~\ref{Alg:Optimal} for estimating smooth functions $f$ on the underlying coin bias $p$, which has constant expected sample complexity, with \emph{zero bias}, at the cost of $O(1)$ variance instead of $O(\rho)$ variance as in Algorithm~\ref{Alg:Main}.
This will be combined with a single-coin adaptive ``filtering" component such that only an $O(\rho)$ fraction of coins will be used in running Algorithm~\ref{Alg:RefinedSampling}, giving an overall $O(\rho)$ variance in Algorithm~\ref{Alg:Optimal}.

Think of the function $f$ as being analogous to the output coefficient $\frac{n}{2k-n}$ of Algorithm~\ref{Alg:Single}---correcting for a probabilistic filtering mechanism, such that the expected output of $f$ for those coins that survive filtering will be essentially $0$ for negative coins ($p\leq\frac{1}{4}$), 1 for positive coins ($p\geq \frac{3}{4}$), and smoothly transitions between 0 and 1 in between.
See later in Definition~\ref{def:f(p)} for the precise instantiation of $f(p)$ we need.

Let $\Bin(n,p,k)$ denote the probability that a Binomial distribution with $n$ trials and bias $p$ outputs $k$.

\begin{algorithm}
\caption{Refined Sampling}
\label{Alg:RefinedSampling}
\vspace*{0.3em}
{\bf Input:} sample access to a coin of bias $p$;  target function $f:[0,1]\rightarrow\mathbb{R}$
\vspace*{-0.5em}
\begin{enumerate}
\setlength{\itemsep}{0em}
\setlength{\itemindent}{-2.5mm}
	\item Choose a number of coin flips $n$ that is a power of 2, choosing $2^{i}$ with probability $\frac{\sqrt{8}-1}{\sqrt{8}} (2^i)^{-1.5}$, where $\frac{\sqrt{8}-1}{\sqrt{8}}$ is the normalizing constant so that the probabilities sum to 1.
	\item Flip the coin $n$ times, and let $k$ be the number of observed heads.
	\item Return $\frac{n^{1.5}\sqrt{8}}{\sqrt{8}-1}\left(f(\frac{k}{n})-\sum_{i=0}^{n/2} f(\frac{i}{n/2})\cdot {n/2\choose i}{n/2\choose k-i}/{n\choose k}\right)$
\end{enumerate}
\vspace*{-0.5em}
\end{algorithm}

The sum in Step 3 of the algorithm is omitted if the power of 2 chosen for the number of coin flips is $n=1$, in which case $ {n/2\choose i}$ would be undefined.
We now describe the properties of Algorithm~\ref{Alg:RefinedSampling} in Lemma~\ref{lem:refined-sampling}.

\begin{lemma}\label{lem:refined-sampling}
Given a coin of bias $p$, and given a function $f:[0,1]\rightarrow\mathbb{R}$ that is bounded by a universal constant, and has 2nd derivative bounded by a universal constant, then Algorithm~\ref{Alg:RefinedSampling} will return an estimate of $f(p)$ that has bias 0, variance $O(1)$, and uses $O(1)$ samples in expectation.
\end{lemma}

\begin{proof}

The expected number of coin flips taken by Algorithm~\ref{Alg:RefinedSampling} is the sum of a fixed geometric series, and is thus $O(1)$ as desired.

We bound the variance of the algorithm by showing that, for each depth $n$, the values returned in Step 3 will have magnitude $O(n^{0.5})$. Consider the sum in the second term of the expression of Step 3. The expression ${n/2\choose i}{n/2\choose k-i}/{n\choose k}$ can be interpreted as: given a sequence of $n$ coin tosses of which $k$ were heads, if a random subsequence of length $n/2$ is chosen, what is the probability that $i$ heads are chosen. This distribution has expectation $\frac{k}{2}$, and variance $<n$.
Since $f$ has second derivative bounded by a constant, the difference of $f$ from $f(\frac{k}{n})$ is upper and lower bounded by quadratics centered at $\frac{k}{n}$.
Thus the difference between $f(\frac{k}{n})$ and the expected value of $f(\frac{i}{n/2})$ when $i$ is drawn from the distribution with pmf ${n/2\choose i}{n/2\choose k-i}/{n\choose k}$ is bounded by a constant times the variance of the random variable $\frac{i}{n/2}$, namely $O(\frac{1}{n})$.
Therefore, when multiplied by $\frac{n^{1.5}}{\sqrt{8}-1}$, the output of Step 3 will be bounded by $O(n^{0.5})$ as desired. Since in Step 1, $n$ is chosen with probability $\frac{\sqrt{8}-1}{n^{1.5}}$, the contribution to the variance from a particular $n$ is at most $\frac{\sqrt{8}-1}{n^{1.5}} O(n^{0.5})^2=O(n^{-0.5})$; summing this bound over all $n$ that are powers of 2 yields a constant, $O(1)$, variance, since geometric series converge.

To analyze the expectation of the values  returned in Step 3 of Algorithm~\ref{Alg:RefinedSampling}, we show that it telescopes across the different depths $n$. Namely, consider the expected contribution just of the second (sum) term at level $n$, $-\sum_{k=0}^n\Bin(n,p,k)\sum_{i=0}^{n/2} f\left(\frac{i}{n/2}\right)\cdot {n/2\choose i}{n/2\choose k-i}/{n\choose k}$. The coefficient in this expression of a given $f(\frac{i}{n/2})$ equals $-\sum_{k=0}^n\Bin(n,p,k) {n/2\choose i}{n/2\choose k-i}/{n\choose k}$; from the discussion at the start of the proof, the $k^\textrm{th}$ term of this sum can be reinterpreted as the probability that, in $n$ tosses of a coin of bias $p$, we have $k$ heads total, and $i$ heads among the first $n/2$ tosses; summed over all $k$ this is clearly just the probability that $i$ heads will be observed among $n/2$ tosses, namely $\Bin(\frac{n}{2},p,i)$. Thus the expected value of the sum term of Step 3 at level $n$ is $-\sum_{i=0}^{n/2}\Bin(\frac{n}{2},p,i)f(\frac{i}{n/2})$, which is exactly the negation of the expectation of the first term of Step 3, at level $n/2$. (The multiplier $\frac{n^{1.5}\sqrt{8}}{\sqrt{8}-1}$ in Step 3 is exactly canceled out by the probability of choosing $n$ in Step 1.)

Thus the expected output of the algorithm, considering only contributions up to some depth $n=2^i$, collapses to just the expectation of the first term of Step 3 at the deepest level, $n$.
This expected output is thus $\sum_{k=0}^n f(\frac{k}{n}) \cdot \Bin(n,p,k)$, namely the expected value of $f(\frac{k}{n})$ when $k$ is drawn from a binomial distribution with $n$ trials and bias $p$.
Since the binomial distribution $\Bin(n,p,\cdot)$ has expectation $pn$ and variance $<n$, and since $f$ has 2nd derivative bounded by a constant, we have that this expectation converges to $f(p)$ for large $n$; namely,  $|f(p)-\sum_{k=0}^n f(\frac{k}{n}) \cdot \Bin(n,p,k)| = O(\frac{1}{n})$. Thus, as $n$ goes to infinity, we see that the expected output of Algorithm~\ref{Alg:RefinedSampling} converges to $f(p)$, as claimed.
\end{proof}

We now give a new non-adaptive algorithm, Algorithm~\ref{Alg:Optimal-rho-hat}, in order to motivate the choice of $f(p)$ that we use for Algorithm~\ref{Alg:RefinedSampling} within Algorithm~\ref{Alg:Optimal-rho-hat}.
Algorithm~\ref{Alg:Optimal-rho-hat} will be a major component of our final algorithm, Algorithm~\ref{Alg:Optimal}.
\begin{algorithm}
\caption{Optimal Algorithm given an estimate $\hat{\rho}$}
\label{Alg:Optimal-rho-hat}
\vspace*{0.3em}
{\bf Given:} A total budget $B$ of coin flips, quality parameter $\Delta$, and an estimate $\hat{\rho}$ that is within a factor of 2 of $\rho$
\vspace*{-0.5em}
\begin{enumerate}
\setlength{\itemsep}{0em}
    \item Run the following on $t=\Theta(\Delta^2 B)$ randomly drawn coins.
    For each coin: simulate a new ``virtual" coin by computing the majority of $\Theta(\frac{1}{\Delta^2})$ flips each time a ``virtual" flip is requested, so that each virtual coin will have probability either $p\leq\frac{1}{4}$ or $p\geq \frac{3}{4}$.
    \begin{enumerate}
    \item For each virtual coin, flip it at most $d=\Theta(\log\frac{1}{\hat{\rho}})$ times but stop if at any point the majority of flips are tails.
    \item If the previous step did not stop early, then run Algorithm~\ref{Alg:RefinedSampling} for the function $f_d(p)$ of Definition~\ref{def:f(p)}.
    \end{enumerate}
    \item Return $\frac{1}{t}$ times the sum of all the values output by Algorithm~\ref{Alg:RefinedSampling} in Step 3(b).
\end{enumerate}
\vspace*{-0.5em}
\end{algorithm}

As mentioned above, the choice of $f(p)$ is a correction for the filtering mechanism.
Concretely, in Algorithm~\ref{Alg:Optimal-rho-hat}, Step 2(a) will stop early on negative coins with probability that is increasingly high for smaller $\rho$, significantly reducing the number of coin flips; and in Step 2(b) we exactly compensate for this (a priori) unknown early stopping probability by running the unbiased Algorithm~\ref{Alg:RefinedSampling} on an appropriately chosen function $f_d(p)$ that is exactly the inverse of this early stopping probability, for positive coins, and 0 for negative coins:

\begin{definition}\label{def:f(p)}
Given a depth $d$, let $f_d(p):[0,1]\rightarrow\mathbb{R}$ be defined to equal 0 for $p\leq \frac{1}{4}$; and for $p\geq \frac{3}{4}$, let $f_d(p)$ equal 1 divided by the probability that a sequence of $d$ flips of a coin of bias $p$ never has a majority-tails initial sequence; for $\frac{1}{4}<p<\frac{3}{4}$, let $f_d(p)$ be chosen so as to smoothly connect the regions $p\leq\frac{1}{4}$ and $p\geq\frac{3}{4}$ so that $f_d(p)$ has second derivative bounded by a universal constant (independent of $d$).
\end{definition}

With this choice of $f(p)$, we state and prove Proposition~\ref{prop:opt-rho-hat}, which gives the soundness and sample complexity bounds for Algorithm~\ref{Alg:Optimal-rho-hat}.

\begin{proposition}
\label{prop:opt-rho-hat}
On input 1) a budget $B = O(\frac{\rho}{\eps^2\Delta^2})$ of coin flips, 2) the quality parameter $\Delta$ and 3) a 2-approximation $\hat{\rho}$ of $\rho$, Algorithm~\ref{Alg:Optimal-rho-hat} returns an estimate of $\rho$ that has additive error at most $\eps$ with probability at least $0.99$, using at most $B$ coin flips.
\end{proposition}

\begin{proof}
We first show the expected output of Algorithm~\ref{Alg:Optimal-rho-hat} equals $\rho$. For each positive coin, Step 1 transforms it into a ``virtual" coin of probability $p\geq\frac{3}{4}$; this coin will ``survive" Step 1(a) with probability exactly $1/f_d(p)$, by definition of $f_d(p)$ in Definition~\ref{def:f(p)}. Thus Algorithm~\ref{Alg:RefinedSampling} will return an estimate of $f_d(p)$, with bias 0. Multiplying through by the survival probability $1/f_d(p)$, and by the probability $\rho$ that a positive coin will be drawn, we see that, over $t$ coins, the expected contribution to the estimate from Step 2 of the \emph{positive} coins will be $\rho$. For each negative coin, by definition $f_d(p)=0$, so the expected contribution from these coins, added over all $\leq t$ of them, and scaled by $\frac{1}{t}$ in Step 2, will be $0$.

To bound the variance of the output of Step 2, we note that at most a $2\rho$ fraction of the coins reach Step 1(b): a $\rho$ fraction of the coins are positive; meanwhile, negative coins, where $p\leq \frac{1}{4}$, have $\exp(-d)$ probability of surviving Step 1(a), which can be made $\leq \rho$ since $d=\Theta(\log\frac{1}{\hat{\rho}})$. Thus the output returned in Step 2 is $\frac{1}{t}$ times the sum of $t$ independent trials of a process that, with probability $\leq 2\rho$ outputs a random variable whose expected squared magnitude is bounded by a constant (by Lemma~\ref{lem:refined-sampling}). For $t=\Theta(\frac{\rho}{\eps^2})$, the expected squared magnitude---and hence the variance---of the output of the algorithm is thus bounded by $O(\frac{t\rho}{t^2})=O(\eps^2)$.
Thus by Chebyshev's inequality, Step 2 will return an estimate accurate to within $O(\eps)$, with constant probability. 

Lastly, we need to verify that Step 1 will exceed the coin flip budget only with small constant probability.
It suffices, using a Markov's inequality argument, to bound the expected number of coin flips used in the steps.
We consider the number of (``virtual") flips from Step 1(a), and also Step 1(b), and then multiply by $\Theta(\frac{1}{\Delta^2})$ as described in Step 1. For a negative coin, the expected number of flips until a majority-tails initial sequence is observed in Step 1(a) is constant by standard random walk analysis, leading to an $O(\Delta^2 B)$ term; for positive coins, there are on average $O(\rho \Delta^2 B)$ of them, so we could afford to flip each $O(\frac{1}{\rho})$ times, but Step 1(a) uses only at most $d=\Theta(\log\frac{1}{\hat{\rho}})$ flips. Step 1(b) is run on an expected $\leq 2\rho$ fraction of the coins, as explained above; and by Lemma~\ref{lem:refined-sampling}, Algorithm~\ref{Alg:RefinedSampling} takes $O(1)$ expected samples, for a total bound of $O(\rho \Delta^2 B)$ virtual flips from Step 1(b). (Algorithm~\ref{Alg:RefinedSampling} could thus afford to take up to $O(\frac{1}{\rho})$ samples on average, so, interestingly, there is a lot of slack here.)

Thus in total we use $O(\Delta^2 B) = O(\frac{\rho}{\eps^2})$ virtual flips, each requiring $\Theta(1/\Delta^2)$ real flips, corresponding to expected sample complexity of $O(B) = O(\frac{\rho}{\eps^2\Delta^2})$.
\end{proof}

Having analyzed Algorithm~\ref{Alg:Optimal-rho-hat}, we can now present our final optimal algorithm, stated as Algorithm~\ref{Alg:Optimal}.
The theoretical guarantees are given in Theorem~\ref{thm:opt}, restated and proved below.

We stress again that, in our presentation of Algorithm~\ref{Alg:Optimal}, the error parameter $\eps$ (of Theorem~\ref{thm:opt}) is not known, since it depends on the budget $B$ and the unknown ground truth $\rho$, yet the returned estimate will have this optimal $\eps$ accuracy regardless. 
This is achieved by Algorithm~\ref{Alg:Optimal}'s calls to Algorithm~\ref{Alg:Main} and Algorithm~\ref{Alg:Optimal-rho-hat}, which collectively cover all regimes of how $\rho$ and $\eps$ relate to each other, yielding optimal error guarantees in each case.

\begin{algorithm}[h]
\caption{Optimal Algorithm}
\label{Alg:Optimal}
\vspace*{0.3em}
{\bf Given:} A total budget $B$ of coin flips and quality parameter $\Delta$
\vspace*{-0.5em}
\begin{enumerate}
\setlength{\itemsep}{0em}
    \item Use Algorithm~\ref{Alg:Main} in Section~\ref{sect:TWalk} on $O(\Delta^2 B)$ many coins (a small fraction of $B$), using an ``$\eps$" that is $\Theta(1/(\Delta^2 B))$, and a constant $\delta$. 
    Let $\hat{\rho}$ be the returned estimate of $\rho$.
        \begin{enumerate}
            \item If Algorithm~\ref{Alg:Main} ever tries to use more than $B/4$ coin flips total, then terminate Algorithm~\ref{Alg:Main} and move onto the next step. 
            \item Otherwise, return the estimate produced by Algorithm~\ref{Alg:Main}.
        \end{enumerate}
    \item Use Algorithm~\ref{Alg:Main} on $\Theta(\sqrt{\Delta^2 B})$ freshly drawn coins, using again an ``$\eps$" that is $\Theta(1/(\Delta^2 B))$, and a constant $\delta$.
    The returned estimate $\hat{\rho}$ will be a 2-approximation to $\rho$.
    If in this step, Algorithm~\ref{Alg:Main} tries to use more than $B/4$ coin flips, terminate and fail, which happens only with small constant probability.
    \item Run Algorithm~\ref{Alg:Optimal-rho-hat} on input $B/2$, $\Delta$, and $\hat{\rho}$, and return its answer.
    \item (If a sub-constant failure probability $\delta$ is desired, then repeat the entire algorithm $\Theta(\log\frac{1}{\delta})$ times and return the median of the outputs, ignoring invocations that failed.)
\end{enumerate}
\vspace*{-0.5em}
\end{algorithm}

\medskip\noindent{\bf Theorem~\ref{thm:opt}.}
\emph{Given coins where a $\rho$ fraction of the coins have bias $\geq\frac{1}{2}+\Delta$, and $1-\rho$ fraction have bias $\leq\frac{1}{2}-\Delta$, then running
Algorithm~\ref{Alg:Optimal} on a budget of $B$ coin flips will estimate $\rho$ to within an additive error of $\pm\eps$, with probability at least $2/3$, where $\eps$ is implicitly defined by the relation $B=\Theta(\frac{\rho}{\Delta^2\eps^2})$ based on the unknown ground truth $\rho$.
If the algorithm is repeated $\Theta(\log\frac{1}{\delta})$ times, and the median estimate is returned, then the probability of failure is at most $\delta$.
}

\begin{proof}

Given the fixed total sample complexity budget of $B$ coin flips, and fixing the unknown ground truth $\rho$, the target additive error parameter $\eps$ is defined by the sample complexity equation $B = \Theta(\frac{\rho}{\eps^2\Delta^2})$.
There are two cases, either $\log\frac{1}{\eps} \le c/\rho$ for some sufficiently small universal constant $c$ (in which case we show that, with high probability, Step 1 will output a correct answer and then halt), or the inequality is in the opposite direction (in which case, with high probability, either Step 1 still produces a correct answer and halts, or Steps 2 and 3 will output a correct answer).

In the case where $\log\frac{1}{\eps} \le c/\rho$, we use Corollary~\ref{Cor:OneSizeFinal} with parameters $\eps_1=\eps$ and $\eps_2=\Theta(\frac{1}{\Delta^2 B})=\Theta(\frac{1}{t})$: Algorithm~\ref{Alg:Main} will have error $\pm\eps$, except with failure probability $0.1+O(t\cdot\poly(\eps_2))=0.1+O(t\cdot\poly(\frac{1}{t}))$, where, as noted in Corollary~\ref{Cor:OneSizeFinal}, we may make the polynomial superlinear to make this failure probability $0.1+o(1)$. Further, the expected sample complexity is $O(\frac{\rho}{\eps_1^2\Delta^2})=O(B)$ in the case where $\log\frac{1}{\eps} \le c/\rho$, so by Markov's inequality, for appropriate constants we can ensure that Algorithm~\ref{Alg:Main} uses $\leq B/4$ samples with high constant probability. Thus in this case, the algorithm will correctly terminate in Step 1(b) with high probability.


Next, we analyze the case where $\log\frac{1}{\eps} \ge c/\rho$.
By Corollary~\ref{Cor:OneSizeFinal}, as above, if Step 1(b) is reached then its answer will be  $\eps$-accurate except with some small constant probability.
Otherwise, since Steps 2 and 3 are statistically independent of Step 1, we can just analyze these steps for the case $\log\frac{1}{\eps} \ge c/\rho$, ignoring what happened in Step 1.

We first claim that Step 2 will return a 2-approximation $\hat{\rho}$ of $\rho$ with high constant probability.
As before, we use Corollary~\ref{Cor:OneSizeFinal} with $\eps_2=\eps$; since (from the algorithm and the parameters of the theorem) this step uses $t=\Theta(\sqrt{\Delta^2 B})=\Theta(\frac{\sqrt{\rho}}{\eps})$ coins, solving the equation $t=\Theta(\frac{\rho}{\eps_1^2})$ of the Corollary yields $\eps_1=\Theta(\sqrt{\eps\sqrt{\rho}})=O(\sqrt{\eps})$.
Since we are in the regime where $\log\frac{1}{\eps} \ge c/\rho$, we have that $\eps_1=O(\sqrt{\eps}) \le O(e^{-c/\rho}) \ll \rho/2$ for sufficiently small $\rho$, meaning that we will approximate $\rho$ to within $\pm \rho/2$, giving us a 2-approximation. The failure probability is $0.1+o(1)$ as above. From Corollary~\ref{Cor:OneSizeFinal}, the expected sample complexity, in our case $\log\frac{1}{\eps} \ge c/\rho$ will be $O(\frac{\rho}{\eps_1^2\Delta^2}\rho\log\frac{1}{\eps_2})$; substituting in the definitions of $\eps_1,\eps_2$ yields $O(\frac{\rho^{3/2}}{\eps\Delta^2}\log\frac{1}{\eps})$. Since $\rho=O(1)$ and $\log\frac{1}{\eps}=o(\frac{1}{\eps})$ this expected sample complexity is thus $O(\frac{\rho}{\eps^2\Delta^2})=O(B)$ and Markov's inequality implies the algorithm exceeds its sample bound in Step 2 with an arbitrarily small constant probability.

We conclude by invoking Proposition~\ref{prop:opt-rho-hat} to show that the estimate returned in Step 3 by Algorithm~\ref{Alg:Optimal-rho-hat} is accurate to within additive error $\eps$ except with small constant probability.
\end{proof}

\section{Characterizing Single-Coin Algorithms}
\label{sect:Framework}

As a crucial first step towards the lower bounds of Section~\ref{sect:UnknownLower} that analyze how information from many different coins may interact, in this section we describe a unified framework for characterizing (adaptive) algorithms that flip only a single coin. Section~\ref{sect:Reduction} will then show a general structural result describing how any adaptive multi-coin algorithm may be broken into single-coin subroutines that may then be analyzed in light of the characterization of this section.

The most general form of an adaptive single-coin algorithm is a decision tree, where each node is a coin flip, and has two outgoing edges denoting the outcome of the coin flip, heads or tails; the current node captures the outcome of the entire sequence of coin flips so far, and thus for each node, a generic algorithm specifies a probability of halting, versus continuing from that node. 

Via a (standard) symmetrization argument, instead of considering the state of the algorithm to be an arbitrary sequence of coin flips, we instead aggregate this information into a pair $(n,k)$ representing the number of coin flips, and the number of heads observed so far. In outline, one may prove by induction on the number of coin flips $n$ that any such decision tree may be ``symmetrized" so that its stopping probability at each node $(n'\leq n,k)$ depends only on $n'$ and $k$, while preserving, for any underlying coin bias $p$, the total probability of hitting the set of decision tree nodes that represent observing $k$ total heads out of $n'$ flips. The inductive step relies on the fundamental property that, conditioned on observing exactly $k$ heads out of $n'$ coin flips, the distribution over all such sequences of coin flips is independent of the coin bias $p$, and depends only on the stopping probabilities along each of the ${n \choose k}$ paths in the decision tree.
This is a direct generalization of the analogous observation in the triangular walk algorithm section (Section~\ref{sect:TWalk}), and is analyzed in slightly different form in Equation~\ref{Eq:Recurrence} below.

We thus consider single-coin algorithms as random walks (Algorithm~\ref{Alg:TriangularWalk}) on the structure of the Pascal Triangle, in which the states are represented by pairs $(n,k)$, where $n$ is the total number of flips of the coin so far, and $k \le n$ is the number of ``heads" responses.
At each state $(n,k)$, the algorithm terminates with some probability $\gamma_{n,k}$, else the algorithm may request a further coin flip and continue the walk. The collection of parameters $\gamma_{n,k}$ we call a \emph{stopping rule}, and specifies that algorithm's behavior.

\begin{algorithm}
\caption{Triangular Walk}
\label{Alg:TriangularWalk}
\vspace*{0.3em}
{\bf Input:} a coin of bias $p$
\vspace*{-0.5em}
\begin{enumerate}
\setlength{\itemsep}{0em}
	\item Initialize state $(n,k)$ to $(0,0)$.
	\item Repeat until termination:
		\begin{enumerate}
			\item With probability $\gamma_{n,k}$, terminate and output $(n,k)$.
			\item Otherwise, sample one more coin flip.
			Increment $n$, and increment $k$ by the result of the flip (0 or 1).
		\end{enumerate}
\end{enumerate}
\vspace*{-0.5em}
\end{algorithm}

This formulation of single-coin algorithms, which we call a \emph{triangular walk}, reveals structure that will be useful to the rest of the analysis of this paper. In particular, since the overall objective of running an adaptive coin-flipping algorithm is to recover information about the bias $p$ of the coin (while minimizing expected sample complexity), it is fortuitous (as we will see) that the outcome of such an algorithm depends on $p$ in an unexpectedly transparent way.
This is given in Definition~\ref{Def:Alpha}.

\begin{definition}
\label{Def:Alpha}
Given a stopping rule $\{\gamma_{n,k}\}$, we define coefficients $\{\alpha_{n,k}\}$, $\{\beta_{n,k}\}$, and $\{\eta_{n,k}\}$, so that, for any $p\in[0,1]$, the triangular walk with stopping rule $\{\gamma_{n,k}\}$ on a coin of bias $p$, the coefficients have the semantics: $\alpha_{n,k} p^k (1-p)^{n-k}$ represents the probability that the walk terminates at $(n,k)$, with all such probabilities summing to 1; $\beta_{n,k} p^k (1-p)^{n-k}$ represents the probability that the triangular walk encounters $(n,k)$, whether or not it terminates there, and $\eta_{n,k} p^k (1-p)^{n-k}$ is the probability that the triangular walk encounters $(n,k)$ but does \emph{not} terminate there. Each of these reparameterizations of the stopping rule may be derived from $\{\gamma_{n,k}\}$ using the following relations.
\begin{align}
\label{Eq:Recurrence}
\beta_{0,0} &= 1\\
\beta_{n+1,k+1} &= \beta_{n,k+1}\cdot(1-\gamma_{n,k+1}) + \beta_{n,k}\cdot(1-\gamma_{n,k})\nonumber\\
\alpha_{n,k} &= \beta_{n,k}\cdot\gamma_{n,k}\nonumber\\
\eta_{n,k} &= \beta_{n,k} - \alpha_{n,k} \;(\,= \beta_{n,k}\cdot(1-\gamma_{n,k})\,).\nonumber
\end{align}

\end{definition}

Consider the original setting, where one has a universe of (different) coins; one might repeatedly run a single-coin algorithm on coins drawn from the universe, and somehow combine their outputs into a final answer. There are many conceivable ways of aggregating the outputs of single-coin algorithms into an estimate,
and the lower bounds of Section~\ref{sect:UnknownLower} consider them all. However, a particularly natural and powerful approach is to construct a linear estimator, namely to have the single-coin algorithm output a real number coefficient $v_{n,k}$ at each termination node, with the overall algorithm estimating the expected output of the single-coin algorithm, across the coins in the universe.
Algorithm~\ref{Alg:Main} works this way, using the median-of-means method (instead of taking the sample mean) to estimate the expected output of Algorithm~\ref{Alg:Single}.
Such linear estimators are surprisingly flexible, and are known to be optimal in certain classes of estimation tasks~\cite{VV:2011:power}.

\section{Fully-Adaptive Lower Bounds}
\label{sect:UnknownLower}


We show in this section that Algorithm~\ref{Alg:Optimal} is optimal in all four problem parameters $\rho$, $\eps$, $\Delta$ and $\delta$, even when compared to all fully-adaptive algorithms that are adaptive across different coins.
In particular, we show the following indistinguishability result (Theorem~\ref{thm:UnknownLower}).

\medskip\noindent{\bf Theorem~\ref{thm:UnknownLower}.}
\emph{For $\rho\in [0,\frac{1}{2})$ and $\epsilon\in (0, 1-2\rho]$, the following two situations are impossible to distinguish with at least $1-\delta$ probability using an expected $o(\frac{\rho}{\eps^2\Delta^2}\log\frac{1}{\delta})$ samples: A) $\rho$ fraction of the coins have probability $\frac{1}{2}+\Delta$ of landing heads and $1-\rho$ fraction of the coins have probability $\frac{1}{2}-\Delta$ of landing heads, versus B) $\rho+\epsilon$ fraction of the coins have probability $\frac{1}{2}+\Delta$ of landing heads and $1-(\rho+\epsilon)$ fraction of the coins have probability $\frac{1}{2}-\Delta$ of landing heads.
	This impossibility crucially includes fully-adaptive algorithms.
}\\

With the algorithmic result of Theorem~\ref{thm:opt}, this lower bound is therefore tight to within a constant factor. We note that the restrictions $\rho<\frac{1}{2}$ and $\epsilon\le 1-2\rho$ reflect the symmetry of the problem, where the pair $\rho,\rho+\epsilon$ is exactly as hard to distinguish as the pair $1-\rho-\epsilon, 1-\rho$, yielding analogous results for the symmetric parameter regime.

\begin{example}\label{ex:simple-lower-bounds}
Even in the constant failure probability regime, the $\Omega(\frac{\rho}{\eps^2\Delta^2})$ lower bound requires significant analysis, forming the bulk of the remainder of this paper, but two special cases have direct proofs. When $\Delta=\Theta(1)$ we can prove a $\Omega(\frac{\rho}{\eps^2})$ lower bound without the $\Delta$ dependence:
consider the case where all coins are unbiased and perfect, meaning that the only source of randomness is from the mixture of coins, which is itself a Bernoulli distribution of bias either $\rho$ or $\rho+\eps$.
We quote the standard fact that, in order to estimate a Bernoulli coin flip of bias $\rho$ to up to additive $\eps$, we need $\Omega(\frac{\rho}{\eps^2})$ samples to succeed with constant probability; this can be proven by a standard (squared) Hellinger distance argument.
On the other hand, it is also straightforward to prove a $\frac{1}{\Delta^2}$ lower bound (covering the regime where $\rho$ and $\epsilon$ are constant): consider the easiest regime for $\rho$ and $\epsilon$, where $\rho=0$ and $\eps = 1$; thus coins either all have $\frac{1}{2}+\Delta$ bias or all have $\frac{1}{2}-\Delta$ bias.
To distinguish whether we have access to positive coins or negative coins requires $\Omega(\frac{1}{\Delta^2})$ samples.
\end{example}

In order to show Theorem~\ref{thm:UnknownLower}, we use the Hellinger distance and KL-divergence between probability distributions as proxies for bounding the total variation distance.
\begin{definition}[Hellinger Distance]
	Given two discrete distributions $P$ and $Q$, the \emph{Hellinger distance} $H(P,Q)$ between them is $$ \frac{1}{\sqrt{2}} \sqrt{\sum_i (\sqrt{p_i}-\sqrt{q_i})^2} = \sqrt{1-\sum_i \sqrt{p_iq_i}}$$
\end{definition}
\begin{definition}[KL-divergence]
	Given two discrete distributions $P$ and $Q$, the \emph{KL-divergence} $\DKL(P||Q)$ between them is $$ \sum_i p_i \log \frac{p_i}{q_i} $$
\end{definition}

The following facts capture how the Hellinger distance and KL-divergence can be used to show sample complexity lower bounds.
\begin{fact}[Chapter 2.4, \cite{Tsybakov:2008}]
	\label{fact:statdist}
	For any two distributions $P$ and $Q$ over the same domain, we have
	$$ \ell_1(P,Q) \le \sqrt{2} H(P,Q) $$
	and furthermore, for any event $E$,
	$$ P(E) + Q(\bar{E}) \ge \frac{1}{2}e^{-\DKL(P||Q)} $$
	The second inequality is also known as the high-probability Pinsker inequality.
\end{fact}

Recall from the introduction that, the main challenge in proving a general lower bound for our problem lies in analyzing the two kinds of adaptivity that algorithms may employ that were both absent in the special cases of Example~\ref{ex:simple-lower-bounds}. Explicitly, when taking samples from a given coin, we can choose whether to ask for another sample based on A) previous results of this coin, and also B) previous results of all the other coins. This first kind of adaptivity, ``single-coin adaptivity", is crucially used in the algorithms presented in the rest of the paper (e.g.~the ``shape" of the stopping rule for our triangular-walk algorithms); in Proposition~\ref{prop:BoundedLower} we analyze the best possible performance of such triangular stopping rules. The most interesting part of the proof of Theorem~\ref{thm:UnknownLower} consists of showing that the second kind of adaptivity (cross-coin adaptivity) cannot help in the lower bound setting, which we analyze via general Hellinger distance/KL-divergence inequalities (Lemmas~\ref{Lem:Reduction} and~\ref{Lem:ReductionKL}) in Section~\ref{sect:Reduction}.

\subsection{Reduction to Single-Coin Adaptive Algorithms}
\label{sect:Reduction}

In this section, we give two related but distinct reductions to single-coin adaptive algorithms.
The first is a general decomposition (``direct sum") inequality that decomposes the squared Hellinger distance of running a fully-adaptive algorithm on two different coin populations into the sum of squared Hellinger distances of running single-coin adaptive algorithms on the two coin populations.
This inequality will lead to a constant probability sample complexity lower bound.
The second inequality instead decomposes the KL divergence into (a constant times) a sum of squared Hellinger distances, however with an additional slight restriction that the two coin populations being considered must be very close to each other.
The upside to using this second inequality is that, an upper bound on the KL divergence combined with the high probability Pinsker inequality allows us to obtain a \emph{high probability} sample complexity lower bound, which in particular is tight in \emph{all} parameters of the problem, up to a multiplicative constant.

Both of the following inequalities are applicable to populations of variables \emph{beyond} Bernoulli coins.
We believe that the general inequalities are of independent interest to the community, since they would be applicable and useful for proving lower bounds on a variety of scenarios involving, for example, a Gaussian variant of the current problem, where instead of getting yes/no answers on the positivity of an item, one gets a real-valued score which correlates with the positivity of the item.

We phrase both lemmas as upper bounds on distances between distributions of the \emph{transcript} of an algorithm, which when combined with the data processing inequality immediately yields upper bounds on distances between distributions of the algorithm's \emph{output}. See, for example, the very end of the proof of Theorem~\ref{thm:UnknownLower}.

\begin{lemma}
\label{Lem:Reduction}
Consider a problem setting where there is a collection of random variables, and an adaptive algorithm can draw variables from the collection and draw independent samples from the drawn variables.
Now consider an arbitrary algorithm that iteratively samples from random variables drawn from the collection, choosing each subsequent variable to sample in an arbitrary adaptive manner based on the results of previous sample outcomes.
Suppose the algorithm terminates almost surely.
Consider two arbitrary collections of random variables, denoted by distributions $\mathcal{A}$ and $\mathcal{B}$ over the set of possible random variables.
Let $H_\mathrm{full}^2$ be the squared Hellinger distance between the transcript of a single run of the algorithm where 1) the random variables are drawn from $\mathcal{A}$ versus where 2) the random variables are drawn from $\mathcal{B}$.
Furthermore, let $H_i^2$ be the squared Hellinger distance between the two scenarios, but instead of running the algorithm as is, we only use random variable $i$ (as drawn either from $\mathcal{A}$ or $\mathcal{B}$ depending on the scenario) and simulate all other random variables as independent random variables that are themselves drawn from the mixture distribution $\frac{\mathcal{A}}{2}+\frac{\mathcal{B}}{2}$.
Then
$$ H_\mathrm{full}^2 \le \sum_\textnormal{variable $i$} H_i^2 $$	
\end{lemma}

\begin{proof}
    It suffices to prove the result for deterministic algorithms, since squared Hellinger distance is linear with respect to mixtures of distributions \emph{with distinct outcomes}, and a randomized algorithm is simply a mixture of deterministic algorithms which also records which of the algorithms the random coins picked.
    Furthermore, the following proof is phrased in terms of the special case where the collection of random variables are Bernoulli coins (which is the setting considered in this paper).
    Barring measure-theoretic formalization issues that we do not discuss, the proof generalizes directly to populations of arbitrary random variables.
    
    A deterministic fully-adaptive algorithm is a decision tree, where each node is labeled by the identity of the coin the algorithm chooses to flip next conditioned on reaching this node, and each edge out of a node is labeled by a heads or tails result for this coin.
	We can view a run of the algorithm as follows: 1) first draw all the random coins from either $\mathcal{A}$ or $\mathcal{B}$ depending on the scenario, and then 2) flip these coins according to this fully-adaptive algorithm---we view choosing the coins from $\mathcal{A}$ or $\mathcal{B}$ as happening at the beginning since all these samples are free and only the coin flips themselves are counted.
	After step 1, fixing the bias of each coin, the probability of ending up at the $i^\text{th}$ leaf of the decision tree is simply the probability (over coin flips) that every edge along the path from the root to that leaf is followed.
	Note that each edge is a probabilistic event depending on only one coin.
	Therefore, this probability can be factored into a product of probabilities, one term for each of the coins.
	For example, suppose the path to leaf $i$ involves coin $j$ returning 5 heads in a row, then getting some particular sequence from flipping some \emph{other} coins, then coin $j$ returning another 2 heads followed by 3 tails.
	Then, if coin $j$ has bias $p_j$, it contributes $p_j^{5+2}(1-p_j)^3$ to the probability product.
	
	We denote by $q^\mathcal{A}_{j,i}$ the \emph{expected} contribution of coin $j$ to the probability product for leaf $i$, over the randomness of $\mathcal{A}$ on the bias of coin $j$.
	In the previous example, $q^\mathcal{A}_{j,i}$ would be equal to $\Exp_{p \from \mathcal{A}}[p^7(1-p)^3]$.
	We similarly define $q^\mathcal{B}_{j,i}$. Explicitly, for leaf $i$ and coin $j$, $q^\mathcal{A}_{j,i}$ is the expectation (over $p$ drawn from $\mathcal{A}$) of $p$ to the exponent of the number of ``heads" edges on the path from the root to node $i$ in the decision tree, times $(1-p)$ to the exponent of the number of ``tails" edges on this path.
	
	Using this notation, the probability of the algorithm reaching leaf $i$, when the coins are sampled from distribution $\mathcal{A}$, would be $\prod_{\text{coin $j$}} q^\mathcal{A}_{j,i}$, since each coin is sampled from $\mathcal{A}$ independently; let $\prod_{\text{coin $j$}} q^\mathcal{B}_{j,i}$ be the respective probability for sampling from $\mathcal{B}$.
	
	Since the total probability of reaching all leaves $i$ must equal 1, this expression yields the immediate corollary, that for any distribution $\mathcal{A}$ over $[0,1]$,
	\begin{equation}
	\label{eq:TotalProb}
		\sum_{\text{leaf $i$}} \prod_\text{coin $j$} q^\mathcal{A}_{j,i} = 1
	\end{equation}
	
	We can now express the squared Hellinger distance with this notation.
	For any two distributions $\vec{a}$ and $\vec{b}$, 1 minus their squared Hellinger distance can be rewritten as $\sum_i \sqrt{a_ib_i}$.
	In our context, the summation is over leaves $i$, and thus the squared Hellinger distance between the two scenarios in question is
	
	\begin{equation}\label{eq:Hellinger-def-sqrt} H_\mathrm{full}^2=1-\sum_{\text{leaf $i$}}\sqrt{\prod_\text{coin $j$} q^\mathcal{A}_{j,i} \prod_\text{coin $j$} q^\mathcal{B}_{j,i} } \end{equation}
	
Since $q^\mathcal{A}_{j,i}$ and $q^\mathcal{B}_{j,i}$ are both non-negative, we simplify the summand as
	
	\begin{equation}
	\hspace*{0.8cm}
	\begin{aligned}
	\label{Eq:AMGM}
	    &\sqrt{\prod_\text{coin $j$} q^\mathcal{A}_{j,i} \prod_\text{coin $j$} q^\mathcal{B}_{j,i} }\\
	    = &\left(\prod_\text{coin $j$} \frac{q^\mathcal{A}_{j,i}+q^\mathcal{B}_{j,i}}{2} \right)\left(\prod_\text{coin $j$} \frac{2\sqrt{q^\mathcal{A}_{j,i}q^\mathcal{B}_{j,i}}}{q^\mathcal{A}_{j,i}+q^\mathcal{B}_{j,i}}\right)\\
	    \ge &\left(\prod_\text{coin $j$} \frac{q^\mathcal{A}_{j,i}+q^\mathcal{B}_{j,i}}{2} \right) \left[1- \sum_\text{coin $j$}\left(1 - \frac{2\sqrt{q^\mathcal{A}_{j,i}q^\mathcal{B}_{j,i}}}{q^\mathcal{A}_{j,i}+q^\mathcal{B}_{j,i}} \right)\right]
	\end{aligned}
	\end{equation}
	where the inequality holds because each $\frac{2\sqrt{q^\mathcal{A}_{j,i}q^\mathcal{B}_{j,i}}}{q^\mathcal{A}_{j,i}+q^\mathcal{B}_{j,i}}$ is less than or equal to 1 by the AM-GM inequality (and at least 0), and therefore we can apply the union bound by treating each term as a probability---namely, for any $x_j\in[0,1]$ we have $\prod_j x_j\geq 1-\sum_j(1-x_j)$.
	
	Observe that our definition of $q$, being an expectation, is thus linear in the distribution in its superscript, and thus $\frac{1}{2}(q^\mathcal{A}_{j,i}+q^\mathcal{B}_{j,i})=q^{\frac{\mathcal{A}}{2}+\frac{\mathcal{B}}{2}}_{j,i}$, and therefore the right hand side of the inequality can be rewritten as
	
	\begin{equation}\label{eq:RHS-simplified} \left(\prod_\text{coin $j$} q^{\frac{\mathcal{A}}{2}+\frac{\mathcal{B}}{2}}_{j,i}\right) \left[1- \sum_\text{coin $j$} \left(1 - \frac{\sqrt{q^\mathcal{A}_{j,i}q^\mathcal{B}_{j,i}}}{q^{\frac{\mathcal{A}}{2}+\frac{\mathcal{B}}{2}}_{j,i}}\right)\right] \end{equation}
	
	Thus the sum of Equation~\ref{eq:RHS-simplified} over all leaves is at most $1-H_\text{full}^2$.  We simplify the summation by changing the summation variable in Equation~\ref{eq:RHS-simplified} from $j$ to $k$, and distributing the initial product so as to form three additive terms (the ``$j\neq k$" in the bounds of the last product below is because the $j=k$ term gets canceled by the denominator from the last term in Equation~\ref{eq:RHS-simplified}):
	\begin{align*}
	&\left(\sum_\text{leaf $i$} \prod_\text{coin $j$} q^{\frac{\mathcal{A}}{2}+\frac{\mathcal{B}}{2}}_{j,i}\right) - \sum_\text{coin $k$} \left(\sum_\text{leaf $i$} \prod_\text{coin $j$} q^{\frac{\mathcal{A}}{2}+\frac{\mathcal{B}}{2}}_{j,i}\right)\\
	& -\sum_\text{coin $k$} \left(\sum_\text{leaf $i$} \left(\prod_\text{coin $j \neq k$} q^{\frac{\mathcal{A}}{2}+\frac{\mathcal{B}}{2}}_{j,i}\right)\sqrt{q^\mathcal{A}_{k,i}q^\mathcal{B}_{k,i}}\right)
	\end{align*}
	
	We know by Equation~\ref{eq:TotalProb} that $\left(\sum_\text{leaf $i$} \prod_\text{coin $j$} q^{\frac{\mathcal{A}}{2}+\frac{\mathcal{B}}{2}}_{j,i}\right) = 1$, and so the sum can be written as
	$$ 1 - \sum_\text{coin $k$} \left(1 - \sum_\text{leaf $i$} \left(\prod_\text{coin $j \neq k$} q^{\frac{\mathcal{A}}{2}+\frac{\mathcal{B}}{2}}_{j,i}\right)\sqrt{q^\mathcal{A}_{k,i}q^\mathcal{B}_{k,i}}\right) $$
	which by definition of $H_k$ is equal to $1 - \sum_\text{coin $k$} H_k^2$: by Equation~\ref{eq:Hellinger-def-sqrt}, 1 minus the squared Hellinger distance between the view of the algorithm when the $k^\textrm{th}$ coin is from $\mathcal{A}$ versus from $\mathcal{B}$, where all remaining coins are drawn from the mixture $\frac{\mathcal{A}}{2}+\frac{\mathcal{B}}{2}$ equals $\sum_\text{leaf $i$} \left(\prod_\text{coin $j \neq k$} q^{\frac{\mathcal{A}}{2}+\frac{\mathcal{B}}{2}}_{j,i}\right)\sqrt{q^\mathcal{A}_{k,i}q^\mathcal{B}_{k,i}}$.
	
	Summarizing, we have shown that $1 - H_\text{full}^2 \ge 1 - \sum_\text{coin $k$} H_k^2$, from which the lemma statement follows.
\end{proof}

We now give the KL-divergence decomposition lemma (Lemma~\ref{Lem:ReductionKL}) which will yield a tight \emph{high probability} sample complexity lower bound, but makes a further assumption than Lemma~\ref{Lem:Reduction} that the two coin populations are close to each other.
As a note, the definition of $H_i^2$ is slightly different in this lemma from the definition in Lemma~\ref{Lem:Reduction}, and is not a typographical mistake.

\begin{lemma}
\label{Lem:ReductionKL}
Consider an arbitrary algorithm that iteratively flips coins from a collection of coins, choosing each subsequent coin to flip in an arbitrary adaptive manner based on the results of previous flips.
Suppose the algorithm terminates almost surely.
Consider two arbitrary mixtures of coins, denoted by distributions $\mathcal{A}$ and $\mathcal{B}$ over the coin bias $[0,1]$.
Let $D_\mathrm{full}$ be the KL-divergence between the transcript of a single run of the algorithm where 1) the coins are drawn from the mixture $\rho \mathcal{A}+ (1-\rho) \mathcal{B}$ versus where 2) the coins are drawn from $(\rho+\eps) \mathcal{A} + (1-\rho-\eps) \mathcal{B}$, where $\rho\in [0,\frac{1}{2})$, $\eps\in (0, 1-2\rho]$ and $\eps < \rho$.
Furthermore, let $H_i^2$ be the squared Hellinger distance between the two scenarios, but instead of running the algorithm as is, we only use coin $i$ (as drawn either from the $\rho$-fraction mixture or the $(\rho+\eps)$-fraction mixture depending on the scenario) and simulate all other coins as independent coins drawn from the $\rho$-fraction mixture.
Then
$$ D_\mathrm{full} = O\left(\sum_\textnormal{coin $i$} H_i^2\right) $$
\end{lemma}

The proof of Lemma~\ref{Lem:ReductionKL} is similar to that of Lemma~\ref{Lem:Reduction} by viewing algorithms as decision trees, with the crucial difference that, rather than using the AM-GM inequality, Lemma~\ref{Lem:ReductionKL} instead bounds the KL-divergence via a quadratic bound $\log(1+x)\geq x-x^2$, valid for $x\in[-\frac{1}{2},1]$.

\begin{proof}
(For the following proof, the set-up up to and including Equation~\ref{eq:TotalProbKL} is essentially the same as that in the proof of Lemma~\ref{Lem:Reduction}, analogously, up to and including Equation~\ref{eq:TotalProb}.
For completeness, we include the context for the specific notation we use in this proof.)

   It suffices to prove the result for deterministic algorithms, since both squared Hellinger distance and KL-divergence are linear with respect to mixtures of distributions \emph{with distinct outcomes}, and a randomized algorithm is simply a mixture of deterministic algorithms which also records which of the algorithms the random coins picked.
    
    A deterministic fully-adaptive algorithm is a decision tree, where each node is labeled by the identity of the coin the algorithm chooses to flip next conditioned on reaching this node, and each edge out of a node is labeled by a heads or tails result for this coin.
	We can view a run of the algorithm as follows: 1) first draw all the random coins from either $\rho \mathcal{A}+ (1-\rho) \mathcal{B}$ or $(\rho+\eps) \mathcal{A}+ (1-\rho-\eps) \mathcal{B}$ depending on the scenario, and then 2) flip these coins according to this fully-adaptive algorithm.
	After step 1, fixing the bias of each coin, the probability of ending up at the $i^\text{th}$ leaf of the decision tree is simply the probability (over coin flips) that every edge along the path from the root to that leaf is followed.
	Note that each edge is a probabilistic event depending on only one coin.
	Therefore, this probability can be factored into a product of probabilities, one term for each of the coins.
	For example, suppose the path to leaf $i$ involves coin $j$ returning 5 heads in a row, then getting some particular sequence from flipping some \emph{other} coins, then coin $j$ returning another 2 heads followed by 3 tails.
	Then, if coin $j$ has bias $p_j$, it contributes $p_j^{5+2}(1-p_j)^3$ to the probability product.
	
	We denote by $q^\mathcal{A}_{j,i}$ the \emph{expected} contribution of coin $j$ to the probability product for leaf $i$, over the randomness of $\mathcal{A}$ on the bias of coin $j$.
	In the previous example, $q^\mathcal{A}_{j,i}$ would be equal to $\Exp_{p \from \mathcal{A}}[p^7(1-p)^3]$.
	We similarly define $q^\mathcal{B}_{j,i}$. Explicitly, for leaf $i$ and coin $j$, $q^\mathcal{A}_{j,i}$ is the expectation (over $p$ drawn from $\mathcal{A}$) of $p$ to the exponent of the number of ``heads" edges on the path from the root to node $i$ in the decision tree, times $(1-p)$ to the exponent of the number of ``tails" edges on this path.
	
	To simplify notation, we also denote by $q^\mathcal{\rho}_{j,i}$ as the above probability product for coin $j$ and leaf $i$ when coin $j$ is drawn from the $\rho$-mixture, namely $\rho \mathcal{A}+ (1-\rho) \mathcal{B}$.
	Note that, by definition, $q^\mathcal{\rho}_{j,i} = \rho q^\mathcal{A}_{j,i} + (1-\rho)q^\mathcal{B}_{j,i}$.
	We use analogous notation for the $(\rho+\eps)$-mixture.
	
	Using this notation, the probability of the algorithm reaching leaf $i$, when the coins are sampled from distribution $\mathcal{A}$, would be $\prod_{\text{coin $j$}} q^\mathcal{A}_{j,i}$, since each coin is sampled from $\mathcal{A}$ independently, with $\prod_{\text{coin $j$}} q^\mathcal{B}_{j,i}$ being the respective probability for sampling from $\mathcal{B}$.
	The observation holds similarly for coin distribution that are the $\rho$-mixture or $(\rho+\eps)$-mixture of $\mathcal{A}$ and $\mathcal{B}$.
	
	Since the total probability of reaching all leaves $i$ must equal 1, this expression yields the immediate corollary, that for any collection of distributions $\mathcal{C}_j$ over $[0,1]$ indexed by coin $j$ (imagine $\mathcal{C}_j$ each being one of $\mathcal{A}$, $\mathcal{B}$, the $\rho$-mixture of the two or the $(\rho+\eps)$-mixture of the two)
	\begin{equation}
	\label{eq:TotalProbKL}
		\sum_{\text{leaf $i$}} \prod_\text{coin $j$} q^{\mathcal{C}_j}_{j,i} = 1
	\end{equation}
	
	We can now express the KL-divergence with the above notation.
	\begin{align*}
	    D_\mathrm{full} &= - \sum_{\text{leaf $i$}} \left(\prod_{\text{coin $j$}} q^{\rho}_{j,i} \right) \log \left(\frac{\prod_{\text{coin $k$}} q^{\rho+\eps}_{k,i}}{\prod_{\text{coin $k$}} q^{\rho}_{k,i}}\right)\\
	    &= - \sum_{\text{leaf $i$}} \left(\prod_{\text{coin $j$}} q^{\rho}_{j,i} \right) \sum_{\text{coin $k$}} \log \left(1+ \eps \frac{ q^\mathcal{A}_{k,i} - q^\mathcal{B}_{k,i}}{\rho q^\mathcal{A}_{k,i} + (1-\rho)q^\mathcal{B}_{k,i}}\right)
	\end{align*}
	where the second line follows from the definition of $q^\rho_{k,i} = \rho q^\mathcal{A}_{k,i} + (1-\rho)q^\mathcal{B}_{k,i}$.
	Further observe that the multiplier to $\eps$ in the second line is upper bounded by $1/\rho$ in magnitude.
	Since Taylor's theorem gives that $\log(1+x)=x-\Theta(x^2)$ for $x\leq 1$, we have when $\eps/\rho \leq 1$, that
	\begin{align*}
	    D_\mathrm{full} &= - \sum_{\text{leaf $i$}} \left(\prod_{\text{coin $j$}} q^{\rho}_{j,i} \right) \sum_{\text{coin $k$}} \left(\eps \frac{ q^\mathcal{A}_{k,i} - q^\mathcal{B}_{k,i}}{q^\mathcal{\rho}_{k,i}} - \Theta(\eps^2) \frac{(q^\mathcal{A}_{k,i} - q^\mathcal{B}_{k,i})^2}{(q^\mathcal{\rho}_{k,i})^2} \right)
	\end{align*}
    We can further simplify the expression by observing that for any fixed coin $k$,
    \begin{align*}
        \sum_{\text{leaf $i$}} \left(\prod_{\text{coin $j$}} q^{\rho}_{j,i} \right) \frac{q^\mathcal{A}_{k,i}}{q^\rho_{k,i}} &= \sum_{\text{leaf $i$}} \left(\prod_{\text{coin $j \neq k$}} q^{\rho}_{j,i} \right) q^\mathcal{A}_{k,i} = 1
    \end{align*}
    where the second equality is by Equation~\ref{eq:TotalProbKL}.
    This observation holds also when we replace the mixture $\mathcal{A}$ with the mixture $\mathcal{B}$.
    Therefore, we have $$ \sum_{\text{leaf $i$}} \left(\prod_{\text{coin $j$}} q^{\rho}_{j,i} \right) \sum_{\text{coin $k$}} \eps \frac{ q^\mathcal{A}_{k,i} - q^\mathcal{B}_{k,i}}{q^\mathcal{\rho}_{k,i}} = 0 $$
    meaning that
    $$ D_\mathrm{full} = \Theta(\eps^2) \sum_{\text{leaf $i$}} \left(\prod_{\text{coin $j$}} q^{\rho}_{j,i} \right) \sum_{\text{coin $k$}}  \frac{(q^\mathcal{A}_{k,i} - q^\mathcal{B}_{k,i})^2}{(q^\mathcal{\rho}_{k,i})^2} = 
    \sum_{\text{coin $k$}} \Theta(\eps^2) \sum_{\text{leaf $i$}} \left(\prod_{\text{coin $j$}} q^{\rho}_{j,i} \right)  \frac{(q^\mathcal{A}_{k,i} - q^\mathcal{B}_{k,i})^2}{(q^\mathcal{\rho}_{k,i})^2}$$
    
    It remains to show that the right hand side is bounded by $\sum_\text{coin $k$} H_k^2$, where, as in the lemma statement, $H_k^2$ is the squared Hellinger distance between a single run of the algorithm when coin $k$ is drawn either from the $\rho$-mixture of $\mathcal{A}$ and $\mathcal{B}$ or the $(\rho+\eps)$-mixture, and all other coins are simulated and simply drawn from the $\rho$-mixture.
    To see this, we write out what $H_k^2$ is, using the definition of squared Hellinger distance:
    \begin{align*}
        H_k^2 &= \sum_{\text{leaf $i$}} \left(\sqrt{q^\rho_{k,i}\prod_{\text{coin $j \ne k$}} q^\rho_{j,i}} - \sqrt{q^{\rho+\eps}_{k,i}\prod_{\text{coin $j \ne k$}} q^\rho_{j,i}}\right)^2\\
        &= \sum_{\text{leaf $i$}} \left(\prod_{\text{coin $j \ne k$}} q^\rho_{j,i}\right) \left(\sqrt{q^\rho_{k,i}} - \sqrt{q^{\rho+\eps}_{k,i}}\right)^2\\
        &= \sum_{\text{leaf $i$}} \left(\prod_{\text{coin $j \ne k$}} q^\rho_{j,i}\right) q^\rho_{k,i} \left(1 - \sqrt{\frac{q^{\rho+\eps}_{k,i}}{q^\rho_{k,i}}}\right)^2\\
        &= \sum_{\text{leaf $i$}} \left(\prod_{\text{coin $j \ne k$}} q^\rho_{j,i}\right) q^\rho_{k,i} \left(1 - \sqrt{1 + \eps\frac{q^\mathcal{A}_{k,i}-q^\mathcal{B}_{k,i}}{\rho q^\mathcal{A}_{k,i}+(1-\rho)q^\mathcal{B}_{k,i}}}\right)^2
    \end{align*}
    where the last line is again by definition that $q^\rho_{k,i} = \rho q^\mathcal{A}_{k,i}+(1-\rho)q^\mathcal{B}_{k,i}$.
    By reasoning we used above, as long as $\eps < \rho$, we have this expression being equal to
    \begin{align*}
        H_k^2 &= \sum_{\text{leaf $i$}} \left(\prod_{\text{coin $j \ne k$}} q^\rho_{j,i}\right) q^\rho_{k,i} \left( \Theta(\eps)\frac{q^\mathcal{A}_{k,i}-q^\mathcal{B}_{k,i}}{\rho q^\mathcal{A}_{k,i}+(1-\rho)q^\mathcal{B}_{k,i}}\right)^2\\
        &= \Theta(\eps^2) \sum_{\text{leaf $i$}} \left(\prod_{\text{coin $j \ne k$}} q^\rho_{j,i}\right) q^\rho_{k,i} \frac{(q^\mathcal{A}_{k,i}-q^\mathcal{B}_{k,i})^2}{(q^\rho_{k,i})^2}\\
        &= \Theta(\eps^2) \sum_{\text{leaf $i$}} \left(\prod_{\text{coin $j$}} q^\rho_{j,i}\right) \frac{(q^\mathcal{A}_{k,i}-q^\mathcal{B}_{k,i})^2}{(q^\rho_{k,i})^2}
    \end{align*}
    which is exactly the term in the sum over coin $k$ for $D_\mathrm{full}$, showing the lemma.
\end{proof}

For the lower bound proof at hand, we show Corollary~\ref{cor:BoundedLower} in the next subsection, which upper bounds the squared Hellinger distance for single-coin adaptive algorithms by a quantity that is proportional to the expected number of samples taken by the algorithm.

\begin{corollary}
	\label{cor:BoundedLower}
    Consider an arbitrary single-coin adaptive algorithm.	
	Let $H^2$ be the squared Hellinger distance between a single run of the algorithm where 1) a coin with bias $\frac{1}{2} + \Delta$ is used with probability $\rho$ and a coin with bias $\frac{1}{2} - \Delta$ is used otherwise, versus a run of the algorithm where 2) a coin with bias $\frac{1}{2} + \Delta$ is used with probability $\rho+\eps$ and a coin with bias $\frac{1}{2} - \Delta$ is used otherwise.
	Furthermore, let $\Exp_{\rho}[n]$ and $\Exp_{\rho+\frac{\eps}{2}}[n]$ be the expected number of coin flips during a run of the algorithm where we use a $\frac{1}{2}+\Delta$ coin with probability $\rho$ and $\rho+\frac{\eps}{2}$ respectively, and a $\frac{1}{2}-\Delta$ coin otherwise.
	If all of $\rho$, $\eps$, $\Delta$ and $\eps/\rho$ are smaller than some universal absolute constant, then
	$$ \max\left[\frac{H^2}{\Exp_{\rho}[n]}, \frac{H^2}{\Exp_{\rho+\frac{\eps}{2}}[n]}\right] = O\left(\frac{\eps^2\Delta^2}{\rho}\right) $$
\end{corollary}

Using Corollary~\ref{cor:BoundedLower} and Lemma~\ref{Lem:ReductionKL}, we now complete the proof of the main high probability indistinguishability result (Theorem~\ref{thm:UnknownLower}) for fully-adaptive algorithms.
We note again that Lemma~\ref{Lem:Reduction}, which is applicable to more general coin populations with fewer restrictions than Lemma~\ref{Lem:ReductionKL}, can be used to derive a constant probability sample complexity lower bound with essentially the same proof as follows, with the exception that we would use the Hellinger distance inequality in Fact~\ref{fact:statdist} instead of the high-probability Pinsker inequality.

\begin{proof}[Proof of Theorem~\ref{thm:UnknownLower}]

Letting $\mathcal{A}$ of be a population of coins that all have $\frac{1}{2}+\Delta$ probability, with $\mathcal{B}$ a population of coins that all have $\frac{1}{2}-\Delta$ probability, our goal is to show the indistinguishability of $\rho\mathcal{A}+(1-\rho)\mathcal{B}$ from $(\rho+\eps)\mathcal{A}+(1-\rho-\eps)\mathcal{B}$. We apply Lemma~\ref{Lem:ReductionKL} and use the lemma's conclusion, that  $D_\mathrm{full} = O\left(\sum_\text{coin $i$} H_i^2\right)$.

Next, for each $i$, the quantity $H_i^2$ of Lemma~\ref{Lem:ReductionKL} describes the squared Hellinger distance between an induced \emph{single-coin} algorithm run on a single coin from scenario $A$ versus $B$ respectively (with the remaining coins being simulated, from scenario $A$ with a $\rho$-fraction mixture). We thus bound $H_i^2$ from Corollary~\ref{cor:BoundedLower}. As in the corollary, let $\Exp_{i,\rho}[n]$ denote the expected number of samples from coin $i$ when running the induced algorithm (for coin $i$) on a mixture that uses a $\frac{1}{2}+\Delta$ coin with probability $\rho$ and a $\frac{1}{2}-\Delta$ coin otherwise. Thus Corollary~\ref{cor:BoundedLower} yields that $H_i^2 = O(\frac{\eps^2\Delta^2}{\rho})\cdot\Exp_{i,\rho}[n]$. Summing, combined with the result from Lemma~\ref{Lem:ReductionKL} above, yields \[D_\mathrm{full} \le  O\left(\frac{\eps^2\Delta^2}{\rho}\right)\cdot \sum_\textnormal{coin $i$}{\Exp}_{i,\rho}[n]  \]

Crucially, the sum (over choice of coin $i$) of the expected number of flips ${\Exp}_{i,\rho}[n]$ (when running the algorithm induced for coin $i$) can be viewed in a different way: the $i^\textrm{th}$ term is exactly the expected number of times that coin $i$ is flipped when running the overall algorithm where every coin is drawn from the $\rho$-fraction mixture in scenario $A$. Namely, this sum counts the total expected number of coin flips (across all coins $i$), for the algorithm run in the setting where all coins are drawn from the $\rho$-fraction mixture.
Thus, for an algorithm that uses $o(\frac{\rho}{\eps^2\Delta^2}\log\frac{1}{\delta})$ flips in expectation, we conclude that \[D_\mathrm{full} \le  O\left(\frac{\eps^2\Delta^2}{\rho}\right)\cdot o\left(\frac{\rho}{\eps^2\Delta^2}\log\frac{1}{\delta}\right)=o\left(\log\frac{1}{\delta}\right) \]

We conclude by using the high-probability Pinsker inequality. In the notation of the inequality, given an algorithm that attempts to classify whether it is in scenario $A$ or $B$, let $P,Q$ respectively be the distributions of its output in scenarios $A,B$ respectively; let $E$ be the event that the algorithm outputs ``scenario $B$". Then the probability that the algorithm is wrong is $P(E)+Q(\bar{E})$. By the high-probability Pinsker inequality this failure probability is at least $\frac{1}{2}e^{-\DKL(P||Q)}\ge \frac{1}{2}e^{-D_\mathrm{full}} \ge \frac{1}{2} e^{-o(\log\frac{1}{\delta})} =\frac{1}{2} \delta^{o(1)} \gg \delta$ as desired, where the first inequality is the data processing inequality for KL-divergence.
\end{proof}

\subsection{Upper Bounding the Squared Hellinger Distance for Single-Coin Adaptive Algorithms}
\label{sect:SingleCoinLower}

In this section, we prove Corollary~\ref{cor:BoundedLower}, though significant technical details are deferred to the appendix. Explicitly, we analyze a simplified scenario in Proposition~\ref{prop:BoundedLower}, after discussing why each of the simplifying assumptions does not give up generality, and cannot affect the key ``squared Hellinger distance per sample" quantity by more than a constant factor.

\begin{enumerate}
    \item Consider a single-coin algorithm. We restrict our attention to algorithms that only stop once they have seen a number of coin flips that is exactly a power of 2. Any stopping rule $S$ that potentially stops in between powers of 2 could be converted into an almost-equivalent rule $S'$ by collecting coin flips up to the next power of 2 and discarding them as necessary so as to simulate $S$: this will sacrifice at most a factor of 2 in sample complexity, and can only increase our Hellinger distance (since discarding data is a form of ``data processing" and thus we may apply the data processing inequality). Thus the new $S'$ will have ``squared Hellinger distance per sample" at least half that of $S$.
    \item By standard symmetrization arguments, a single-coin algorithm can always be implemented such that decisions only depend on the \emph{number} of flips for a coin as well as the \emph{number of observed ``heads"}, as opposed to the explicit \emph{sequence} of heads/tails observations. Thus we restrict our attention to stopping rules in the sense of Algorithm~\ref{Alg:TriangularWalk}, specified in full generality by a triangle of stopping coefficients $\{\gamma_{n,k}\}$.

    \item There is in some sense a ``phase change" once an algorithm has received $\Omega(\frac{1}{\Delta^2})$ samples from a single coin: after this point, the algorithm might have good information about whether the coin is of type $\frac{1}{2}+\Delta$ versus type $\frac{1}{2}-\Delta$, and might productively make subtle adaptive decisions after this point. We restrict our analysis to the regime where \emph{no coin is flipped more than $10^{-8}/\Delta^2$ times}: formally, we show an impossibility result in the following stronger setting, where we assume that whenever a single coin is flipped $10^{-8}/\Delta^2$ times, then the coin's true bias (either $\frac{1}{2}+\Delta$ or $\frac{1}{2}-\Delta$) is \emph{immediately} revealed to the algorithm.
Thus any coin flips beyond $10^{-8}/\Delta^2$ that an algorithm desires can instead be simulated at no cost.

Formally, an impossibility result in this setting with ``advice" (Proposition~\ref{prop:BoundedLower}) implies the analogous result in the original setting (Corollary~\ref{cor:BoundedLower}) by the data processing inequality for Hellinger distance (since Hellinger distance is an $f$-divergence): simulating additional coin flips in terms of ``advice" is itself ``data processing", and thus can only decrease the Hellinger distance. Thus the setting without advice has smaller-or-equal Hellinger distance, and uses greater-or-equal number of samples, and hence the bound on their ratio in Proposition~\ref{prop:BoundedLower} implies the corresponding bound in Corollary~\ref{cor:BoundedLower}.

\end{enumerate}

\begin{proposition}
	\label{prop:BoundedLower}
	Consider an arbitrary stopping rule $\{\gamma_{n,k}\}$ that 1) is non-zero only for $n$ that are powers of 2, and 2) $\gamma_{10^{-8}/\Delta^2,k} = 1$ for all $k$, that is the random walk always stops if it reaches $10^{-8}/\Delta^2$ coin flips.
	Suppose that given a coin, after a random walk on the Pascal triangle according to the stopping rule, the position $(n,k)$ that the walk ended at is always revealed, and furthermore, if $n = 10^{-8}/\Delta^2$, then the bias of the coin is also revealed.
	Let $H^2$ be the squared Hellinger distance between a single run of the above process where 1) a coin with bias $\frac{1}{2} + \Delta$ is used with probability $\rho$ and a coin with bias $\frac{1}{2} - \Delta$ is used otherwise versus 2) a coin with bias $\frac{1}{2} + \Delta$ is used with probability $\rho+\eps$ and a coin with bias $\frac{1}{2} - \Delta$ is used otherwise.
	Furthermore, let $\Exp_{\rho}[n]$ and $\Exp_{\rho+\frac{\eps}{2}}[n]$ be the expected number of coin flips during a run of the algorithm where we use a $\frac{1}{2}+\Delta$ coin with probability $\rho$ and $\rho+\frac{\eps}{2}$ respectively, and a $\frac{1}{2}-\Delta$ coin otherwise.
	If all of $\rho$, $\eps$, $\Delta$ and $\eps/\rho$ are smaller than some universal absolute constant, then
	$$ \max\left[\frac{H^2}{\Exp_{\rho}[n]}, \frac{H^2}{\Exp_{\rho+\frac{\eps}{2}}[n]}\right] = O\left(\frac{\eps^2\Delta^2}{\rho}\right) $$
\end{proposition}

It remains to prove Proposition~\ref{prop:BoundedLower}.
For the rest of the section, we shall use the notation $h^+_{n,k} = (\frac{1}{2}+\Delta)^k(\frac{1}{2}-\Delta)^{n-k}$ and $h^-_{n,k} = (\frac{1}{2}-\Delta)^k(\frac{1}{2}+\Delta)^{n-k}$ for convenience.
The proofs for upper bounding $\frac{H^2}{\Exp_{\rho}[n]}$ and $\frac{H^2}{\Exp_{\rho+\frac{\eps}{2}}[n]}$ are essentially the same, and here we give the high-level outline of the proof for bounding the latter, with calculations and details in Appendix~\ref{app:lower}.

The first step in the proof is the following lemma that writes out the squared Hellinger distance induced by a given stopping rule $\{\gamma_{n,k}\}$, whose proof can be found in the appendix. The expression in Lemma~\ref{lem:HellingerForm} avoids square roots and in other ways simplifies aspects of the squared Hellinger distance by estimating terms to within a constant factor, which is folded into a multiplicative ``big-$\Theta$" term at the start of the expression. The two lines in the expression below capture the different forms of the Hellinger distance for stopping \emph{before} the last row versus \emph{at} the last row---recall that we prove impossibility under the stronger model where, upon reaching the last row the algorithm receives the true bias of the coin (as ``advice"). Thus the squared Hellinger distance coefficients from elements of the last row are typically much larger than for other rows, capturing the cases when this advice is valuable. Recall from Definition~\ref{Def:Alpha} that $\{\alpha_{n,k}\}$ is defined from the stopping rule $\{\gamma_{n,k}\}$, so that when multiplied by $h^+_{n,k}$ or $h^-_{n,k}$ respectively, it equals the probability of encountering $(n,k)$ without necessarily stopping there, in the cases of positive and negative bias respectively.
\begin{lemma}
\label{lem:HellingerForm}
Consider the two probability distributions in Proposition~\ref{prop:BoundedLower} over locations $(n,k)$ in the Pascal triangle of depth $10^{-8}/\Delta^2$ and bias $p \in \{\frac{1}{2}\pm\Delta\}$, generated by the given stopping rule $\{\gamma_{n,k}\}$ in the two cases 1) a coin with bias $\frac{1}{2} + \Delta$ is used with probability $\rho$ and a coin with bias $\frac{1}{2} - \Delta$ is used otherwise versus 2) a coin with bias $\frac{1}{2} + \Delta$ is used with probability $\rho+\eps$ and a coin with bias $\frac{1}{2} - \Delta$ is used otherwise.
If $\eps/\rho$ is smaller than some universal constant, then the squared Hellinger distance between these two distributions can be written as

\begin{align*}
	\Theta(\eps^2) \Biggl[&\sum_{n < \frac{10^{-8}}{\Delta^2}, k \in [0..n]} \alpha_{n,k} ((\rho+\frac{\eps}{2}) h^+_{n,k} + (1-\rho-\frac{\eps}{2}) h^-_{n,k}) \frac{(h^+_{n,k}-h^-_{n,k})^2}{(\rho h^+_{n,k} + (1-\rho) h^-_{n,k})^2}\\
	+& \sum_{n = \frac{10^{-8}}{\Delta^2}, k \in [0..n]} \alpha_{n,k} ((\rho+\frac{\eps}{2}) h^+_{n,k} + (1-\rho-\frac{\eps}{2}) h^-_{n,k}) \frac{\frac{h^+_{n,k}}{\rho} + h^-_{n,k}}{\rho h^+_{n,k} + (1-\rho) h^-_{n,k}} \Biggr]
\end{align*}
\end{lemma}

Intuitively, Lemma~\ref{lem:HellingerForm} breaks up the squared Hellinger distance into its contributions from each location $(n,k)$ in the triangle, with the  coefficient $\alpha_{n,k}$ depending on the stopping rule (proportional to the algorithm's probability of stopping at location $(n,k)$\,), and the remaining portion of the expression depending only on $n,k,\Delta,\rho$, with the $\eps$ dependence already factored out in the initial $\Theta(\eps^2)$ term.

The rest of the analysis uses the above tools to upper bound the squared Hellinger distance per sample.
For concrete details and calculations, please refer to Appendix~\ref{app:lower}.
The high level idea of the analysis is to split the expression of Lemma~\ref{lem:HellingerForm} for the total squared Hellinger distance per sample into three components, with the contribution from each location $(n,k)$ assigned to either 1) the last row $n = \frac{10^{-8}}{\Delta^2}$, 2) a ``high discrepancy region" where $h^+_{n,k}/h^-_{n,k} \ge 1/\rho^{0.1}$ which is towards the right of the triangle, potentially contributing large amounts to the squared Hellinger distance and 3) a ``central" region that is the rest of the triangle.
The last row, because of the nature of ``advice", clearly needs its own analysis.
As for the rest of the triangle, we divide it into the ``central" and ``high discrepancy" regions, and bound their contributions to the squared Hellinger distance per sample using different strategies.
For the central region, the key insight is that the squared Hellinger distance term is bounded by a well-behaved quadratic function in that region.
On the other hand, for the high discrepancy region, the key observation is that the region is defined such that it is a large number of standard deviations away from where a non-stopping random walk on the Pascal triangle should concentrate, and thus it is very unlikely for the algorithm to enter that region.
We take additional care to show that, for any stopping rule used by any algorithm, it \emph{cannot} sufficiently skew the distribution of where the walk ends up---for example, while the distribution might skew to the right if the algorithm stops whenever it enters the ``left" side of the triangle, we show that this cannot significantly save on expected sample complexity nor substantially increase the squared Hellinger distance per sample.
The analysis for the high discrepancy region makes crucial use of our simplifying assumption that the stopping rule \emph{only} stops at powers of 2 coin flips, letting us analyze large sequences of coin flips at a time, where we may take advantage of the tight concentration of the Binomial distribution over sufficiently many coin flips to bound the effect of any skewing-towards-the-right that can be introduced by the stopping rule.

Propositions~\ref{prop:Wizard},~\ref{prop:scary} and~\ref{prop:notscary} in Appendix~\ref{app:lower} assert that for each of the respective regions, their contribution to the squared Hellinger distance, divided by the expected sample complexity, is at most $O(\eps^2\Delta^2/\rho)$.
Summing up the three terms is an upper bound on the \emph{total} squared Hellinger distance per expected sample of $O(\eps^2\Delta^2/\rho)$, completing the proof of Proposition~\ref{prop:BoundedLower}.


\section{Experimental Results}
\label{sec:experiments}

We give simulation results to demonstrate the practical efficacy of our proposed algorithm.
In our experimental setups, we compare the convergence rates of 1) our algorithm (``T-WALK (15)" on the plots), 2) the natural majority vote method mentioned in the Introduction (``VOTING" on the plots) and 3) the ``SWITCH" method proposed in previous work by Chung et al.~\cite{Chung:2017} which has been observed to perform well in practice, but does not have a theoretical analysis.
For our algorithm, we choose the maximum number of flips for a single coin to be 15 ($=c\log\frac{1}{\eps}$) in Algorithm~\ref{Alg:Single}.
We also make the assumption that the noise parameter satisfies $\Delta \ge 0.3$, meaning that we can use Algorithm~\ref{Alg:Single} directly instead of using Algorithm~\ref{Alg:Main} to simulate virtual coins before feeding them into Algorithm~\ref{Alg:Single}.
To further improve the practical performance of Algorithm~\ref{Alg:Single}, we ran a local search method to improve on the non-zero output coefficients in Step 3(d) of Algorithm~\ref{Alg:Single}, using the assumption that $\Delta \ge 0.3$.
Concretely, recall that we output a non-zero coefficient when the maximum number of coin flips (15) has occurred and the majority of coin flips has been heads. Thus for $k \in \{8,\ldots,15\}$, we output 8: 0, 9: 6.913, 10: 5.032, 11: 2.101, 12: 0.636, 13: 1.965, 14: 1.016, 15: 1.009.
We note again that these coefficients are \emph{reusable} in practice, as long as the $\Delta \ge 0.3$ assumption can be made.

Figure~\ref{table:experiments} presents the experimental results, for ``coin quality" $\Delta=0.3$ or $0.4$, and ground truth fraction of positive coins $\rho$ taking representative values $0.005, 0.01, 0.03,$ or $0.1$. For each plot, the $x$-axis corresponds to the number of coin flips, with all algorithms eventually converging to the ground truth for enough coin flips. Standard deviation bars are computed over 10 runs of each different setting.
The estimates, given a strict budget of coin flips as given by the $x$-coordinate, are computed according to Section~\ref{sec:practice}.

In all cases, our algorithm (plotted in yellow) performs close to the ground truth (horizontal black line), while the alternative algorithms take longer to converge, or have high variance, as depicted by the error bars. In particular, as discussed in the introduction, our adaptive methods have the most potential for improvement in the more challenging and more practical regime where $\rho$ is small (top few plots), and where $\Delta$ is smaller (left column).

\begin{figure}[t]
    \centering
    \begin{tabular}{cc}
    \includegraphics[width = 6cm, keepaspectratio]{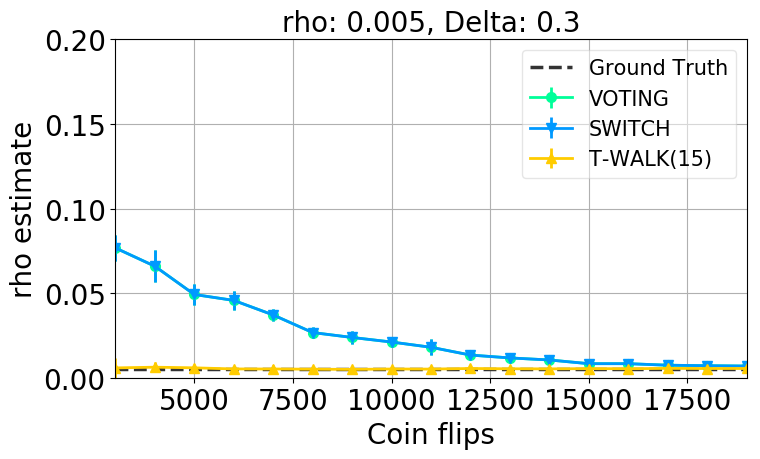} & \includegraphics[width = 6cm, keepaspectratio]{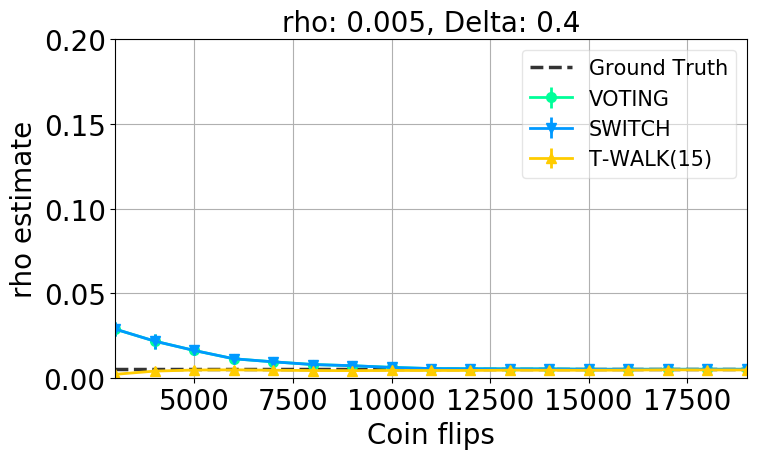}\\
    \includegraphics[width = 6cm, keepaspectratio]{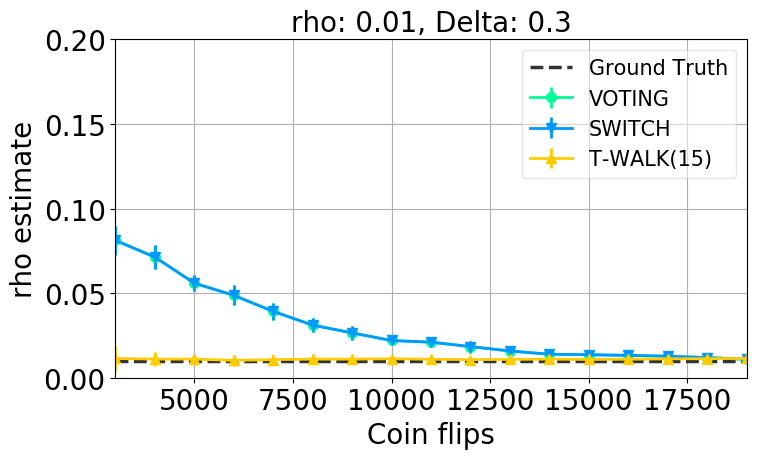} & \includegraphics[width = 6cm, keepaspectratio]{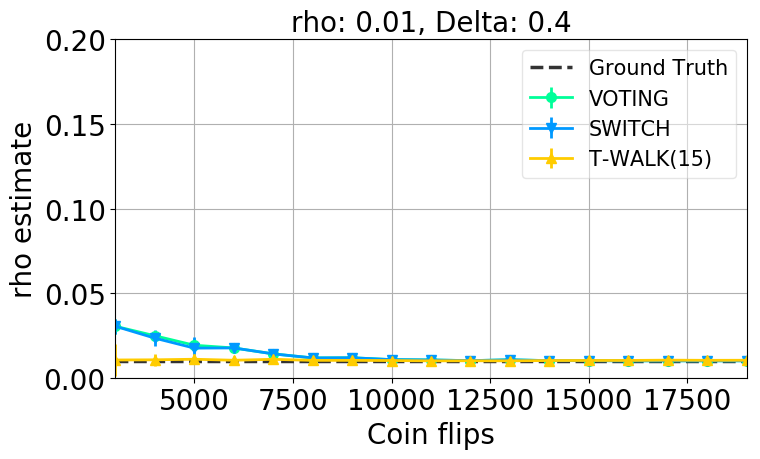}\\
    \includegraphics[width = 6cm, keepaspectratio]{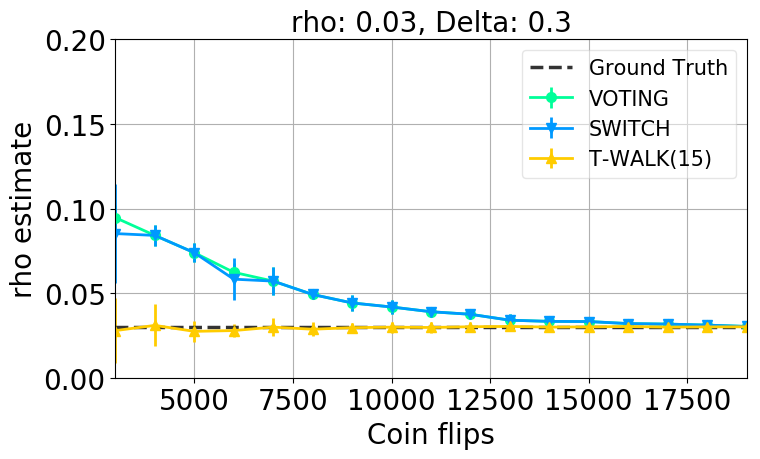} & \includegraphics[width = 6cm, keepaspectratio]{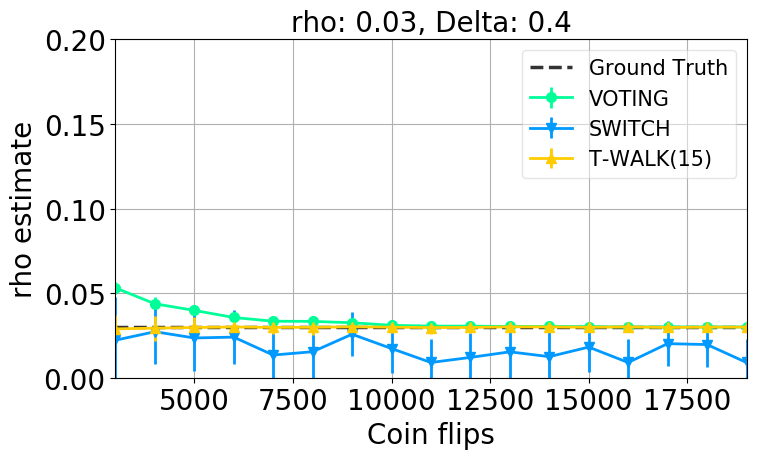}\\
    \includegraphics[width = 6cm, keepaspectratio]{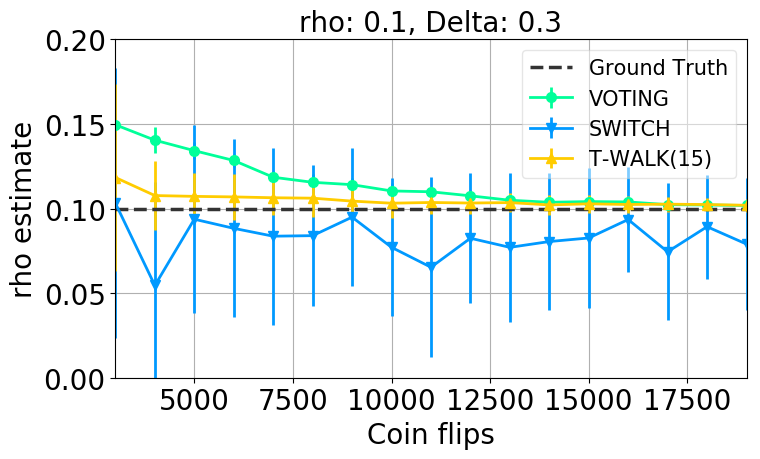} & \includegraphics[width = 6cm, keepaspectratio]{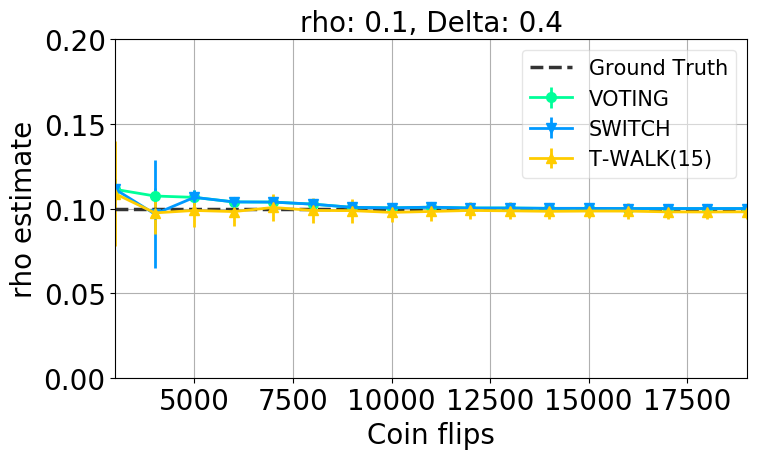}\\
    \end{tabular}
    \caption{Experimental Results}
    \label{table:experiments}
\end{figure}

\section{Algorithm for Known Conditional Distributions of Coins}

We now present our algorithms for the scenario where we know 1) the conditional distribution $h^+$ of the biases of positive coins, 2) analogously the distribution $h^-$ for the negative coins, as well as, for now circularly 3) the mixture parameter $\rho$ itself.
In practice, of course, we would only have an estimate $\rho$ for the mixture parameter itself, with the goal being to refine the estimate.
Assuming for the sake of analysis that our knowledge of the two conditional distributions as well as the mixture parameter are perfect (even if in practice they are only guesses), we derive a simple method based on linear and quadratic programming tools for computing the triangular walk linear estimator (an instantiation of Algorithm~\ref{Alg:TriangularWalk} in Section~\ref{sect:Framework}) with the \emph{minimum variance} subject to the constraints that 1) the estimator has expected output exactly 0 when given a randomly chosen negative coin, and 2) expected output exactly 1 for a randomly chosen positive coin.
That is, we enforce that the estimator is unbiased no matter what the true mixture parameter is, but we optimize its variance given our (assumed to be perfect) knowledge of the mixture parameter.

This method is practically relevant as a bootstrapping approach.
If our estimates of the conditional distributions and mixture parameter are indeed close to the ground truth, then it is easy to show bounds on the decrease in the estimator's performance as our estimates deviate from the truth.
As such, we focus on the analysis of the method when our knowledge of the parameters are assumed to be perfect.
The sample complexity of our algorithm is given in Theorem~\ref{thm:QPUpper}.

To complement the above upper bound result, we show that the linear estimator constructed from perfect knowledge of the relevant parameters is essentially an optimal estimator (Theorem~\ref{thm:QPOpt}) up to constant factors in sample complexity, under those exact same parameters.
This gives strong evidence for the unique algorithmic challenges presented by the ``uncertainty about uncertainty" regime of our problem, as discussed at the beginning of the paper.

\subsection{A Quadratic+Linear Programming Approach}
\label{sect:QP}

In this section we shall use extensively the notations $\alpha_{n,k}, \beta_{n,k}$ and $\gamma_{n,k}$ defined in Definition~\ref{Def:Alpha}.

We now give an overview on the steps required to derive the minimum variance unbiased estimator (in the form of Algorithm~\ref{Alg:TriangularWalk}), as described at the beginning of the section.
First, we assume that we are given a fixed stopping rule, and derive output coefficients for the corresponding linear estimator that has minimum variance.
We formulate a quadratic program (Figure~\ref{Fig:QP}) with the output coefficients $\{v_{n,k}\}$ as the variables, fixing $\alpha_{n,k}$ as constants.
The quadratic program can be solved analytically, which allows us to derive for $\{v_{n,k}\}$ closed form expressions that makes an unbiased estimator with minimum variance assuming the given stopping rule, as well as perfect knowledge of the conditional distribution of biases and the mixture parameter.
Furthermore, the objective value (a function in $\alpha_{n,k}$) of the quadratic program turns out (Lemma~\ref{Lem:Vvalue}) to be the reciprocal of a \emph{linear} function in terms of $\alpha_{n,k}$.
With this representation of the objective, then, we can use the structural observations in Section~\ref{sect:Framework} to formulate a \emph{linear} program that solves for the optimum stopping rule given the conditional distributions of biases (conditional on a positive coin, or a negative coin) and mixture parameter.
In practice, the linear program is first solved to give the stopping rule, then the output coefficients can be calculated from the first step in the analysis.

Having the above overview in mind, we describe the details of the derivation.
To simplify notation, let $h^{-}_{n,k}$ be shorthand for $\Exp_{p \from h^-}\left(p^k(1-p)^{n-k}\right)$ (a generalization of the notation from Section~\ref{sect:UnknownLower}), and similarly for $h^{+}_{n,k}$.
Thus $\alpha_{n,k}h^{-}_{n,k}$ is the probability that if we randomly choose a negative coin, executing the triangular walk with that coin will stop at state $(n,k)$.
Similarly, $\alpha_{n,k} h^+_{n,k}$ is the analogous probability using a randomly chosen positive coin instead.

The quadratic program mentioned above is given in Figure~\ref{Fig:QP}.
We use variables $\{\tilde{v}_{n,k}\}$, constraining them such that the expected output over a randomly chosen positive coin (from distribution $h^+$) has value 1 greater than that over a randomly chosen negative coin (from distribution $h^-$).
Under this constraint, we minimize the second moment of the output when items are drawn from the mixture $\rho h^+ + (1-\rho) h^-$.
Any optimal solution to this optimization will choose the variables $\{\tilde{v}_{n,k}\}$ such that the expected output of an item drawn from the universe is 0, implying that $\sum_{n,k} \alpha_{n,k}h^{+}_{n,k} \tilde{v}_{n,k} = 1-\rho$ and $\sum_{n,k} \alpha_{n,k}h^{-}_{n,k} \tilde{v}_{n,k} = -\rho$
Therefore, we can compute $\{v_{n,k}\}$ using $\{\tilde{v}_{n,k}\}$ by setting $v_{n,k} = \tilde{v}_{n,k} + \rho$.
As a consequence, $\sum_{n,k} \alpha_{n,k}h^{+}_{n,k} v_{n,k} = 1$ and $\sum_{n,k} \alpha_{n,k}h^{-}_{n,k} v_{n,k} = 0$, satisfying the unbiasedness requirement as desired.

\begin{figure}
\centering
\fbox{
\begin{tabular}{lll}
\vspace*{.4em}
	minimize & $\sum_{n,k} \alpha_{n,k}\left(\rho h^{+}_{n,k} + (1-\rho)h^{-}_{n,k}\right)\tilde{v}_{n,k}^2$\\
	subject to & $\sum_{n,k} \alpha_{n,k}h^{+}_{n,k}\tilde{v}_{n,k} = 1+\sum_{n,k} \alpha_{n,k}h^{-}_{n,k}\tilde{v}_{n,k} $\\
\end{tabular}}
	\label{Fig:QP}
	\caption{A QP formulation for computing the output coefficients in terms of the stopping rule}
\end{figure}

The quadratic program in Figure~\ref{Fig:QP} can be solved analytically using Langrange multipliers.
We give the results as Lemma~\ref{Lem:Vvalue}, and defer the calculations to the appendix.

\begin{lemma}
	\label{Lem:Vvalue}
	For the quadratic program in Figure~\ref{Fig:QP}, the optimal assignments to $\{\tilde{v}_{n,k}\}$ are
	$$ \tilde{v}_{n,k} = \frac{\frac{h^{+}_{n,k}-h^{-}_{n,k}}{\rho h^{+}_{n,k} + (1-\rho)h^{-}_{n,k}}}{\sum_{m,j} \alpha_{m,j}\frac{(h^{+}_{m,j} - h^{-}_{m,j})^2}{\rho h^{+}_{m,j} + (1-\rho)h^{-}_{m,j}}}$$
	(and we choose $v_{n,k} = \tilde{v}_{n,k} + \rho$), giving an objective value of
	$$ \frac{1}{\sum_{n,k} \alpha_{n,k}\frac{(h^{+}_{n,k} - h^{-}_{n,k})^2}{\rho h^{+}_{n,k} + (1-\rho)h^{-}_{n,k}}}$$
\end{lemma}

As mentioned at the beginning of the section, the optimal objective value of the quadratic program, namely the minimum variance achievable given a stopping rule, is the reciprocal of a \emph{linear} function in $\{\alpha_{n,k}\}$.
Note that the total sample complexity of the linear estimator, if we use the median-of-means method to estimate its expectation, is proportional to product of the variance of the linear estimator and the expected sample complexity of one run of the random walk.
Therefore, if we fix the expected sample complexity of one run to be $n_0$, we can in fact optimize the total sample complexity by minimizing the variance over all possible stopping rules with the expected sample complexity of $n_0$.
Observe that the reciprocal of the variance, divided by $n_0$, is simply the reciprocal of the total sample complexity of the stopping rule, that we would therefore like to \emph{maximize}.
Moreover, such function is a linear function in $\{\alpha_{n,k}\}$.
Thus, we can write the optimization problem as the linear program in Figure~\ref{Fig:LP}, by taking the objective to maximize the reciprocal of the quadratic program solution, divided by $n_0$.
The program includes (slightly adapted versions of) the recurrence relations introduced in Equation~\ref{Eq:Recurrence} as constraints.
Moreover, in order to control the sample complexity of the algorithm, the program also contains constraints enforcing that 1) the expected number of responses solicited for a random item is bounded by $n_0$ and 2) the maximum depth of the triangle is bounded by some parameter $\nmax$.
In addition to the interpretation as the maximum amount of resources we would ever invest on a single coin/item, the maximum depth constraint can also be interpreted as a computational constraint on how much time we can spend on computing the description of the linear estimator.

\begin{figure}
\centering
\fbox{\begin{tabular}{lll}
\vspace*{.4em}
	maximize & $\frac{1}{n_0}\sum_{n,k} \frac{(h^{+}_{n,k} - h^{-}_{n,k})^2}{\rho h^{+}_{n,k} + (1-\rho)h^{-}_{n,k}} \, \alpha_{n,k}$\\
	subject to & $\beta_{0,0} = 1$\\
	& $\beta_{n+1,k+1} = \beta_{n,k+1} - \alpha_{n,k+1} + \beta_{n,k} - \alpha_{n,k}$\\
	& $\alpha_{n,k} \le \beta_{n,k}$\\
	& $\alpha_{\nmax,k} = 1$ for all $k$ (Max depth constraint)\\
	& $\sum_{n,k} n \cdot \alpha_{n,k}(\rho h^{+}_{n,k} + (1-\rho) h^{-}_{n,k}) \le n_0$ (Bounding expected sample complexity)\\
	where & $\alpha_{n,k}, \beta_{n,k} \ge 0$\\
\end{tabular}}
	\caption{An LP formulation for finding the best stopping rule given an expected sample complexity}
	\label{Fig:LP}
\end{figure}

Since, ultimately, we wish to optimize over all possible values in $n_0$, such a linear program formulation (in Figure~\ref{Fig:LP}) cannot be used directly.
However, consider the following rewriting of the program.
We can always divide the $\{\alpha_{n,k}, \beta_{n,k}\}$ variables by $n_0$ and not change the meaning of the program, if we rescale the constraints and objective correspondingly.
This modification has the following effects: it 1) changes the $n_0$ in the objective and the fifth constraint into $1$, 2) preserves the second and third constraints as well as the non-negativity constraints and 3) changes the first and fourth constraints into ``variable = $1/n_0$".
The first and fourth constraints are now the only components in the new program that depend on $n_0$, and since we ultimately wish to optimize over all possible $n_0$, we can replace these constraints with the weaker constraint that they all equal to each other without specifying what they are equal to.
This results in the linear program in Figure~\ref{Fig:NewLP}, which by the above reasoning is equivalent to optimizing the total sample complexity for a stopping rule.

\begin{figure}
\centering
\fbox{\begin{tabular}{lll}
\vspace*{.4em}
	maximize & $\sum_{n,k} \frac{(h^{+}_{n,k} - h^{-}_{n,k})^2}{\rho h^{+}_{n,k} + (1-\rho)h^{-}_{n,k}} \, \alpha_{n,k}$\\
	subject to & $\beta_{n+1,k+1} = \beta_{n,k+1} - \alpha_{n,k+1} + \beta_{n,k} - \alpha_{n,k}$\\
	& $\alpha_{n,k} \le \beta_{n,k}$\\
	& $\alpha_{\nmax,k} = \beta_{0,0}$ for all $k$ (Max depth constraint)\\
	& $\sum_{n,k} n \cdot \alpha_{n,k}(\rho h^{+}_{n,k} + (1-\rho) h^{-}_{n,k}) \le 1$\\
	where & $\alpha_{n,k}, \beta_{n,k} \ge 0$\\
\end{tabular}}
	\caption{An LP formulation for finding the best stopping rule independent of the expected sample complexity for a single coin}
	\label{Fig:NewLP}
\end{figure}

To obtain the optimal stopping rule $\{\gamma_{n,k}\}$, we solve the linear program in Figure~\ref{Fig:NewLP}, rescale every variable such that $\beta_{0,0} = 1$, and calculate $\gamma_{n,k} = \alpha_{n,k}/\beta_{n,k}$.
If the solution to the linear program (in Figure~\ref{Fig:NewLP}) is $1/S$, then the expected sample complexity is $O(\frac{S}{\eps^2}\log\frac{1}{\delta})$ to estimate $\rho$ to within an additive $\eps$ with probability at least $1-\delta$.
This can be achieved by taking the median-of-means of $O(\log \frac{1}{\delta})$ groups of samples of size $O(S/\eps^2)$, each of which has a constant probability concentration to within additive $\eps$ by Chebyshev's inequality.
Summarizing the above gives the following theorem.

\medskip\noindent{\bf Theorem~\ref{thm:QPUpper}.}
\emph{Suppose we are given 1) the distribution of coin biases conditioned on being a positive coin, 2) the analogous distribution for negative coins and 3) the mixture parameter $\rho$ (which, again, is a circular assumption but useful for a bootstrapping approach).
Suppose further that we are given 4) the parameter $\nmax$, which controls the maximum depth of the triangular walk.}

\emph{Then, following the method described earlier in this section, we can find the linear estimator for $\rho$ that minimizes variance, subject to a) the expected output of the estimator on input a random positive coin is 1 and b) the analogous expected output for a random negative coin is 0.}

\emph{Moreover, if the objective of the linear program in Figure~\ref{Fig:NewLP} is $1/S$, then the expected sample complexity of the constructed linear estimator is $O(\frac{S}{\eps^2}\log\frac{1}{\delta})$, which will estimate $\rho$ to within an additive error of $\eps$ with probability at least $1-\delta$.}

\subsection{Optimality of such linear estimators}
\label{sect:QPLower}

In this section, we show that in fact, the linear estimators produced by the linear program in Figure~\ref{Fig:LP} are optimal compared with any single-coin adaptive but possibly non-linear estimators, subject to the same maximum depth constraints.

Our approach for lower bounding the sample complexity is to fix the distributions $h^+$ and $h^-$ of positive and negative coin biases respectively, and show that with a small number of samples, it is impossible to distinguish the case between A) a $\rho$ and $(1-\rho)$ mixture of positive and negative coins and B) a $(\rho+\eps)$ and $(1-\rho-\eps)$ mixture.
To show indistinguishability, we again use the notion of Hellinger distance.
Since each stopping rule induces different distribution on the Pascal triangle, under randomly chosen coins from each of the A and B scenarios, we will upper bound the (squared) Hellinger distance between the scenarios.

Lemma~\ref{lem:HellingerApprox} shows that the squared Hellinger distance is in fact a linear function in $\{\alpha_{n,k}\}$ and furthermore, in the regime where $\eps \ll \rho$, is within a constant factor of the objective in the linear program in Figure~\ref{Fig:NewLP}.
The coincidence will allow us to show matching lower bounds.

\begin{lemma}
\label{lem:HellingerApprox}
Consider an arbitrary stopping rule $\{\gamma_{n,k}\}$ giving coefficients $\{\alpha_{n,k}\}$.
If $\eps/\rho$ is smaller than some universal constant, then the squared Hellinger distance between 1) a  coin randomly chosen as in case A (described in the paragraphs above) inducing a distribution on the Pascal triangle given the stopping rule and 2) a coin randomly chosen as in case B instead, is
$$ \Theta(\eps^2)\sum_{n,k} \frac{(h^{+}_{n,k} - h^{-}_{n,k})^2}{\rho h^{+}_{n,k} + (1-\rho)h^{-}_{n,k}} \, \alpha_{n,k} $$
\end{lemma}

We defer the proof and calculations to Appendix~\ref{app:calc}, but it is completely analogous to that of Lemma~\ref{lem:HellingerForm}.

With Lemma~\ref{lem:HellingerApprox}, we now prove Theorem~\ref{thm:QPOpt}.

\medskip\noindent{\bf Theorem~\ref{thm:QPOpt}.}
	\emph{As in Theorem~\ref{thm:QPUpper}, suppose we are given 1) the distribution of coin biases conditioned on being a positive coin, 2) the analogous distribution for negative coins, 3) the mixture parameter $\rho$, as well as the parameter $\nmax$, which controls the maximum depth of the triangular walk.}
	
	\emph{The linear estimator produced from solving the linear program in Figure~\ref{Fig:NewLP}, as described in Theorem~\ref{thm:QPUpper}, has total expected sample complexity that is within a constant factor of any optimal single-coin adaptive algorithm with $\geq\frac{2}{3}$ probability of success, subject to the same maximum depth constraint.}
	
	\emph{Combining with a corollary of Lemma~\ref{Lem:Reduction}, restricted to fully-adaptive algorithms that invests at most $\nmax$ flips on any single coin, this shows that our linear estimator in fact has sample complexity within a constant factor of any fully-adaptive algorithm satisfying the maximum depth constraint for every single coin.}

\begin{proof}
Given an arbitrary stopping rule, if it induces a squared Hellinger distance of $H^2$ between the two cases with a single random walk, then we can lower bound the number of random walks needed in the single-coin adaptive algorithm in order to solve the distinguishing task with constant probability of success, by $\Theta(1/H^2)$, using the subaddivity of squared Hellinger distance, and that the total Hellinger distance needs to be at least constant to solve the distinguishing task.
Thus, if $n_0$ is the expected number of coin flips for a random walk, the overall expected sample complexity is lower bounded by $\Omega(n_0/H^2)$.
Since we need to find a lower bound that applies to \emph{all} single-coin adaptive algorithms, we need to find the smallest $n_0/H^2$ over all the possible stopping rules (subject to the same max-depth constraint), or equivalently, maximize $H^2/n_0$ (which can alternatively be interpreted as the squared Hellinger distance per expected sample).
Lemma~\ref{lem:HellingerApprox} tells us that we can replace $H^2$ with the expression in the lemma and lose no more than multiplicative constants.
Thus, if we fix $n_0$, finding the best lower bound up to multiplicative constants is equivalent to solving the optimization problem that is exactly the one in Figure~\ref{Fig:LP}, except for an extra factor of $\Theta(\eps^2)$ in the objective.
We again wish to maximize the $H^2/n_0$ over all possible choices of $n_0$ as well, and therefore, following the same reasoning as before, we arrive at the linear program that is essentially the one in Figure~\ref{Fig:NewLP}, again except for the factor of $\Theta(\eps^2)$ in the objective.
This linear program has no $n_0$ dependency, and has objective that is $\Theta(\eps^2)$ times that of the one in Figure~\ref{Fig:NewLP}, which is the reciprocal of the (expected) total sample complexity of the optimal linear estimator produced as described in Section~\ref{sect:QP}.
Summarizing, if the solution to the linear program in Figure~\ref{Fig:NewLP} is $1/S$, then the maximum $H^2/n_0$ over all possible stopping rules would be within a constant factor of $\eps^2/S$, giving a lower bound of $\Omega(S/\eps^2)$ on the expected sample complexity (under case A) for a constant probability of success in the task of distinguishing between case A and case B.
This lower bound matches the upper bound of $O(S/\eps^2)$ on the total sample complexity of the linear estimator we produce according to Theorem~\ref{thm:QPUpper}.

As given in the theorem statement, combining this result with a corollary of Lemma~\ref{Lem:Reduction} shows that our linear estimator is in fact competitive to within a constant factor in sample complexity with \emph{fully-adaptive} algorithms that invest at most $\nmax$ flips on any single coin.
\end{proof}

\bibliographystyle{plain}
\bibliography{triangle}

\begin{thebibliography}{10}

\bibitem{ACK14}
Jayadev Acharya, Cl{\'e}ment~L Canonne, and Gautam Kamath.
\newblock A chasm between identity and equivalence testing with conditional
  queries.
\newblock {\em arXiv preprint arXiv:1411.7346}, 2014.

\bibitem{BDKR05}
Tugkan Batu, Sanjoy Dasgupta, Ravi Kumar, and Ronitt Rubinfeld.
\newblock The complexity of approximating the entropy.
\newblock {\em SIAM J.~Comput}, 35(1):132--150, 2005.

\bibitem{bell2000environmental}
Graham Bell, Martin~J Lechowicz, and Marcia~J Waterway.
\newblock Environmental heterogeneity and species diversity of forest sedges.
\newblock {\em Journal of Ecology}, 88(1):67--87, 2000.

\bibitem{Belovs:2016}
Aleksandrs Belovs and Eric Blais.
\newblock A polynomial lower bound for testing monotonicity.
\newblock In {\em Proc.~STOC '16}, pages 1021--1032, 2016.

\bibitem{Braverman:2016}
Mark Braverman, Ankit Garg, Tengyu Ma, Huy~L. Nguyen, and David~P. Woodruff.
\newblock Communication lower bounds for statistical estimation problems via a
  distributed data processing inequality.
\newblock In {\em Proc.~STOC'16}, page 1011–1020, 2016.

\bibitem{brennan:2020}
Jennifer Brennan, Ramya~Korlakai Vinayak, and Kevin Jamieson.
\newblock Estimating the number and effect sizes of non-null hypotheses.
\newblock In {\em Proc.~ICML'20}, 2020.

\bibitem{Canonne:survey}
Cl{\'e}ment Canonne.
\newblock {A Survey on Distribution Testing. Your Data is Big. But is it Blue?}
\newblock 2017.

\bibitem{CRS14}
Cl{\'e}ment Canonne, Dana Ron, and Rocco~A Servedio.
\newblock Testing equivalence between distributions using conditional samples.
\newblock In {\em Proc.~SODA '14}, pages 1174--1192, 2014.

\bibitem{CR14}
Cl{\'e}ment Canonne and Ronitt Rubinfeld.
\newblock Testing probability distributions underlying aggregated data.
\newblock In {\em Proc.~ICALP '14}, pages 283--295, 2014.

\bibitem{Can15}
Cl{\'e}ment~L Canonne.
\newblock Big data on the rise?
\newblock In {\em Proc.~ICALP '15}, pages 294--305, 2015.

\bibitem{CRS15}
Cl{\'e}ment~L Canonne, Dana Ron, and Rocco~A Servedio.
\newblock Testing probability distributions using conditional samples.
\newblock {\em SIAM J.~Comput}, 44(3):540--616, 2015.

\bibitem{CFGM13:2016}
Sourav Chakraborty, Eldar Fischer, Yonatan Goldhirsh, and Arie Matsliah.
\newblock On the power of conditional samples in distribution testing.
\newblock {\em SIAM J.~Comput}, 45(4):1261--1296, 2016.

\bibitem{Karp:2014}
Karthekeyan Chandrasekaran and Richard Karp.
\newblock Finding a most biased coin with fewest flips.
\newblock In {\em Proc.~COLT'14}, pages 394--407, 2014.

\bibitem{Chen:2017}
Xi~Chen, Erik Waingarten, and Jinyu Xie.
\newblock {Beyond Talagrand Functions: New Lower Bounds for Testing
  Monotonicity and Unateness}.
\newblock In {\em Proc.~STOC '17}, pages 523--536, 2017.

\bibitem{Chung:2017}
Yeounoh Chung, Sanjay Krishnan, and Tim Kraska.
\newblock {A Data Quality Metric (DQM): How to Estimate the Number of
  Undetected Errors in Data Sets}.
\newblock {\em Proc.~VLDB '17}, 10(10):1094--1105, 2017.

\bibitem{colwell1994estimating}
Robert~K Colwell and Jonathan~A Coddington.
\newblock Estimating terrestrial biodiversity through extrapolation.
\newblock {\em Phil. Trans. R. Soc. Lond. B}, 345(1311):101--118, 1994.

\bibitem{falahatgar2017maxing}
Moein Falahatgar, Yi~Hao, Alon Orlitsky, Venkatadheeraj Pichapati, and Vaishakh
  Ravindrakumar.
\newblock Maxing and ranking with few assumptions.
\newblock In {\em Proc.~NeurIPS '17}, pages 7060--7070, 2017.

\bibitem{Falahatgar:2015faster}
Moein Falahatgar, Ashkan Jafarpour, Alon Orlitsky, Venkatadheeraj Pichapati,
  and Ananda~Theertha Suresh.
\newblock Faster algorithms for testing under conditional sampling.
\newblock In {\em Proc.~COLT '15}, pages 607--636, 2015.

\bibitem{falahatgar2017maximum}
Moein Falahatgar, Alon Orlitsky, Venkatadheeraj Pichapati, and Ananda~Theertha
  Suresh.
\newblock Maximum selection and ranking under noisy comparisons.
\newblock In {\em Proc.~ICML '17}, pages 1088--1096, 2017.

\bibitem{GMV06}
Sudipto Guha, Andrew McGregor, and Suresh Venkatasubramanian.
\newblock Streaming and sublinear approximation of entropy and information
  distances.
\newblock In {\em Proc.~SODA '06}, pages 733--742, 2006.

\bibitem{Jamieson:2016}
Kevin Jamieson, Daniel Haas, and Ben Recht.
\newblock The power of adaptivity in identifying statistical alternatives.
\newblock In {\em Proc.~NIPS'16}, pages 775--783, 2016.

\bibitem{Levi:2013}
Reut Levi, Dana Ron, and Ronitt Rubinfeld.
\newblock Testing properties of collections of distributions.
\newblock {\em Theory of Computing}, 9(1):295--347, 2013.

\bibitem{Levi:2014}
Reut Levi, Dana Ron, and Ronitt Rubinfeld.
\newblock Testing similar means.
\newblock {\em SIAM Journal on Discrete Mathematics}, 28(4):1699--1724, 2014.

\bibitem{lord1965strong}
Frederic~M Lord.
\newblock A strong true-score theory, with applications.
\newblock {\em Psychometrika}, 30(3):239--270, 1965.

\bibitem{lord1975empirical}
Frederic~M Lord and Noel Cressie.
\newblock An empirical bayes procedure for finding an interval estimate.
\newblock {\em Sankhy{\=a}: The Indian Journal of Statistics, Series B}, pages
  1--9, 1975.

\bibitem{Malloy:2012}
Matthew~L Malloy, Gongguo Tang, and Robert~D Nowak.
\newblock Quickest search for a rare distribution.
\newblock In {\em Proc.~CISS'12}, pages 1--6, 2012.

\bibitem{millar1986distribution}
Wayne~J Millar.
\newblock Distribution of body weight and height: comparison of estimates based
  on self-reported and observed measures.
\newblock {\em Journal of Epidemiology \& Community Health}, 40(4):319--323,
  1986.

\bibitem{palmer1990small}
Michael~W Palmer and Philip~M Dixon.
\newblock Small-scale environmental heterogeneity and the analysis of species
  distributions along gradients.
\newblock {\em Journal of Vegetation Science}, 1(1):57--65, 1990.

\bibitem{RS09}
Ronitt Rubinfeld and Rocco~A Servedio.
\newblock Testing monotone high-dimensional distributions.
\newblock {\em Random Struct Algor}, 34(1):24--44, 2009.

\bibitem{Shah:2016}
N.~B. {Shah}, S.~{Balakrishnan}, and M.~J. {Wainwright}.
\newblock Feeling the bern: Adaptive estimators for bernoulli probabilities of
  pairwise comparisons.
\newblock In {\em 2016 IEEE International Symposium on Information Theory
  (ISIT)}, pages 1153--1157, 2016.

\bibitem{Shah:2015}
Nihar Shah, Dengyong Zhou, and Yuval Peres.
\newblock Approval voting and incentives in crowdsourcing.
\newblock In {\em Proc.~ICML '15}, pages 10--19, 2015.

\bibitem{Shah:2017}
Nihar~B Shah and Martin~J Wainwright.
\newblock Simple, robust and optimal ranking from pairwise comparisons.
\newblock {\em JMLR}, 18(1):7246--7283, 2017.

\bibitem{Tian:2017}
Kevin Tian, Weihao Kong, and Gregory Valiant.
\newblock Learning populations of parameters.
\newblock In {\em Proc.~NeurIPS '17}, pages 5778--5787, 2017.

\bibitem{Tsybakov:2008}
Alexandre~B Tsybakov.
\newblock {\em Introduction to nonparametric estimation}.
\newblock Springer Science \& Business Media, 2008.

\bibitem{VV:2011:power}
Gregory Valiant and Paul Valiant.
\newblock The power of linear estimators.
\newblock In {\em Proc.~FOCS '11}, pages 403--412, 2011.

\bibitem{vinayak2019maximum}
Ramya~Korlakai Vinayak, Weihao Kong, Gregory Valiant, and Sham Kakade.
\newblock Maximum likelihood estimation for learning populations of parameters.
\newblock In {\em International Conference on Machine Learning}, pages
  6448--6457, 2019.

\bibitem{Wang:2018}
Jingyan Wang and Nihar~B Shah.
\newblock Your 2 is my 1, your 3 is my 9: Handling arbitrary miscalibrations in
  ratings.
\newblock In {\em Proc.~AAMAS '19}, 2019.
\newblock To appear.

\end{thebibliography}

\newpage
\appendix

\section{Non-Adaptive Lower Bound}
\label{app:non-adaptive}

Here we give the remaining calculations for the non-adaptive lower bound of $O(\frac{\rho}{\eps^2\Delta^2}\log\frac{1}{\rho})$.

Recall that, to show a non-adaptive lower-bound, consider a random variable $S$ that uniformly chooses between scenarios ``$\rho$" and ``$\rho+\eps$" respectively, where coins will have bias $\frac{1}{2}+\Delta$ with probability $\rho$ or $\rho+\eps$ respectively depending on the outcome of $S$, and bias $\frac{1}{2}-\Delta$ otherwise.
We will show that the mutual information between $n$ flips of a single coin and the scenario variable $S$ is at most $O(n\frac{\eps^2\Delta^2}{\rho\log(1/\rho)})$, and thus that, even when combining information from several coins, at least $\Omega(\frac{\rho}{\eps^2\Delta^2}\log\frac{1}{\rho})$ samples are needed to distinguish the two scenarios with constant probability.

Let $\Bin(n,p,k)$ denote the probability that a Binomial distribution with $n$ trials and bias $p$ has value $k$.

The mutual information is exactly represented as \begin{align*}&\textstyle{\frac{1}{2}\sum_{k=0}^n (\rho\Bin(n,\frac{1}{2}+\Delta,k)+(1-\rho)\Bin(n,\frac{1}{2}-\Delta,k))\log\frac{(\rho\Bin(n,\frac{1}{2}+\Delta,k)+(1-\rho)\Bin(n,\frac{1}{2}-\Delta,k))}{((\rho+\frac{\eps}{2})\Bin(n,\frac{1}{2}+\Delta,k)+(1-\rho-\frac{\eps}{2})\Bin(n,\frac{1}{2}-\Delta,k))}}\\
&\textstyle{}+((\rho+\eps)\Bin(n,\frac{1}{2}+\Delta,k)+(1-\rho-\eps)\Bin(n,\frac{1}{2}-\Delta,k))\log\frac{((\rho+\eps)\Bin(n,\frac{1}{2}+\Delta,k)+(1-\rho-\eps)\Bin(n,\frac{1}{2}-\Delta,k))}{((\rho+\frac{\eps}{2})\Bin(n,\frac{1}{2}+\Delta,k)+(1-\rho-\frac{\eps}{2})\Bin(n,\frac{1}{2}-\Delta,k))}\end{align*}

Claim is that, for $x,y\geq 0$, we have $x\log\frac{x}{(x+y)/2}+y\log\frac{y}{(x+y)/2}\leq \frac{(x-y)^2}{x+y}$.

Letting $x$ be $\rho\Bin(n,\frac{1}{2}+\Delta,k)+(1-\rho)\Bin(n,\frac{1}{2}-\Delta,k)$ and $y$ be the $\rho+\eps$ mixture analogue, the mutual information is less than or equal to:
\begin{align*}
&\frac{\eps^2}{4} \sum_{k=0}^n \frac{(\Bin(n,\frac{1}{2}+\Delta,k)-\Bin(n,\frac{1}{2}-\Delta,k))^2}{((\rho+\frac{\eps}{2})\Bin(n,\frac{1}{2}+\Delta,k)+(1-\rho-\frac{\eps}{2})\Bin(n,\frac{1}{2}-\Delta,k))}\\
\le\;&\eps^2 \sum_{k=0}^n \frac{(\Bin(n,\frac{1}{2}+\Delta,k)-\Bin(n,\frac{1}{2}-\Delta,k))^2}{(\rho\Bin(n,\frac{1}{2}+\Delta,k)+\Bin(n,\frac{1}{2}-\Delta,k))}
\end{align*}
Since $(x-y)^2\leq 2(x^2+y^2)$, and $\frac{1}{x+y}\leq \min\{\frac{1}{x},\frac{1}{y}\}$, we also have
\begin{align*}
    & \sum_{k=0}^n \frac{(\Bin(n,\frac{1}{2}+\Delta,k)-\Bin(n,\frac{1}{2}-\Delta,k))^2}{(\rho\Bin(n,\frac{1}{2}+\Delta,k)+\Bin(n,\frac{1}{2}-\Delta,k))}\\
    \leq \;& 2\sum_{k=0}^n \frac{\Bin(n,\frac{1}{2}+\Delta,k)^2+\Bin(n,\frac{1}{2}-\Delta,k)^2}{(\rho\Bin(n,\frac{1}{2}+\Delta,k)+\Bin(n,\frac{1}{2}-\Delta,k))}\\
    \leq \;& 2\min\left\{\sum_{k=0}^n \frac{\Bin(n,\frac{1}{2}+\Delta,k)^2}{\rho\Bin(n,\frac{1}{2}+\Delta,k))},\sum_{k=0}^n \frac{\Bin(n,\frac{1}{2}+\Delta,k)^2}{\Bin(n,\frac{1}{2}-\Delta,k))}\right\}+2\sum_{k=0}^n \frac{\Bin(n,\frac{1}{2}-\Delta,k)^2}{\Bin(n,\frac{1}{2}-\Delta,k))}\\
    = \; & 2\min\left\{\frac{1}{\rho},\left(\frac{1+12\Delta^2}{1-4\Delta^2}\right)^n\right\}+2
\end{align*}

For $\Delta$ bounded below by any universal positive constant, this last expression is $O(\min\{\frac{1}{\rho},e^{O(\Delta^2 n)}\})$. Since the components of the minimum are 1) equal for $n=O(\frac{1}{\Delta^2}\log\frac{1}{\rho})$ and 2) constant and convex in $n$ respectively, we can bound the minimum by a linear function that goes through this intersection point: for $n\geq\frac{1}{\Delta^2}$ the minimum is bounded by $O(n\frac{\Delta^2}{\rho\log\frac{1}{\rho}})$. Multiplying by $\epsilon^2$ gets a bound on the mutual information, and dividing by $n$ gets a bound on mutual information per sample of $O(\frac{\eps^2\Delta^2}{\rho\log\frac{1}{\rho}})$.

For the remaining regime of $n \le \frac{1}{\Delta^2}$:

\[\sum_{k=0}^n \frac{(\Bin(n,\frac{1}{2}+\Delta,k)-\Bin(n,\frac{1}{2}-\Delta,k))^2}{\rho\Bin(n,\frac{1}{2}+\Delta,k)+\Bin(n,\frac{1}{2}-\Delta,k)}\leq\sum_{k=0}^n \frac{(\Bin(n,\frac{1}{2}+\Delta,k)-\Bin(n,\frac{1}{2}-\Delta,k))^2}{\Bin(n,\frac{1}{2}-\Delta,k)}
=\left(\frac{1+12\Delta^2}{1-4\Delta^2}\right)^n-1\]

This last expression is $O(\Delta^2 n)$ for $n\leq\frac{1}{\Delta^2}$, and thus we can bound the mutual information per sample by $O(\eps^2\Delta^2)$ here.

Combining the two bounds, we conclude the mutual information per sample is at most $O(\frac{\eps^2\Delta^2}{\rho\log\frac{1}{\rho}})$ for all $n$, and thus its inverse, $O(\frac{\rho\log\frac{1}{\rho}}{\eps^2\Delta^2})$ lower-bounds the number of non-adaptive samples needed for our task.

\section{Proof of Proposition~\ref{prop:BoundedLower}}
\label{app:lower}

This appendix proves Proposition~\ref{prop:BoundedLower}, which upper bounds the squared Hellinger distance per sample for any single-coin algorithm of (without loss of generality) a particular form stated in the proposition.
We state the proposition again for the reader's convenience, as well as Lemma~\ref{lem:HellingerForm} that is introduced in Section~\ref{sect:SingleCoinLower}, which simplifies the expression of the squared Hellinger distance by sacrificing a constant factor.

\medskip\noindent{\bf Proposition~\ref{prop:BoundedLower}.}
\emph{
Consider an arbitrary stopping rule $\{\gamma_{n,k}\}$ that 1) is non-zero only for $n$ that are powers of 2, and 2) $\gamma_{10^{-8}/\Delta^2,k} = 1$ for all $k$, that is the random walk always stops if it reaches $10^{-8}/\Delta^2$ coin flips.
	Suppose that given a coin, after a random walk on the Pascal triangle according to the stopping rule, the position $(n,k)$ that the walk ended at is always revealed, and furthermore, if $n = 10^{-8}/\Delta^2$, then the bias of the coin is also revealed.
	Let $H^2$ be the squared Hellinger distance between a single run of the above process where 1) a coin with bias $\frac{1}{2} + \Delta$ is used with probability $\rho$ and a coin with bias $\frac{1}{2} - \Delta$ is used otherwise versus 2) a coin with bias $\frac{1}{2} + \Delta$ is used with probability $\rho+\eps$ and a coin with bias $\frac{1}{2} - \Delta$ is used otherwise.
	Furthermore, let $\Exp_{\rho}[n]$ and $\Exp_{\rho+\frac{\eps}{2}}[n]$ be the expected number of coin flips during a run of the algorithm where we use a $\frac{1}{2}+\Delta$ coin with probability $\rho$ and $\rho+\frac{\eps}{2}$ respectively, and a $\frac{1}{2}-\Delta$ coin otherwise.
	If all of $\rho$, $\eps$, $\Delta$ and $\eps/\rho$ are smaller than some universal absolute constant, then
	$$ \max\left[\frac{H^2}{\Exp_{\rho}[n]}, \frac{H^2}{\Exp_{\rho+\frac{\eps}{2}}[n]}\right] = O\left(\frac{\eps^2\Delta^2}{\rho}\right) $$
}

As mentioned in Section~\ref{sect:SingleCoinLower}, the proofs for bounding $\frac{H^2}{\Exp_{\rho}[n]}$ and $\frac{H^2}{\Exp_{\rho+\frac{\eps}{2}}[n]}$ are essentially identical, so here we present the proof only for the latter.

\medskip\noindent{\bf Lemma~\ref{lem:HellingerForm}.}\emph{
Consider the two probability distributions in Proposition~\ref{prop:BoundedLower} over locations $(n,k)$ in the Pascal triangle of depth $10^{-8}/\Delta^2$ and bias $p \in \{\frac{1}{2}\pm\Delta\}$, generated by the given stopping rule $\{\gamma_{n,k}\}$ in the two cases of 1) a coin with bias $\frac{1}{2} + \Delta$ is used with probability $\rho$ and a coin with bias $\frac{1}{2} - \Delta$ is used otherwise versus 2) a coin with bias $\frac{1}{2} + \Delta$ is used with probability $\rho+\eps$ and a coin with bias $\frac{1}{2} - \Delta$ is used otherwise.
If $\eps/\rho$ is smaller than some universal constant, then the squared Hellinger distance between these two distributions can be written as
\begin{align*}
	\Theta(\eps^2) \Biggl[&\sum_{n < \frac{10^{-8}}{\Delta^2}, k \in [0..n]} \alpha_{n,k} ((\rho+\frac{\eps}{2}) h^+_{n,k} + (1-\rho-\frac{\eps}{2}) h^-_{n,k}) \frac{(h^+_{n,k}-h^-_{n,k})^2}{(\rho h^+_{n,k} + (1-\rho) h^-_{n,k})^2}\\
	+& \sum_{n = \frac{10^{-8}}{\Delta^2}, k \in [0..n]} \alpha_{n,k} ((\rho+\frac{\eps}{2}) h^+_{n,k} + (1-\rho-\frac{\eps}{2}) h^-_{n,k}) \frac{\frac{h^+_{n,k}}{\rho} + h^-_{n,k}}{\rho h^+_{n,k} + (1-\rho) h^-_{n,k}} \Biggr]
\end{align*}}

We perform separate analyses on three regions of the Pascal triangle:
1) the last row $n = \frac{10^{-8}}{\Delta^2}$, 2) a ``high discrepancy region" where $h^+_{n,k}/h^-_{n,k} \ge 1/\rho^{0.1}$ which is towards the right of the triangle, potentially contributing large amounts to the squared Hellinger distance and 3) a ``central" region that is the rest of the triangle.
We shall show that each region contributes small squared Hellinger distance per sample, and thus their sum bounds the total squared Hellinger distance per sample, completing the proof of Proposition~\ref{prop:BoundedLower}.

We present the three analyses in the order of central region (Section~\ref{sec:non-scary}), high discrepancy region (Section~\ref{sect:scary}) and the last row (Section~\ref{sect:last}).

\subsection{``Central" Region}\label{sec:non-scary}

For the purposes of this section, define $b_{n,k,\rho+\frac{\eps}{2}}$ to equal $ ((\rho+\frac{\eps}{2}) h^+_{n,k} + (1-\rho-\frac{\eps}{2}) h^-_{n,k})$, so that $\alpha_{n,k} b_{n,k,\rho+\frac{\eps}{2}}$ is the probability of reaching and stopping at location $(n,k)$ under a $\rho+\frac{\epsilon}{2}$ mixture of the two coin types.
Further, let $R_{n,k,\rho}$ be defined to equal $\frac{(h^+_{n,k}-h^-_{n,k})^2}{(\rho h^+_{n,k} + (1-\rho) h^-_{n,k})^2}$, which is the contribution of location $(n,k)$ to the squared Hellinger distance \emph{per unit of probability mass that stops there}.

By Lemma~\ref{lem:HellingerForm}, the contribution to the squared Hellinger distance from the central region of the triangle is bounded by the sum, over this region, of $\epsilon^2 \alpha_{n,k} b_{n,k,\rho+\frac{\eps}{2}} R_{n,k,\rho}$.

\begin{proposition}
    \label{prop:notscary}
    
    For an arbitrary stopping rule, the contribution of the central region to the squared Hellinger distance, divided by the (total) expected sample complexity $\Exp_{\rho+\frac{\eps}{2}}[n]$ of the walk using a $\rho+\frac{\eps}{2}$ mixture of $\frac{1}{2}\pm\Delta$ coins, is at most $O(\eps^2\Delta^2/\rho)$. Explicitly, with notation for $b$ and $R$ defined in the previous paragraphs, we have
    \[\epsilon^2\sum_{n < \frac{10^{-8}}{\Delta^2}, k \text{ s.t. } \frac{h^+_{n,k}}{h^-_{n,k}} < \frac{1}{\rho^{0.1}}} \alpha_{n,k}b_{n,k,\rho+\frac{\eps}{2}} {R}_{n,k,\rho}=O\left(\frac{\epsilon^2\Delta^2}{\rho}\right){\Exp}_{\rho+\frac{\eps}{2}}[n]\]
\end{proposition}

\begin{proof}
 We upper bound this quantity here by instead 1) replacing $R_{n,k,\rho}$ by a similar quantity $\widehat{R}_{n,k,\rho}$ that is an upper bound on $R$ in the central region, and 2) summing over the entire triangle instead of just the central region.
Let $\widehat{R}_{n,k,\rho}=2\left(\min\left(\frac{h^+_{n,k}}{h^-_{n,k}}, \frac{1}{\rho^{0.1}}\right) - 1\right)^2$. This bounds $R$ in the region where $h^+_{n,k}/h^-_{n,k}\leq 1/\rho^{0.1}$:
in this regime, $\widehat{R} = \frac{2(h^+_{n,k} - h^-_{n,k})^2}{(h^-_{n,k})^2}$.
The numerator of $R$ is at most $\frac{1}{2}$ of the numerator of $\widehat{R}$, and the denominator of $R$ is at least $\frac{1}{2}$ of the denominator of $\widehat{R}$.

Thus we instead prove the related fact that \begin{equation}\label{eq:non-scary-bound}\epsilon^2\sum_{n \leq \frac{10^{-8}}{\Delta^2},\,k\in[0..n]} \alpha_{n,k}b_{n,k,\rho+\frac{\eps}{2}} \widehat{R}_{n,k,\rho}=O\left(\frac{\epsilon^2\Delta^2}{\rho}\right){\Exp}_{\rho+\frac{\eps}{2}}[n]\end{equation}

We prove this by induction on a row $i$, where we define $A^i_{n,k}$ to be the stopping probabilities (corresponding to the product $\alpha_{n,k}b_{n,k,\rho+\frac{\eps}{2}}$) for the variant of the given stopping rule where we \emph{force} the rule to stop at row $i$ if it reaches this row; analogously define ${\Exp}^i_{\rho+\frac{\eps}{2}}[n]$ to be the expected number of samples taken by this rule. We consider how both the left hand side and ${\Exp}^i_{\rho+\frac{\eps}{2}}[n]$ change as we increase $i$ by 1, and show that the ratio of their change is $O\left(\frac{\epsilon^2\Delta^2}{\rho}\right)$.

See Lemma~\ref{lem:nonScaryExcessHellinger} for a proof of this fact.
The proof of the lemma rely on the concrete definitions of $b_{n,k,\rho+\frac{\eps}{2}}$ and $R_{n,k,\rho}$, and so both the lemma statement and the proof write out the expressions for purposes of calculations.

As a proof sketch of the ground covered by Lemma~\ref{lem:nonScaryExcessHellinger}: if for some location $(i,k)$ some amount of probability mass $m$ continues down to row $i+1$ instead of stopping here, then the expected number of samples increases by exactly $m$. Meanwhile, this probability mass $m$ will end up split between locations $(i+1,k)$ and $(i+1,k+1)$, where for a coin of bias $p$ (that will be $\frac{1}{2}\pm\Delta$), we will have $m (1-p)$ mass going left and $m p$ mass going right, contributing to $A^{i+1}_{i+1,k}$ and $A^{i+1}_{i+1,k+1}$ entries respectively. The change in the left hand side of Equation~\ref{eq:non-scary-bound} induced by sending mass $m$ down to level $i+1$ is thus expressed as a linear combination of 3 evaluations of the function $\widehat{R}_{n,k,\rho}$. Since $\widehat{R}_{n,k,\rho}$ is essentially a quadratic function of the ratio $\frac{h^+_{n,k}}{h^-_{n,k}}$, this linear combination evaluates to the difference between a quadratic evaluated at 1 point, versus the weighted average of the quadratic at 2 surrounding points, and is bounded by $m\cdot O(\frac{\Delta^2}{\rho^{0.2}})$ essentially because of the second derivative of the quadratic in the central region.

\end{proof}

\begin{lemma}
\label{lem:nonScaryExcessHellinger}
For any $(n,k)$,
\begin{align*}
    &\; \eps^2\eta_{n,k} \left[ ((\rho+\frac{\eps}{2}) h^+_{n+1,k+1} + (1-\rho-\frac{\eps}{2}) h^-_{n+1, k+1}) \times 2\left(\min\left(\frac{h^+_{n+1,k+1}}{h^-_{n+1,k+1}}, \frac{1}{\rho^{0.1}}\right) - 1\right)^2\right.\\
    &\; \; \;\left. + \; \; \; ((\rho+\frac{\eps}{2}) h^+_{n+1,k} + (1-\rho-\frac{\eps}{2}) h^-_{ n+1, k}) \times 2\left(\min\left(\frac{h^+_{n+1,k}}{h^-_{n+1,k}}, \frac{1}{\rho^{0.1}}\right) - 1\right)^2\right]\\
    \le &\; \eps^2\eta_{n,k}((\rho+\frac{\eps}{2}) h^+_{n,k} + (1-\rho-\frac{\eps}{2}) h^-_{n,k}) \left[2\left(\min\left(\frac{h^+_{n,k}}{h^-_{n,k}},\frac{1}{\rho^{0.1}}\right) - 1\right)^2 + O\left(\frac{\Delta^2}{\rho^{0.2}}\right)\right]
\end{align*}
\end{lemma}

\begin{proof}
	It suffices to show that the left hand side of the inequality is upper bounded by the right hand side, substituting in both options for the minimum.
	For the $1/\rho^{0.1}$ case, since both summands on the left hand side are upper bounded by the $1/\rho^{0.1}$ case of their expressions, the inequality follows trivially and in fact without the excess term of $O(\Delta^2/\rho^{0.2})$.
	
	We now prove the other case, for which it is sufficient to show that
	\begin{align*}
    &\; \eps^2\eta_{n,k} \left[ ((\rho+\frac{\eps}{2}) h^+_{n+1,k+1} + (1-\rho-\frac{\eps}{2}) h^-_{n+1, k+1}) \times 2\left(\frac{h^+_{n+1,k+1}}{h^-_{n+1,k+1}} - 1\right)^2\right.\\
    &\; \; \;\left. + \; \; \; ((\rho+\frac{\eps}{2}) h^+_{n+1,k} + (1-\rho-\frac{\eps}{2}) h^-_{ n+1, k}) \times 2\left(\frac{h^+_{n+1,k}}{h^-_{n+1,k}} - 1\right)^2\right]\\
    \le &\; \eps^2\eta_{n,k}((\rho+\frac{\eps}{2}) h^+_{n,k} + (1-\rho-\frac{\eps}{2}) h^-_{n,k}) \left[2\left(\frac{h^+_{n,k}}{h^-_{n,k}} - 1\right)^2 + O\left(\frac{\Delta^2}{\rho^{0.2}}\right)\right]
	\end{align*}
	when $h^+_{n,k}/h^-_{n,k} \le 1/\rho^{0.1}$.
	
	In turn, we can break this inequality into a conjunction of two inequalities, that
	$$ h^+_{n+1,k+1} \left(\frac{h^+_{n+1,k+1}}{h^-_{n+1,k+1}} - 1\right)^2 + h^+_{n+1,k} \left(\frac{h^+_{n+1,k}}{h^-_{n+1,k}} - 1\right)^2 \le h^+_{n,k} \left[\left(\frac{h^+_{n,k}}{h^-_{n,k}} - 1\right)^2 + O\left(\frac{\Delta^2}{\rho^{0.2}}\right)\right] $$
	and $$ h^-_{n+1,k+1} \left(\frac{h^+_{n+1,k+1}}{h^-_{n+1,k+1}} - 1\right)^2 + h^-_{n+1,k} \left(\frac{h^+_{n+1,k}}{h^-_{n+1,k}} - 1\right)^2 \le h^-_{n,k} \left[\left(\frac{h^+_{n,k}}{h^-_{n,k}} - 1\right)^2 + O\left(\frac{\Delta^2}{\rho^{0.2}}\right)\right] $$
	again assuming that $h^+_{n,k}/h^-_{n,k} \le 1/\rho^{0.1}$.
	
	For the first inequality, observe that $$ \frac{h^+_{n+1,k+1}}{h^-_{n+1,k+1}} = \frac{\frac{1}{2}+\Delta}{\frac{1}{2}-\Delta}\frac{h^+_{n,k}}{h^-_{n,k}} \quad \text{and} \quad \frac{h^+_{n+1,k}}{h^-_{n+1,k}} = \frac{\frac{1}{2}-\Delta}{\frac{1}{2}+\Delta}\frac{h^+_{n,k}}{h^-_{n,k}} $$
	and also $ h^+_{n+1,k+1} = h^+_{n,k} (\frac{1}{2}+\Delta) $ and $h^+_{n+1,k} = h^+_{n,k} (\frac{1}{2}-\Delta) $.
	We therefore factor out and drop the $h^+_{n,k}$ on both sides, simplify, and reduce to showing that
	$$ \left(\frac{1}{2}+\Delta\right) \left(\frac{\frac{1}{2}+\Delta}{\frac{1}{2}-\Delta}\frac{h^+_{n,k}}{h^-_{n,k}} - 1\right)^2 + \left(\frac{1}{2}-\Delta\right) \left(\frac{\frac{1}{2}-\Delta}{\frac{1}{2}+\Delta}\frac{h^+_{n,k}}{h^-_{n,k}} - 1\right)^2 \le \left(\frac{h^+_{n,k}}{h^-_{n,k}} - 1\right)^2 + O\left(\frac{\Delta^2}{\rho^{0.2}}\right) $$
	The left hand side is
	\begin{align*}
	&\; \left(\frac{1}{2}+\Delta\right) \left(\frac{\frac{1}{2}+\Delta}{\frac{1}{2}-\Delta}\frac{h^+_{n,k}}{h^-_{n,k}} - 1\right)^2 + \left(\frac{1}{2}-\Delta\right) \left(\frac{\frac{1}{2}-\Delta}{\frac{1}{2}+\Delta}\frac{h^+_{n,k}}{h^-_{n,k}} - 1\right)^2\\
	= &\; \left(\frac{h^+_{n,k}}{h^-_{n,k}}\right)^2\left(\frac{(\frac{1}{2}+\Delta)^3}{(\frac{1}{2}-\Delta)^2} + \frac{(\frac{1}{2}-\Delta)^3}{(\frac{1}{2}+\Delta)^2} \right) - 2\frac{h^+_{n,k}}{h^-_{n,k}}\left(\frac{(\frac{1}{2}+\Delta)^2}{\frac{1}{2}-\Delta} + \frac{(\frac{1}{2}-\Delta)^2}{\frac{1}{2}+\Delta} \right) + 1\\
	= &\; \left(\frac{h^+_{n,k}}{h^-_{n,k}}\right)^2(1 + O(\Delta^2)) - 2\frac{h^+_{n,k}}{h^-_{n,k}}\left(\frac{(\frac{1}{2}+\Delta)^2}{\frac{1}{2}-\Delta} + \frac{(\frac{1}{2}-\Delta)^2}{\frac{1}{2}+\Delta} \right) + 1\\
	\le &\; \left(\frac{h^+_{n,k}}{h^-_{n,k}}\right)^2(1 + O(\Delta^2)) - 2\frac{h^+_{n,k}}{h^-_{n,k}} + 1\\
	= &\; \left(\frac{h^+_{n,k}}{h^-_{n,k}} - 1\right)^2 + O(\Delta^2)\left(\frac{h^+_{n,k}}{h^-_{n,k}}\right)^2\\
	\le &\; \left(\frac{h^+_{n,k}}{h^-_{n,k}} - 1\right)^2 + O\left(\frac{\Delta^2}{\rho^{0.2}}\right)
	\end{align*}
	where the last inequality holds again because we have $h^+_{n,k}/h^-_{n,k} \le 1/\rho^{0.1}$ by our case analysis.
	
	For the second inequality, via similar reasoning as above, we only need to show that
	$$ \left(\frac{1}{2}-\Delta\right) \left(\frac{\frac{1}{2}+\Delta}{\frac{1}{2}-\Delta}\frac{h^+_{n,k}}{h^-_{n,k}} - 1\right)^2 + \left(\frac{1}{2}+\Delta\right) \left(\frac{\frac{1}{2}-\Delta}{\frac{1}{2}+\Delta}\frac{h^+_{n,k}}{h^-_{n,k}} - 1\right)^2 \le \left(\frac{h^+_{n,k}}{h^-_{n,k}} - 1\right)^2 + O\left(\frac{\Delta^2}{\rho^{0.2}}\right) $$
	The left hand side is
	\begin{align*}
	&\; \left(\frac{1}{2}-\Delta\right) \left(\frac{\frac{1}{2}+\Delta}{\frac{1}{2}-\Delta}\frac{h^+_{n,k}}{h^-_{n,k}} - 1\right)^2 + \left(\frac{1}{2}+\Delta\right) \left(\frac{\frac{1}{2}-\Delta}{\frac{1}{2}+\Delta}\frac{h^+_{n,k}}{h^-_{n,k}} - 1\right)^2\\
	= &\; \left(\frac{h^+_{n,k}}{h^-_{n,k}}\right)^2\left(\frac{(\frac{1}{2}+\Delta)^2}{\frac{1}{2}-\Delta} + \frac{(\frac{1}{2}-\Delta)^2}{\frac{1}{2}+\Delta} \right) - 2\frac{h^+_{n,k}}{h^-_{n,k}}\left(\frac{1}{2}+\Delta + \frac{1}{2} - \Delta\right)+ 1\\
	\le &\; \left(\frac{h^+_{n,k}}{h^-_{n,k}}\right)^2(1 + O(\Delta^2)) - 2\frac{h^+_{n,k}}{h^-_{n,k}} + 1\\
	= &\; \left(\frac{h^+_{n,k}}{h^-_{n,k}} - 1\right)^2 + O(\Delta^2)\left(\frac{h^+_{n,k}}{h^-_{n,k}}\right)^2\\
	\le &\; \left(\frac{h^+_{n,k}}{h^-_{n,k}} - 1\right)^2 + O\left(\frac{\Delta^2}{\rho^{0.2}}\right)
	\end{align*}
	with reasoning as in the previous inequality, thus completing the proof of the lemma.
\end{proof}

\subsection{``High Discrepancy" Region}
\label{sect:scary}

\begin{proposition}
\label{prop:scary}
	Consider an arbitrary stopping rule $\{\gamma_{n,k}\}$ that 1) is non-zero only for $n$ that are powers of 2, and 2) $\gamma_{10^{-8}/\Delta^2,k} = 1$ for all $k$, that is the random walk always stops after $10^{-8}/\Delta^2$ coin flips.
	Let $$ H_\text{\scary}^2 = \Theta(\eps^2) \sum_{n < \frac{10^{-8}}{\Delta^2}, k \text{ s.t. } \frac{h^+_{n,k}}{h^-_{n,k}} \ge \frac{1}{\rho^{0.1}}} \alpha_{n,k} \left((\rho+\frac{\eps}{2}) h^+_{n,k} + (1-\rho-\frac{\eps}{2}) h^-_{n,k}\right) \frac{(h^+_{n,k}-h^-_{n,k})^2}{(\rho h^+_{n,k} + (1-\rho) h^-_{n,k})^2} $$
	be the contribution to the squared Hellinger distance by the ``high discrepancy" region.
	Furthermore, again let $\Exp_{\rho+\frac{\eps}{2}}[n]$ be the expected number of coin flips on this random walk, where we use a $\frac{1}{2}+\Delta$ coin with probability $\rho+\frac{\eps}{2}$ (instead of $\rho$ or $\rho+\eps$), and a $\frac{1}{2}-\Delta$ coin otherwise.
	If all of $\rho$, $\eps$, $\Delta$ and $\eps/\rho$ are smaller than some universal absolute constant, then
	$$ \frac{H_\text{\scary}^2}{\Exp_{\rho+\frac{\eps}{2}}[n]} = O\left(\frac{\eps^2\Delta^2}{\rho}\right) $$
\end{proposition}

The key observation for this section is that the ``high discrepancy" region is in fact at least $\Omega(\log \frac{1}{\rho})$ standard deviations away from where a random walk on the triangle (without a stopping rule) would concentrate; and thus it is very unlikely for the random walk to enter the region.
However, the existence of a stopping rule could potentially skew the distribution of the random walk on each row towards the ``high discrepancy" side of the triangle, while saving on sample complexity by stopping early whenever the walk enters the other side of the triangle.
In this section, we essentially show that this cannot happen.

The analysis in this section relies on our assumption that the stopping rule only stops at rows that are powers of 2 (unlike the analysis of the previous section). Intuitively, if the random walk ends up very far to the right, then there must be a single region of rows $[2^i..2^{i+1}]$ where, without any stopping rule on intermediate rows to guide it, the walk still somehow makes unlikely progress to the right. More explicitly, if the distribution of reaching-and-not-stopping-at row $2^{i+1}$ is skewed significantly far to the right of the distribution of reaching-and-not-stopping-at row $2^{i}$ (despite the intervening process being strictly a binomially distributed random walk), then the only way this could have occurred is if an overwhelming fraction of the probability mass reaching row $2^{i+1}$ stops there. Namely, if probability mass $m$ emerges below row $2^{i+1}$ and skewed far to the right, the potential Hellinger distance gains this induces will be more than counterbalanced by the huge addition to sample complexity induced by the overwhelming (relative to $m$) probability of stopping at row $2^{i+1}$.

We utilize the following fact, essentially a consequence of a Binomial distribution being upper bounded by a corresponding Gaussian.

\begin{fact}
Let $\Bin(n,p,k)$ denote the probability that a Binomial distribution with $n$ trials and bias $p$ has value $k$.
If $\Delta$ is sufficiently small, then there exists some absolute constant $C$ such that for all $n \ge 1$, and for both $\frac{1}{2}+\Delta$ and $\frac{1}{2}-\Delta$ substituted in the expression ``$\frac{1}{2}\pm\Delta$" below,
$$\sum_{k\in[0..n]} e^{\frac{(k-(\frac{1}{2}\pm \Delta)n)^2}{n}}\Bin(n,\frac{1}{2}\pm \Delta,k)\leq C $$
\end{fact}

The sum of the pointwise products of the Binomial pmf and the inverse Gaussian can instead be re-expressed as the evaluation of a convolution between corresponding functions evaluated at a single point. We express this straightforward corollary below, and use it crucially in this section and the next.

\begin{fact}
\label{fact:Conv}
Consider the sequences $f^+_{n,k}(m) = e^{\frac{(k-(\frac{1}{2}+\Delta)n-m)^2}{n}}$ for $m \in \mathbb{Z}$, and $f^-_{n,k}(m) = e^{\frac{(k-(\frac{1}{2}-\Delta)n-m)^2}{n}}$ for $m \in \mathbb{Z}$.
Let $\Bin(n,p)$ be the pmf of the Binomial distribution with $n$ trials and bias $p$.
If $\Delta$ is sufficiently small, then there exists some absolute constant $C$ such that for all $n \ge 1$ and all $k$,
$$ (f^+_{n,k} \ast \Bin(n, \frac{1}{2}+\Delta))(k) \le C $$
and
$$ (f^-_{n,k} \ast \Bin(n, \frac{1}{2}-\Delta))(k) \le C  $$
\end{fact}

To start lower bounding the expected sample complexity of the random walk, we start with the following two lemmas stating that if there is probability $c$ of reaching a right tail on a particular power-of-2 row, then there must be a tail on the previous power-of-2 row that the walk has high probability reaching.
These are formalized as Lemma~\ref{lem:ScaryTailBoostPos} and~\ref{lem:ScaryTailBoostNeg} for $\frac{1}{2}+\Delta$ coins and $\frac{1}{2}-\Delta$ coins respectively.
The crux of the arguments are (weighted) averaging arguments based on Fact~\ref{fact:Conv}.

\begin{lemma}
\label{lem:ScaryTailBoostPos}
Consider an arbitrary stopping rule $\{\gamma_{n,k}\}$ that is non-zero only for $n$ that are powers of 2.
For a coin with bias $\frac{1}{2}+\Delta$, suppose at row $2^j$ there is some position $k \in [(\frac{1}{2}+\Delta)2^j..2^j]$ such that the total probability mass of the random walk reaching positions $\ge k$ at row $2^j$ is at least $c$.
Then, there must be some position $k' \in [0..2^{j-1}]$ at row $2^{j-1}$ such that the probability of reaching positions $\ge k'$ at that row is at least $\frac{c}{C} \cdot f^+_{2^{j-1},k}(k') = \frac{c}{C} \cdot e^{\frac{(k-(\frac{1}{2}+\Delta)2^{j-1}-k')^2}{2^{j-1}}}$, where the constant $C$ and the function $f^+_{n,k}$ are defined in Fact~\ref{fact:Conv}.
\end{lemma}

\begin{proof}
Let us denote by $D^{\downarrow}_{n}$ the vector (over $k \in [0..n]$) of probabilities that the random walk using a coin of bias $\frac{1}{2}+\Delta$ reaches but does not stop at the location $(n,k)$.
Similarly, let us denote by $D_n$ the vector (over $k' \in [0..n]$) of probabilities that the random walk using a $\frac{1}{2}+\Delta$ coin reaches the location $(n,k)$ (and can either stop at or leave the location).

Consider the vector $I$ that is 1 for all coordinates $\le 0$, and 0 otherwise.
Then for any vector $v$, $(v\ast I)(k) = \sum_{i \le k} v(k)$, using ``$\ast$" to denote convolution.

Assume for the sake of contradiction that the statement is false, namely that for all $k' \in [0..2^{j-1}]$, $(D_{2^{j-1}}\ast I)(k') < \frac{c}{C} \cdot f^+_{2^{j-1},k}(k')$.
Then, since $D^{\downarrow}_{2^{j-1}} \le D_{2^{j-1}}$ pointwise, we have for all $k' \in [0..2^{j-1}]$, $(D^\downarrow_{2^{j-1}}\ast I)(k') < \frac{c}{C} \cdot f^+_{2^{j-1},k}(k')$.
Observe that $D^\downarrow_{2^{j-1}}\ast I$ is constant for all coordinates $\le 0$, and that $f^+_{2^{j-1},k}$ is a decreasing function in the same region if $k \in [(\frac{1}{2}+\Delta)2^j..2^j]$ (as in the lemma assumption), and therefore $D^\downarrow_{2^{j-1}}\ast I < f^+_{2^{j-1},k}$ also for that region since the inequality holds at coordinate $0$.
As for coordinates $> 2^j$, $D^\downarrow_{2^{j-1}}\ast I$ is 0, whilst $f^+_{2^{j-1},k}$ is strictly positive.
It follows that the inequality also holds for coordinates $> 2^j$, and thus it holds everywhere.

From this, using the commutativity of convolution, we have
\begin{align*}
    D_{2^{j}} \ast I &= \left(D^\downarrow_{2^{j-1}} \ast \Bin(2^{j-1},\frac{1}{2}+\Delta)\right) \ast I\\
    &= (D^\downarrow_{2^{j-1}}\ast I) \ast \Bin(2^{j-1},\frac{1}{2}+\Delta)\\
    &< f^+_{2^{j-1},k} \ast \Bin(2^{j-1},\frac{1}{2}+\Delta)
\end{align*}
which holds pointwise, in particular at coordinate $k$.
However, $(D_{2^j} \ast I)(k) = c$ by assumption, but $f^+_{2^{j-1},k} \ast \Bin(2^{j-1},\frac{1}{2}+\Delta)(k) \le c$ by Fact~\ref{fact:Conv}, which is a contradiction.
\end{proof}

\begin{lemma}
\label{lem:ScaryTailBoostNeg}
Consider an arbitrary stopping rule $\{\gamma_{n,k}\}$ that is non-zero only for $n$ that are powers of 2.
For a coin with bias $\frac{1}{2}-\Delta$, suppose at row $2^j$ there is some position $k \in [(\frac{1}{2}-\Delta)2^j..2^j]$ such that the total probability mass of the random walk reaching positions $\ge k$ at row $2^j$ is at least $c$.
Then, there must be some position $k' \in [0..2^{j-1}]$ at row $2^{j-1}$ such that the probability of reaching positions $\ge k'$ at that row is at least $\frac{c}{C} \cdot f^-_{2^{j-1},k}(k') = \frac{c}{C} \cdot e^{\frac{(k-(\frac{1}{2}-\Delta)2^{j-1}-k')^2}{2^{j-1}}}$, where the constant $C$ and the function  $f^-_{n,k}$ are defined in Fact~\ref{fact:Conv}.
\end{lemma}

\begin{proof}
The proof is completely analogous to that of Lemma~\ref{lem:ScaryTailBoostPos}.
\end{proof}

In order to conclude the sample complexity lower bound corresponding to a particular row, we need the following lemma saying that, if we repeatedly apply Lemma~\ref{lem:ScaryTailBoostPos} (or Lemma~\ref{lem:ScaryTailBoostNeg}), then some row $2^j$ will have a large probability of stopping at that row, which will contribute a large amount to the overall sample complexity. Further, when $2^j$ is smaller (corresponding to fewer samples taken before stopping), the probability bound induced by the following lemma will be correspondingly higher, so that the product of the row and its stopping probability (i.e., a lower bound on total sample complexity) will be high for the $j$ produced by the lemma.

\begin{lemma}
\label{lem:ProbBoost}
Consider an arbitrary sequence of numbers $\{g_j\}_{j \in [0..J]}$ such that $\sum_j g_j = K$.
Let $r_j$ be chosen arbitrarily such that $r_j \ge \frac{1}{C} e^{g_j^2/2^j}$, where $C$ is the constant in Fact~\ref{fact:Conv} and let $\pi_j = \prod_{i = j}^J r_i$.
Furthermore suppose that $K^2 \ge 100 \log(2C) \cdot 2^J$. Then there exists $j \in [0..J]$ such that $\pi_j 2^{j-J} \ge e^{0.01K^2/2^J}$.
\end{lemma}

\begin{proof}
Taking logarithms and rearranging, we see that it suffices to show the existence of $j$ such that $(j-J-1)\log (2C) + \sum_{i = j}^{J} \frac{g_j^2}{2^j} \ge 0.01 \frac{K^2}{2^J}$.

The sequence $\frac{K}{5} \cdot 0.8^{J-j}$ for $j \in [0..J]$ sums up to less than $K$.
Since $\sum_j g_j = K$, there must exist a $j$ such that $g_j \ge \frac{K}{5} 0.8^{J-j}$.
Therefore, $\frac{g_j^2}{2^j} \ge \frac{K^2}{25} \frac{0.64^{J-j}}{2^j} = \frac{K^2}{25} \frac{1.28^{J-j}}{2^J}$.

It suffices to show that $\frac{K^2}{25} \frac{1.28^{J-j}}{2^J} \ge 0.01 \frac{K^2}{2^J} + (J-j+1)\log(2C)$.
It is easy to check that a sufficient condition is $K^2/2^J \ge 100\log(2C)$, as assumed in the lemma statement; thus we conclude the above inequality for all $j \in [0..J]$.
\end{proof}

Now we use Lemmas~\ref{lem:ScaryTailBoostPos},~\ref{lem:ScaryTailBoostNeg} and~\ref{lem:ProbBoost} to prove the sample complexity lower bound corresponding to a particular row (Lemma~\ref{lem:SampleBoost}).
Afterwards we shall combine these bounds across all possible power-of-2 rows to prove Proposition~\ref{prop:scary}.

\begin{lemma}
\label{lem:SampleBoost}
Consider an arbitrary stopping rule $\{\gamma_{n,k}\}$ that is non-zero only for $n$ that are powers of 2.
For a mixture coin that has bias $\frac{1}{2}+\Delta$ with probability $\rho+\frac{\eps}{2}$ and bias $\frac{1}{2}-\Delta$ otherwise, suppose at row $2^{J+1}$ there is some position $k \in [(\frac{1}{2}+\Delta)2^{J+1}..2^{J+1}]$ such that the probability mass of the random walk reaching positions $\ge k$ at row $2^{J+1}$ is $c$.
If $k \ge (\frac{1}{2}+\Delta)2^{J+1} + \sqrt{100\log(2C)}2^{\frac{J}{2}} + 1$, then the expected sample complexity of a single random walk using the above mixture coin is at least $2^{J-1} \cdot c \cdot e^{0.01(k-(\frac{1}{2}+\Delta)2^{J+1})^2/2^J}$.
\end{lemma}

We point out that the restriction on $k$ (that it lies at least a constant number of standard deviations to the right of its mean) includes the entire high discrepancy region, as analyzed in this section, and further includes all of the larger yet analogous region for the analysis of the last row in the next section.

\begin{proof}[Proof of Lemma~\ref{lem:SampleBoost}]
The probability of the random walk reaching positions $\ge k$ at row $2^{J+1}$ using a mixture coin is the sum of $\rho+\frac{\eps}{2}$ times such probability of the random walk using a $\frac{1}{2} + \Delta$ coin and $1-\rho-\frac{\eps}{2}$ times such probability of the random walk using a $\frac{1}{2}-\Delta$ coin.
Since the total probability of this walk reaching positions $\ge k$ equals $c$, at least half this probability must come from one of the two coin types. Explicitly, at least one of the following two statements has to be true: 1) the probability that the random walk using a coin with $\frac{1}{2}+\Delta$ bias reaches positions $\ge k$ at row $2^{J+1}$ is at least $c/(2\rho + \eps)$, or 2) the same probability but using a $\frac{1}{2} - \Delta$ coin instead is at least $c/(2 - 2\rho - \eps)$.

For case 1, we repeatedly apply Lemma~\ref{lem:ScaryTailBoostPos} to generate a sequence of $\{k_j\}$ from $j = J$ backwards (and $k_{J+1} = k$), until $k_{j^*} < (\frac{1}{2}+\Delta)2^{j^*}$ or $j^* = 0$.
By induction, the probability of reaching positions $\ge k_j$ at row $2^j$ is at least $\frac{c}{2\rho + \eps}\cdot\prod_{i = j}^J \frac{1}{C} e^{\frac{(k_{i+1} - (\frac{1}{2}+\Delta)2^i) - k_i)^2}{2^i}}$.
We would now apply Lemma~\ref{lem:ProbBoost} with $g_i = k_{i+1} - k_i - (\frac{1}{2}+\Delta)2^i$ for $i \ge j^*$, and $g_i = 0$ for $i < j^*$, noting that $K$ in that lemma that we get is $K = \sum_{i=j^*}^J k_{i+1} - k_i - (\frac{1}{2}+\Delta)2^i \ge k_{J+1} - k_{j^*} - \sum_{i=j^*}^J (\frac{1}{2}+\Delta)2^i > k_{J+1} (= k) - (\frac{1}{2}+\Delta)2^{J+1} - 1$ since $k_{j^*} < (\frac{1}{2}+\Delta)2^{j^*}$ if $j^* > 0$ and $k_0 \le 1$ when $j^* = 0$.
Since we assumed in the lemma statement that $k \ge (\frac{1}{2}+\Delta)2^{J+1} + \sqrt{100\log(2C)}2^{\frac{J}{2}} + 1$, we have $K^2/2^J \ge 100\log(2C)$.

Therefore, as a result of applying Lemma~\ref{lem:ProbBoost}, we know that there exists $j$ such that $$2^{j-J} \prod_{i=j}^J \frac{1}{C} e^{\frac{(k_{i+1} - (\frac{1}{2}+\Delta)2^i) - k_i)^2}{2^i}} \ge e^{0.01(k-(\frac{1}{2}+\Delta)2^{J+1})^2/2^J}$$

Thus in case 1, we multiply the left hand side by $c/(2\rho+\eps) 2^J$ to give a lower bound on the expected sample complexity of the random walk, using a $\frac{1}{2}+\Delta$ coin.
We thus use the above inequality to conclude a lower bound of $2^J \frac{c}{2\rho + \eps} e^{0.01(k-(\frac{1}{2}+\Delta)2^{J+1})^2/2^J}$ for the expected sample complexity conditioned on a $\frac{1}{2}+\Delta$ coin.
Since the mixture coin has probability $\rho+\frac{\eps}{2}$ of being a $\frac{1}{2}+\Delta$ coin, the lemma statement follows.

The proof for case 2 is completely analogous, using Lemma~\ref{lem:ScaryTailBoostNeg} instead of Lemma~\ref{lem:ScaryTailBoostPos}, and noting that $k-(\frac{1}{2}-\Delta)2^{J+1} \ge k-(\frac{1}{2}+\Delta)2^{J+1} \ge 0$.
\end{proof}

Equipped with Lemma~\ref{lem:SampleBoost}, we prove Proposition~\ref{prop:scary}.

\begin{proof}[Proof of Proposition~\ref{prop:scary}]
The general strategy is to show using Lemma~\ref{lem:SampleBoost} that, for each row (from 1 to $10^{-8}/\Delta^2$), if there is some probability $c_J$ for the random walk to reaching the high discrepancy region, then: 1) the total expected sample complexity must be large, and 2) by Lemma~\ref{lem:HellingerForm}, if there is probability $c_J$ of reaching the high discrepancy region at row $2^J$, then the contribution to the squared Hellinger distance by the high discrepancy region at row $2^J$ is upper bounded by $\Theta(c_J \eps^2/\rho^2)$.
Thus the squared Hellinger distance per sample complexity for the high discrepancy region of each row is small, and our bounds are in fact strong enough for us to simply take a union bound over the rows and lose by no more than a constant factor.
We now formalize the above argument.

Consider the rows $2^{J+1}$ for $J \in [-1..(\log_2 \frac{10^{-8}}{\Delta^2}) - 1]$.
Recall that the high discrepancy region consists of coordinates $k \in [0..2^{J+1}]$ such that $h^+_{2^{J+1},k}/h^-_{2^{J+1},k} \ge 1/\rho^{0.1}$.
Observe that $$ \frac{h^+_{2^{J+1},k}}{h^-_{2^{J+1},k}} = \left(\frac{1+2\Delta}{1-2\Delta}\right)^{2k-2^{J+1}} $$
and therefore the high discrepancy region consists of $k$ such that $2k-2^{J+1} \ge \frac{.1\log\frac{1}{\rho}}{\log\frac{1+2\Delta}{1-2\Delta}}$, implying that
$$ k \ge 2^{J} + \frac{.1\log\frac{1}{\rho}}{\log\frac{1+2\Delta}{1-2\Delta}} \ge \frac{1}{2}2^{J+1} + \frac{.099\log\frac{1}{\rho}}{4\Delta} $$
Furthermore, since $J \le (\log_2 \frac{10^{-8}}{\Delta^2}) - 1$, we have $2^{J} \le \frac{0.01}{2\Delta^2}$, which for sufficiently small $\rho$ and $\Delta$ (both smaller than some absolute constant, with no requirements on how they depend on each other) means that $\frac{.099\log\frac{1}{\rho}}{4\Delta} \ge \Delta 2^{J+1} + \sqrt{100\log(2C)}2^{\frac{J}{2}} + 1$.
Thus the coordinates $k$ in the high discrepancy region always satisfy the precondition of Lemma~\ref{lem:SampleBoost}.

Now note that for sufficiently small $\rho$ (smaller than some absolute constant),
$$ \left(k - \left(\frac{1}{2}+\Delta\right)2^{J+1}\right)^2 \ge \left(\frac{.098\log\frac{1}{\rho}}{4\Delta}\right)^2 \ge \frac{10^{-8}(\log\frac{1}{\rho})^2}{\Delta^2} $$
Therefore, if the probability of the random walk using a random coin reaches the high discrepancy region at row $2^{J+1}$ is $c_{J+1}$, then by Lemma~\ref{lem:SampleBoost}, the total expected sample complexity of the random walk must be at least $2^{J-1}\cdot c_{J+1} \cdot e^{0.01\frac{10^{-8}(\log\frac{1}{\rho})^2}{\Delta^2 \cdot 2^J}}$.

We can now upper bound the ratio between the high discrepancy region contribution to the squared Hellinger distance and the total expected sample complexity of the random walk by
\begin{align*}
    &\; \frac{\sum_{J \in [-1..(\log_2 \frac{10^{-8}}{\Delta^2}) - 1]} \Theta\left(\frac{\eps^2}{\rho^2}\right)c_{J+1}}{\Exp_{\rho+\frac{\eps}{2}} [n]}\\
    = &\; \Theta\left(\frac{\eps^2}{\rho^2}\right)\sum_{J \in [-1..(\log_2 \frac{10^{-8}}{\Delta^2}) - 1]} \frac{c_{J+1}}{\Exp_{\rho+\frac{\eps}{2}}[n]}\\
    \le &\; \Theta\left(\frac{\eps^2}{\rho^2}\right) \sum_{J \in [-1..(\log_2 \frac{10^{-8}}{\Delta^2}) - 1]} \frac{c_{J+1}}{2^{J-1}\cdot c_{J+1} \cdot e^{0.01\frac{10^{-8}(\log\frac{1}{\rho})^2}{\Delta^2 \cdot 2^J}}}\\
    = &\; \Theta\left(\frac{\eps^2}{\rho^2}\right) \sum_{J \in [-1..(\log_2 \frac{10^{-8}}{\Delta^2}) - 1]} \frac{\rho^{0.02 \cdot \log\frac{1}{\rho} \cdot \frac{10^{-8}}{2\Delta^2 \cdot 2^J}}}{2^{J-1}}\\
    \le &\; \Theta\left(\frac{\eps^2}{\rho^2}\right) \sum_{J \in [-1..(\log_2 \frac{10^{-8}}{\Delta^2}) - 1]} \frac{2^{-\frac{10^{-8}}{2\cdot \Delta^2 \cdot 2^J}}\rho^{0.02 \cdot \log\frac{1}{\rho}}}{2^{J-1}} \quad \text{ since for sufficiently small $\rho$, we have $\rho^{0.02\cdot\log\frac{1}{\rho}} < \frac{1}{2}$}\\
    = &\; \Theta\left(\frac{\eps^2\Delta^2\rho^{0.02\log\frac{1}{\rho}}}{\rho^{2}}\right) \quad \text{ as the sum is bounded by $O(\rho^{0.02\cdot\log\frac{1}{\rho}}\Delta^2)$}\\
    = &\; O\left(\frac{\eps^2\Delta^2}{\rho}\right)
\end{align*}

\end{proof}

\subsection{The Last Row}
\label{sect:last}

We lastly analyze the squared Hellinger distance contribution from the last row of the triangle.

\begin{proposition}
\label{prop:Wizard}
	Consider an arbitrary stopping rule $\{\gamma_{n,k}\}$ that 1) is non-zero only for $n$ that are powers of 2, and 2) $\gamma_{10^{-8}/\Delta^2,k} = 1$ for all $k$, that is the random walk always stops after $10^{-8}/\Delta^2$ coin flips.
	Let $$ H_\text{last}^2 = \Theta(\eps^2) \sum_{n = \frac{10^{-8}}{\Delta^2}, k \in [0..n]} \alpha_{n,k} \left((\rho+\frac{\eps}{2}) h^+_{n,k} + (1-\rho-\frac{\eps}{2}) h^-_{n,k}\right) \frac{\frac{h^+_{n,k}}{\rho} + h^-_{n,k}}{\rho h^+_{n,k} + (1-\rho) h^-_{n,k}} $$
	be the contribution of the squared Hellinger distance from the last row of the triangle, namely row $10^{-8}/\Delta^2$.
	Furthermore, again let $\Exp_{\rho+\frac{\eps}{2}}[n]$ be the expected number of coin flips on this random walk, where we use a $\frac{1}{2}+\Delta$ coin with probability $\rho+\frac{\eps}{2}$ (instead of $\rho$ or $\rho+\eps$), and a $\frac{1}{2}-\Delta$ coin otherwise.
	If all of $\rho$, $\eps$, $\Delta$ and $\eps/\rho$ are smaller than some universal absolute constant, then
	$$ \frac{H_\text{last}^2}{\Exp_{\rho+\frac{\eps}{2}}[n]} = O\left(\frac{\eps^2\Delta^2}{\rho}\right) $$
\end{proposition}

The squared Hellinger distance contribution from the last row has a different form from the rest of the triangle, and can be large even outside the previously ``high discrepancy" region.
While the term $(\frac{h^+_{n,k}}{\rho} + h^-_{n,k})/(\rho h^+_{n,k} + (1-\rho) h^-_{n,k})$ is still upper bounded by $1/\rho^2$ everywhere, it may be as large as $\Theta(1/\rho)$ even in when $h^+_{n,k}/h^-_{n,k} = \Theta(1)$.
The intuition for this section is again that despite having a stopping rule that may have subtle effects on the distribution, it is impossible to skew the distribution of the random walk so much that it appears mostly in the ``high discrepancy" side of the triangle.
We shall use Lemma~\ref{lem:SampleBoost} again along with a case analysis and a weighted averaging argument to show sample complexity lower bounds, which lets us upper bound the squared Hellinger distance contribution per expected sample, as required.

\begin{proof}
	We separate the last row again into a ``high discrepancy" region and a ``central" region, but with a different criterion: whether $$ \frac{\frac{h^+_{n,k}}{\rho} + h^-_{n,k}}{\rho h^+_{n,k} + (1-\rho) h^-_{n,k}} \ge \frac{C}{\rho} $$
	where $C$ is the constant specified in Fact~\ref{fact:Conv}.
	The criterion can be equivalently stated as whether $h^+_{n,k}/h^-_{n,k} \ge r$ for some $r = \Theta(1)$.
	
	For the ``central" region, suppose there is probability $c_{n,k}$ of reaching position $(n,k)$ in that region for the random walk that uses a $\rho+\frac{\eps}{2}$ mixture random coin.
	Consider an alternate form of the squared Hellinger distance contribution that is within a constant factor of that presented in the proposition statement, assuming that $\eps/\rho$ is small:
	$$ \Theta(\eps^2) \sum_{n = \frac{10^{-8}}{\Delta^2}, k \in [0..n]} \alpha_{n,k} \left(\frac{h^+_{n,k}}{\rho} + h^-_{n,k}\right) $$
	This approximation holds since $\rho h^+_{n,k} + (1-\rho) h^-_{n,k}$ and $(\rho+\frac{\eps}{2}) h^+_{n,k} + (1-\rho-\frac{\eps}{2}) h^-_{n,k}$ are within constant factors of each other.
	Note that $c_{n,k} = \alpha_{n,k} ((\rho+\frac{\eps}{2})h^+_{n,k} + (1-\rho-\frac{\eps}{2})h^-_{n,k})$, and so when $h^+_{n,k}/h^-_{n,k} \le r = \Theta(1)$, we have both $\alpha_{n,k} h^+_{n,k} = O(c)$ and $\alpha_{n,k} h^-_{n,k} = O(c)$.
	Thus the squared Hellinger contribution of this location is upper bounded by $O(c\eps^2/\rho)$, yet the total sample complexity is lower bounded by $\Omega(cn) = \Omega(c/\Delta^2)$, giving a fraction that is $O(\eps^2\Delta^2/\rho)$.
	
	For the ``central" region, recall that it consists of the locations where
	\begin{equation}
	\label{eq:WizardTerm}
	    \frac{\frac{h^+_{n,k}}{\rho} + h^-_{n,k}}{\rho h^+_{n,k} + (1-\rho) h^-_{n,k}}
	\end{equation}
	ranges from $\frac{C}{\rho}$ to $\frac{1}{\rho^2}$.
	We separate this region into $O(\log\frac{1}{\rho})$ buckets delimited by consecutive powers of 2.
	
	Suppose there is probability $c_\text{\scary}$ of the $\rho+\frac{\eps}{2}$ mixture coin random walk entering the ``high discrepancy" region in the last row.
	Note that the geometric sequence $1, 0.8, 0.64, \ldots$ converges to 5, and therefore the sequence $\{\frac{c_\text{\scary}}{5} 0.8^i\}_i$ sums to $c_\text{\scary}$.
	If we took only the first $O(\log\frac{1}{\rho})$ many terms, they sum to strictly less than $c_\text{\scary}$.
	By a standard averaging argument, there must exist a bucket such that the probability of reaching locations in that bucket $i$ is greater than $\frac{c_\text{\scary}}{5} 0.8^i$.
	Note that the values (Equation~\ref{eq:WizardTerm}) inside bucket $i$ range from $C \cdot 2^i/\rho$ to $2 \cdot C \cdot 2^i/\rho$.
	It is possible to calculate that the locations $k$ within bucket $i$ satisfy $k \ge \frac{n}{2} + \frac{1}{16}\frac{i \log C}{\Delta}$, where $n$ again is $10^{-8}/\Delta^2$.
	For a sufficiently small $\Delta$, this lower bound in location is at least $(\frac{1}{2}+\Delta)n + \sqrt{100\log(2C)\cdot n}+1$, and therefore we can apply Lemma~\ref{lem:SampleBoost}.
	
	Furthermore, the locations are also at least $(i \cdot 10^{-2} \log C)/\Delta$ away from $(\frac{1}{2}+\Delta)n$, and so Lemma~\ref{lem:SampleBoost} guarantees a sample complexity of at least $\Theta(1/\Delta^2) \times \frac{c_\text{\scary}}{5} 0.8^i \times e^{0.01((i \cdot 10^{-2} \log C)/\Delta)^2 \cdot \Delta^2 \times 10^8} \ge \Theta(1/\Delta^2) \times \frac{c_\text{\scary}}{5} 0.8^i \times e^{100(i \log C)^2} \ge \Omega(2^i c_\text{\scary}/\Delta^2)$, where the last inequality is true because the large exponential term has a logarithm that is quadratic in $i$ and with a base that is a lot greater than $1/0.8$.
	
	The squared Hellinger distance contribution from the ``high discrepancy" region in the last row is upper bounded by $O(c_\text{\scary}\eps^2 2^i/\rho)$, and we have shown a sample complexity lower bound of $\Omega(2^i c_\text{\scary}/\Delta^2)$.
	We therefore conclude a fraction of $O(\eps^2\Delta^2/\rho)$ for the squared Hellinger distance per expected sample, for contributions from the ``high discrepancy" region in the last row.
	
	Summarizing, both the ``high discrepancy" and ``central" region contribute no more than $O(\eps^2\Delta^2/\rho)$ times $\Exp_{\rho+\frac{\eps}{2}}[n]$ to $H^2_\text{last}$.
	Therefore, the proposition follows from summing the two contributions.
\end{proof}

\section{Remaining Proofs/Calculations of Results}
\label{app:calc}

\medskip\noindent{\bf Lemma~\ref{lem:HellingerForm}.}\emph{
Consider the two probability distributions in Proposition~\ref{prop:BoundedLower} over locations $(n,k)$ in the Pascal triangle of depth $10^{-8}/\Delta^2$ and bias $p \in \{\frac{1}{2}\pm\Delta\}$, generated by the given stopping rule $\{\gamma_{n,k}\}$ in the two cases of 1) a coin with bias $\frac{1}{2} + \Delta$ is used with probability $\rho$ and a coin with bias $\frac{1}{2} - \Delta$ is used otherwise versus 2) a coin with bias $\frac{1}{2} + \Delta$ is used with probability $\rho+\eps$ and a coin with bias $\frac{1}{2} - \Delta$ is used otherwise.
If $\eps/\rho$ is smaller than some universal constant, then the squared Hellinger distance between these two distributions can be written as
\begin{align*}
	\Theta(\eps^2) \Biggl[&\sum_{n < \frac{10^{-8}}{\Delta^2}, k \in [0..n]} \alpha_{n,k} ((\rho+\frac{\eps}{2}) h^+_{n,k} + (1-\rho-\frac{\eps}{2}) h^-_{n,k}) \frac{(h^+_{n,k}-h^-_{n,k})^2}{(\rho h^+_{n,k} + (1-\rho) h^-_{n,k})^2}\\
	& \sum_{n = \frac{10^{-8}}{\Delta^2}, k \in [0..n]} \alpha_{n,k} ((\rho+\frac{\eps}{2}) h^+_{n,k} + (1-\rho-\frac{\eps}{2}) h^-_{n,k}) \frac{\frac{h^+_{n,k}}{\rho} + h^-_{n,k}}{\rho h^+_{n,k} + (1-\rho) h^-_{n,k}} \Biggr]
\end{align*}
}

\begin{proof}
	For $n < 10^{-8}/\Delta^2$, the probability of that the location $(n,k)$ is revealed, for the two distributions we consider in Proposition~\ref{prop:BoundedLower}, are $\alpha_{n,k} (\rho h^+_{n,k} + (1-\rho) h^-_{n,k})$ and $\alpha_{n,k} ((\rho+\eps) h^+_{n,k} + (1-\rho-\eps) h^-_{n,k})$ respectively.
	Thus, the contribution by these locations to the squared Hellinger distance is proportional to:
	\begin{align*}
		&\; \sum_{n,k} (\sqrt{\alpha_{n,k}(\rho h^+_{n,k} + (1-\rho)h^-_{n,k})} - \sqrt{\alpha_{n,k}((\rho+\eps) h^+_{n,k} + (1-\rho-\eps)h^-_{n,k})})^2\\
		= &\; \sum_{n,k} \alpha_{n,k} (\rho h^+_{n,k} + (1-\rho)h^-_{n,k}) \left(1 - \sqrt{\frac{(\rho+\eps) h^+_{n,k} + (1-\rho-\eps)h^-_{n,k}}{\rho h^+_{n,k} + (1-\rho)h^-_{n,k}}}\right)^2\\
		= &\; \sum_{n,k} \alpha_{n,k} (\rho h^+_{n,k} + (1-\rho)h^-_{n,k}) \left(1 - \sqrt{1 + \eps \frac{h^+_{n,k} - h^-_{n,k}}{\rho h^+_{n,k} + (1-\rho)h^-_{n,k}}}\right)^2\\
	= &\; \sum_{n,k} \alpha_{n,k} (\rho h^+_{n,k} + (1-\rho)h^-_{n,k}) \left(1 - \left(1 + \frac{1}{2} \eps \frac{h^+_{n,k} - h^-_{n,k}}{\rho h^+_{n,k} + (1-\rho)h^-_{n,k}} + \Theta\left(\left(\eps\frac{h^+_{n,k} - h^-_{n,k}}{\rho h^+_{n,k} + (1-\rho)h^-_{n,k}}\right)^2\right)\right)\right)^2\\
	= &\; \sum_{n,k} \alpha_{n,k} (\rho h^+_{n,k} + (1-\rho)h^-_{n,k}) \left(\frac{1}{2} \eps \frac{h^+_{n,k} - h^-_{n,k}}{\rho h^+_{n,k} + (1-\rho)h^-_{n,k}} + \Theta\left(\left(\eps\frac{h^+_{n,k} - h^-_{n,k}}{\rho h^+_{n,k} + (1-\rho)h^-_{n,k}}\right)^2\right)\right)^2
	\end{align*}
	Note that the multiplier to $\eps$ is upper bounded by $1/\rho$, and therefore if $\eps/\rho$ is sufficiently small, we have the last line being equal to
	\begin{align*}
	&\; \sum_{n,k} \alpha_{n,k} (\rho h^+_{n,k} + (1-\rho)h^-_{n,k}) \left( \Theta\left(\eps \frac{h^+_{n,k} - h^-_{n,k}}{\rho h^+_{n,k} + (1-\rho)h^-_{n,k}}\right)\right)^2\\	
	= &\; \sum_{n,k} \alpha_{n,k} (\rho h^+_{n,k} + (1-\rho)h^-_{n,k}) \;\Theta\left(\eps^2 \frac{(h^+_{n,k} - h^-_{n,k})^2}{(\rho h^+_{n,k} + (1-\rho)h^-_{n,k})^2}\right)\\
	= &\;\Theta(\eps^2)\sum_{n,k} \alpha_{n,k} (\rho h^{+}_{n,k} + (1-\rho)h^{-}_{n,k}) \frac{(h^{+}_{n,k} - h^{-}_{n,k})^2}{(\rho h^{+}_{n,k} + (1-\rho)h^{-}_{n,k})^2}
	\end{align*}
	Finally, note that is $\eps/\rho$ is a small constant, then $\rho$ and $1-\rho$ are respectively within a small constant factor of $\rho + \frac{\eps}{2}$ and $1-\rho-\frac{\eps}{2}$, meaning that $(\rho h^{+}_{n,k} + (1-\rho)h^{-}_{n,k})$ is within a constant factor of $((\rho+\frac{\eps}{2}) h^+_{n,k} + (1-\rho-\frac{\eps}{2}) h^-_{n,k})$.
	
	For $n = 10^{-8}/\Delta^2$, the probability that $((n,k), \frac{1}{2}+\Delta)$ is revealed is
	$\alpha_{n,k} \, \rho \, h^+_{n,k}$ and $\alpha_{n,k} (\rho+\eps)h^+_{n,k}$ for the two scenarios respectively.
	A similar calculation above gives a squared Hellinger distance contribution of
	$$ \Theta(\eps^2) \, \alpha_{n,k} \left(\rho+\frac{\eps}{2}\right) \frac{h^+_{n,k}}{\rho} $$
	As for the contribution from the revealing of $((n,k), \frac{1}{2}-\Delta$, the respective probabilities are
	$\alpha_{n,k} \, (1-\rho) \, h^-_{n,k}$ and $\alpha_{n,k} (1-\rho-\eps)h^-_{n,k}$, and similar calculations give a squared Hellinger distance contribution of $$ \Theta(\eps^2) \, \alpha_{n,k} \, h^-_{n,k}$$ which with algebraic manipulation and approximations as in the $n < 10^{-8}/\Delta^2$ case completes the proof of the lemma.
\end{proof}

\medskip\noindent{\bf Lemma~\ref{Lem:Vvalue}.} \emph{For the quadratic program in Figure~\ref{Fig:QP}, the optimal assignments to $\{\tilde{v}_{n,k}\}$ are
	$$ \tilde{v}_{n,k} = \frac{\frac{h^{+}_{n,k}-h^{-}_{n,k}}{\rho h^{+}_{n,k} + (1-\rho)h^{-}_{n,k}}}{\sum_{m,j} \alpha_{m,j}\frac{(h^{+}_{m,j} - h^{-}_{m,j})^2}{\rho h^{+}_{m,j} + (1-\rho)h^{-}_{m,j}}}$$
	(and we choose $v_{n,k} = \tilde{v}_{n,k} + \rho$), giving an objective value of
	$$ \frac{1}{\sum_{n,k} \alpha_{n,k}\frac{(h^{+}_{n,k} - h^{-}_{n,k})^2}{\rho h^{+}_{n,k} + (1-\rho)h^{-}_{n,k}}}$$}
	
\begin{proof}
	We use the method of Lagrangian multiplier to find the optimal assignment to $\{\tilde{v}_{n,k}\}$.
	
	The Lagrangian of the program is
	$$ L = \sum_{m,j} \alpha_{m,j}(\rho h^{+}_{m,j} + (1-\rho)h^{-}_{m,j}) \tilde{v}_{m,j}^2 + \lambda \left(\left( \sum_{m,j} \alpha_{m,j} (h^{+}_{m,j} - h^{-}_{m,j}) \tilde{v}_{m,j}\right) - 1\right)$$
	where $\lambda$ is the Lagrange multiplier.
	
	We need to find assignments to $\{\tilde{v}_{n,k}\}$ and $\lambda$ such that $\Grad_{\{\tilde{v}_{n,k}\}, \lambda} L = 0$.
	Computing the partial derivatives gives the following system of equations:
	\begin{align}
	\label{Eq:VValue1}
	2\alpha_{n,k}(\rho h^{+}_{n,k} + (1-\rho)h^{-}_{n,k})\tilde{v}_{n,k}	 + \lambda \alpha_{n,k}(h^{+}_{n,k} - h^{-}_{n,k}) &= 0 \text{ for all $n,k$}\\
	\label{Eq:VValue2}
	\sum_{m,j} \alpha_{m,j}(h^{+}_{m,j}-h^{-}_{m,j})\tilde{v}_{m,j d} = 1
	\end{align}
	Rearranging Equation~\ref{Eq:VValue1} gives
	\begin{equation}
	\label{Eq:VValue3}
		\tilde{v}_{n,k} = \frac{-\lambda (h^{+}_{n,k} - h^{-}_{n,k})}{2(\rho h^{+}_{n,k} + (1-\rho)h^{-}_{n,k})}
	\end{equation}
	and substituting this into Equation~\ref{Eq:VValue2} gives
	$$ -\lambda \sum_{m,j} \frac{\alpha_{m,j}(h^{+}_{m,j}-h^{-}_{m,j})^2}{2(\rho h^{+}_{m,j} + (1-\rho)h^{-}_{m,j})} = 1 $$
	which lets us solve for $\lambda$
	$$ \lambda = -1/\sum_{m,j} \frac{\alpha_{m,j}(h^{+}_{m,j}-h^{-}_{m,j})^2}{2(\rho h^{+}_{m,j} + (1-\rho)h^{-}_{m,j})} $$
	which when substituted back into Equation~\ref{Eq:VValue3} gives
	$$ \tilde{v}_{n,k} = \frac{\frac{h^{+}_{n,k} - h^{-}_{n,k}}{\rho h^{+}_{n,k} + (1-\rho)h^{-}_{n,k}}}{\sum_{m,j} \frac{\alpha_{m,j}(h^{+}_{m,j}-h^{-}_{m,j})^2}{\rho h^{+}_{m,j} + (1-\rho)h^{-}_{m,j}}} $$
	as desired.
	
	The optimal value of the program can be calculated by substituting the assignment to the objective function.
\end{proof}

\medskip\noindent{\bf Lemma~\ref{lem:HellingerApprox}.}
	\emph{Consider an arbitrary stopping rule $\{\gamma_{n,k}\}$ giving coefficients $\{\alpha_{n,k}\}$.
The squared Hellinger distance between 1) a random coin drawn in case A inducing a distribution on the Pascal triangle given the stopping rule and 2) a random coin drawn in case B instead, is
$$ \Theta(\eps^2)\sum_{n,k} \frac{(h^{+}_{n,k} - h^{-}_{n,k})^2}{\rho h^{+}_{n,k} + (1-\rho)h^{-}_{n,k}} \, \alpha_{n,k} $$
assuming that $\eps /\rho$ is smaller than some universal constant.}

\begin{proof}
	In scenario 1, the distribution induced by a random coin on the Pascal triangle is
	$$ \alpha_{n,k}(\rho h^+_{n,k} + (1-\rho)h^-_{n,k}) $$
	and similarly for scenario 2,
	$$ \alpha_{n,k}((\rho+\eps) h^+_{n,k} + (1-\rho-\eps)h^-_{n,k}) $$
	The squared Hellinger distance is therefore proportional to
	\begin{align*}
		&\; \sum_{n,k} (\sqrt{\alpha_{n,k}(\rho h^+_{n,k} + (1-\rho)h^-_{n,k})} - \sqrt{\alpha_{n,k}((\rho+\eps) h^+_{n,k} + (1-\rho-\eps)h^-_{n,k})})^2\\
		= &\; \sum_{n,k} \alpha_{n,k} (\rho h^+_{n,k} + (1-\rho)h^-_{n,k}) \left(1 - \sqrt{\frac{(\rho+\eps) h^+_{n,k} + (1-\rho-\eps)h^-_{n,k}}{\rho h^+_{n,k} + (1-\rho)h^-_{n,k}}}\right)^2\\
		= &\; \sum_{n,k} \alpha_{n,k} (\rho h^+_{n,k} + (1-\rho)h^-_{n,k}) \left(1 - \sqrt{1 + \eps \frac{h^+_{n,k} - h^-_{n,k}}{\rho h^+_{n,k} + (1-\rho)h^-_{n,k}}}\right)^2\\
	= &\; \sum_{n,k} \alpha_{n,k} (\rho h^+_{n,k} + (1-\rho)h^-_{n,k}) \left(1 - \left(1 + \frac{1}{2} \eps \frac{h^+_{n,k} - h^-_{n,k}}{\rho h^+_{n,k} + (1-\rho)h^-_{n,k}} + \Theta\left(\left(\eps\frac{h^+_{n,k} - h^-_{n,k}}{\rho h^+_{n,k} + (1-\rho)h^-_{n,k}}\right)^2\right)\right)\right)^2\\
	= &\; \sum_{n,k} \alpha_{n,k} (\rho h^+_{n,k} + (1-\rho)h^-_{n,k}) \left(\frac{1}{2} \eps \frac{h^+_{n,k} - h^-_{n,k}}{\rho h^+_{n,k} + (1-\rho)h^-_{n,k}} + \Theta\left(\left(\eps\frac{h^+_{n,k} - h^-_{n,k}}{\rho h^+_{n,k} + (1-\rho)h^-_{n,k}}\right)^2\right)\right)^2
	\end{align*}
	Note that the multiplier to $\eps$ is upper bounded by $1/\rho$, and therefore if $\eps/\rho$ is sufficiently small, we have the last line being equal to
	\begin{align*}
	&\; \sum_{n,k} \alpha_{n,k} (\rho h^+_{n,k} + (1-\rho)h^-_{n,k}) \left( \Theta\left(\eps \frac{h^+_{n,k} - h^-_{n,k}}{\rho h^+_{n,k} + (1-\rho)h^-_{n,k}}\right)\right)^2\\	
	= &\; \sum_{n,k} \alpha_{n,k} (\rho h^+_{n,k} + (1-\rho)h^-_{n,k}) \;\Theta\left(\eps^2 \frac{(h^+_{n,k} - h^-_{n,k})^2}{(\rho h^+_{n,k} + (1-\rho)h^-_{n,k})^2}\right)\\
	= &\; \Theta(\eps^2)\sum_{n,k} \frac{(h^{+}_{n,k} - h^{-}_{n,k})^2}{\rho h^{+}_{n,k} + (1-\rho)h^{-}_{n,k}} \, \alpha_{n,k}
	\end{align*}
\end{proof}

\end{document}